\documentclass{article}

\PassOptionsToPackage{numbers}{natbib}
\usepackage[main, final]{neurips_2026}

\usepackage[utf8]{inputenc} 
\usepackage[T1]{fontenc}    
\usepackage{hyperref}       
\usepackage{url}            
\usepackage{booktabs}       
\usepackage{amsfonts}       
\usepackage{nicefrac}       
\usepackage{microtype}      
\usepackage{xcolor}         

\usepackage{graphicx}
\usepackage{wrapfig}
\usepackage{amssymb,amsmath,amsthm}

\usepackage{pgf}
\usepackage{subcaption}
\usepackage{makecell}

\usepackage{float}
\usepackage{transparent}

\usepackage[colorinlistoftodos]{todonotes}

\newtheorem{theorem}{Theorem}
\newtheorem{lemma}{Lemma}
\newtheorem{proposition}{Proposition}

\def\our{\mbox{LAPLEX}}
\def\triu{\mathrm{triu}}
\def\tril{\mathrm{tril}}
\def\R{\mathbb{R}}

\usepackage[normalem]{ulem}
\newif\ifrevhl \revhlfalse

\newcommand{\revdel}[1]{\ifrevhl\textcolor{gray}{\sout{#1}}\else\fi}

\author{
\textbf{\L{}ukasz Struski}$^{1}$ \\
\texttt{\small lukasz.struski@uj.edu.pl} \\
\And
\textbf{Hanna Blazhko}$^{1}$ \\
\texttt{\small hanna.blazhko@student.uj.edu.pl} \\
\And
\textbf{Piotr Kubaty}$^{1,2}$ \\
\texttt{\small piotr.kubaty@doctoral.uj.edu.pl} \\
\And
\textbf{Jacek Tabor}$^{1,3}$ \\
\texttt{\small jacek.tabor@uj.edu.pl} \\
\And
\\[-2em]
\begin{tabular}{@{}l@{}}
\small $^{1}$ Faculty of Mathematics and Computer Science, Jagiellonian University, Kraków, Poland \\
\small $^{2}$ Doctoral School of Exact and Natural Sciences, Jagiellonian University, Kraków, Poland \\
\small $^{3}$ Centre for Credible Artificial Intelligence, Warsaw University of Technology, Poland
\end{tabular}
}

\title{LAPLEX: The FFT of Learnable Laplace Kernels}

\begin{document}

\maketitle

\begin{abstract}
Fast matrix--vector multiplication in deep learning usually comes with a
choice: fixed geometry and exact computation, as in the fast Fourier transform
(FFT), or adaptive\space%
geometry paid for by dense parameters, random features, or low-rank
surrogates. To move beyond this trade-off, we introduce \our{}, a class of
exact, trainable (phased) Laplace-kernel operators.

A \our{} layer is a typically full-rank dense matrix, implicitly defined by
learnable coordinate anchors, with FFT-like scaling. Consequently, it supports
trainable matrix--vector operations at vector dimensions up to $10^9$ on
modern GPUs. As a neural layer, it yields compact projections and
classification heads interpretable as soft, trainable routing models. The same
primitive also serves as an efficient Gram operator, enabling high-dimensional
covariance models on flattened images of dimension $3 \cdot 10^6$ that
preserve visible spatial structure without imposing convolutional bias. 
These
applications reflect a single principle: dense geometry can be learned without
storing a dense matrix, which enables data-adaptive global interactions in
regimes where ordinary dense layers are out of reach. In this sense, \our{}
separates expressivity from storage cost: it behaves like a dense trainable
matrix, but is represented and applied through a small structured set of
parameters. Our code is
available at: \url{https://github.com/gmum/laplex}.\end{abstract}

\section{Introduction}

Modern machine-learning systems increasingly operate in regimes where the
ambient dimension reaches millions or billions, and the bottleneck is no
longer only statistical efficiency but also the wall-clock cost of
\emph{linear operations}: applying matrices, computing Gram operators, and
mixing features. These operations are often treated as routine infrastructure,
yet they determine which models can be trained, which kernels can be used
exactly, and which high-dimensional representations fit on hardware at all.
Standard scaling techniques typically sacrifice at least one of the properties
we want: \emph{exactness, learnability, and near-linear scaling}.
For a linear map $A\colon\mathbb{R}^{k}\to\mathbb{R}^{n}$, let $k$ and $n$
denote the input and output dimensions, and let $r\le\min(n,k)$ be a rank budget.
Low-rank methods~\citep{hu2022lora,udell2016generalized} retain dense,
basis-dependent factors at $\mathcal{O}(r(n+k))$ memory; randomized\space%
sketches~\citep{achlioptas2003database,charikar2002finding,rahimi2007random}
act on a fixed indexing and trade exactness for scalability; sparse
methods~\citep{frankle2019lottery,han2015learning} can break hardware
regularity and are difficult to train end-to-end; classical fast transforms
such as the FFT are exact and cheap, but fixed.

\paragraph{The hidden cost of basis-dependent operators.}
A dense linear map $A\in\mathbb{R}^{n\times k}$ consists of $n$ inner
products, assigning a separate coefficient to every input--output coordinate
pair. Even a projection to $n\ll k$ stores $nk$ coefficients: reducing the
output dimension does not remove dependence on the input basis. A rank-$r$
factorization $A=UV^\top$, with $U\in\mathbb{R}^{n\times r}$ and
$V\in\mathbb{R}^{k\times r}$, reduces storage to $r(n+k)$ coefficients but
still allocates parameters across the ambient coordinate basis. Low-rank
parameterizations such as LoRA~\citep{hu2022lora} reduce this cost without
eliminating that dependence.
\par%
A different route is to parameterize operators \emph{through the indices
themselves}. The FFT is the canonical example: exact, free of learnable
per-coordinate weights, and $\mathcal{O}(n\log n)$. Sketches such as
CountSketch~\citep{charikar2002finding}, Random Fourier
Features~\citep{rahimi2007random}, and Johnson--Lindenstrauss
projections~\citep{achlioptas2003database} follow a related index-level
philosophy, but they act on a \emph{fixed} indexing and cannot adapt how
coordinates are matched. They are efficient, but rigid.

\begin{wrapfigure}{R}{0.55\textwidth}
    \vspace{-\baselineskip}
    \centering
    \includegraphics[width=\linewidth]{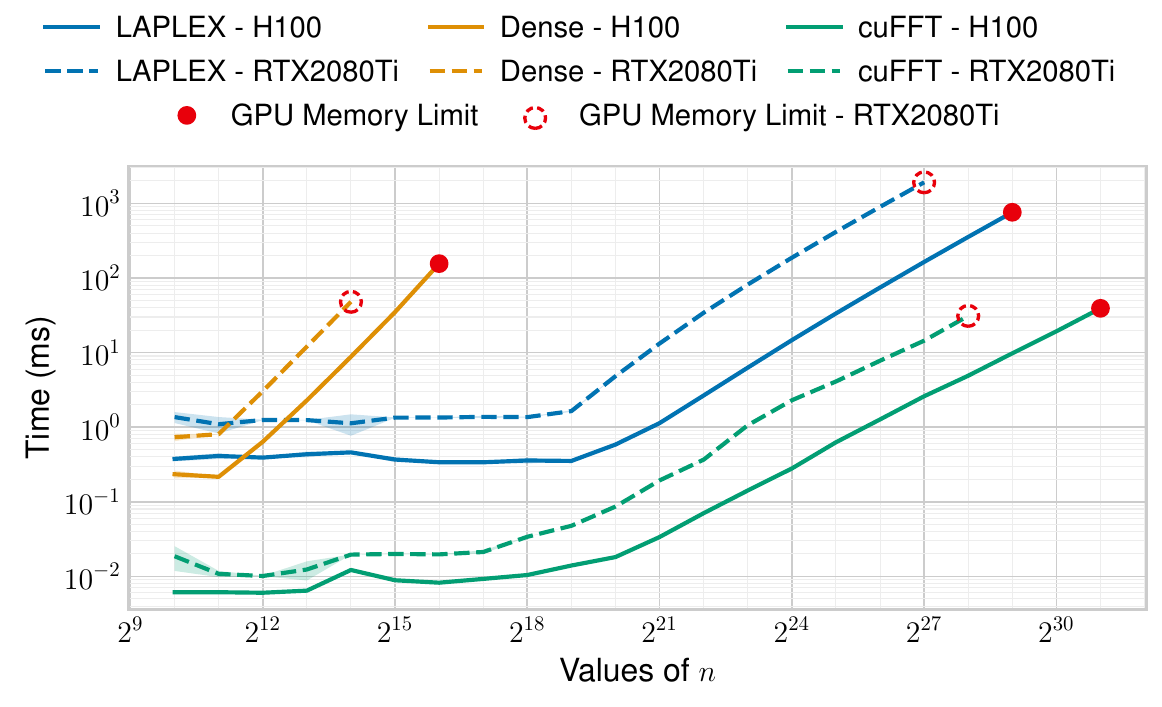}
    \caption{%
        Runtime of \textsc{Laplex}, Dense, and \texttt{torch.fft} (cuFFT) versus the number of elements~$n$ (log\textsubscript{2} scale) on NVIDIA H100 (solid) and RTX~2080~Ti (11\,GB; dashed). Shaded regions show standard deviation; red dots (dashed red circles for the RTX~2080~Ti) mark the last size that fits in GPU memory. \textsc{Laplex} scales efficiently to large $n$, cuFFT is consistently fastest, while Dense runs out of memory at relatively small sizes.
    }
    \label{fig:comp_runtime}
    \vspace{\baselineskip}
\end{wrapfigure}

\paragraph{Our idea: index-level operators that can learn their input geometry.}
We pursue operators that retain the memory advantages of index-based
computation but, unlike the FFT or CountSketch, can \emph{learn} how to
align, shift, and permute coordinates. The resulting object is a single structured matrix: the
\emph{Laplace-kernel} matrix
\[
   \our(a,b) \;=\;
   \bigl[\exp(-|a_i - b_j|)\bigr]_{i,j},
\]
which, in a natural limiting regime, encodes \emph{all} permutations of
the coordinate set while using only $\mathcal{O}(n)$ parameters.
Despite its apparent density, matrix--vector products with $\our$
can be computed \emph{exactly}, with no approximation, in $\mathcal{O}(n)$
time\footnote{During training, sorting gives $\mathcal{O}(n\log n)$ time.}
and memory, without ever materializing the matrix. The same holds for
applying the weighted Gram operator to a vector: for $A=\our(a,b)$,
$(A\,\mathrm{diag}(D)\,A^\top)v = A\bigl(D\odot(A^\top v)\bigr)$ is just two
such products. Only when the full $n\!\times\!n$ Gram matrix itself is needed
does the cost become $\mathcal{O}(n^2+k\log n)$ (Thm.~\ref{thm:gram}), the
$n^2$ term being the unavoidable output size. The algorithm is fully
parallel: in our experiments, wall-clock time remains almost flat until the
GPU saturates; see Fig.~\ref{fig:comp_runtime}.

The resulting primitive is useful wherever a model needs a large linear map
but cannot afford a dense one. In representation learning, \our{} can serve
as a trainable projection from massive feature vectors to compact embeddings.
In fine-tuning, it can replace dense or LoRA-style heads with structured
class--feature alignments. In kernel methods and second-order
optimization, its exact Gram operator makes computations available that
would otherwise require Nystr\"om, random-feature, or low-rank
approximations; we demonstrate this concretely for natural-gradient / Gauss--Newton
optimization (App.~\ref{app:second-order}), and test \our{} attention in App.~\ref{app:attention}. As a proof of concept, we train a Gaussian density model on
flattened $3\!\times\!1024\!\times\!1024$ images; see
Fig.~\ref{fig:div2k-samples}. The learned density preserves visible pixel
correspondence despite operating on flattened images. The common pattern is
that \our{} keeps the algebraic regularity of a fast transform while making
the coordinate system itself learnable.

\paragraph{Contributions.}
\our{} combines the efficiency and regularity of fast transforms with
learnable geometry over coordinates. It does not
rely on approximation, dimensionality reduction, or irregular sparsity, while
retaining exactness and near-linear time and memory complexity. This shift
from parameterizing matrices to parameterizing \emph{index interactions}
makes large-scale matrix--vector operations and structured Gram construction
both tractable and adaptive. Operations that\space%
were previously bottlenecks, such as exact kernel application or Gram
computation, become practical building blocks in high-dimensional models.
\begin{itemize}
    \item We identify the basis-dependence limitation of low-rank operators
          and motivate \emph{index-level, permutation-aware} operators as a
          principled alternative.
    \item We show that the resulting Laplace-kernel matrix $\our$ uses
          $\mathcal{O}(n)$ parameters while encoding all coordinate
          permutations in a natural limiting regime.
    \item We give an \emph{exact}, fully parallel $\mathcal{O}(n\log n)$
          algorithm\footnote{We provide both a native PyTorch implementation
          and a CUDA-optimized implementation.} for matrix--vector products with $\our$ (including the weighted Gram operator
          applied to a vector) using $\mathcal{O}(n)$ memory, without
          materializing the matrix, enabling dimensions up to $10^9$, and an exact $\mathcal{O}(n^2+k\log n)$ algorithm for the full weighted Gram
          matrix.
    \item We show empirically that $\our$ matches FFT-class scaling while
          functioning efficiently as a drop-in structured layer in modern
          machine-learning systems.
\end{itemize}

\paragraph{Paper organization.}
We organize the paper around the roles that \our{} plays. Sec.~\ref{sec:related-work}
gives a compact map of the nearest literature. Sec.~\ref{sec:theory}
develops the exact scan algebra for matrix--vector and Gram products.
Sec.~\ref{sec:scaling-consequences} shows why this algebra changes the
feasible high-dimensional regime, through primitive GPU scaling and a
density-modeling experiment based on exact Gramians. Sec.~\ref{sec:learnable-routing}
studies \our{} as a learnable relaxation of index routing and tests this
view through structured projections against CountSketch and related
baselines. Sec.~\ref{sec:restricted-views} relates FFT-based matrix--vector
multiplication with Toeplitz matrices (Toeplitz--FFT below) and RFF to
fixed-grid and Monte Carlo views of the same Laplace kernel.\space%
Finally, Sec.~\ref{sec:learnable-layers} evaluates anchor geometry in
neural classification heads, and Sec.~\ref{sec:conclusions} discusses
limitations and open directions.

\section{Related Work}
\label{sec:related-work}

We focus on the closest algorithmic alternatives: exact fast transforms,
randomized kernel approximations, low-rank parameterizations, and sparsity.
The useful comparison is not only asymptotic complexity, but also what each
method fixes in advance: a grid, a random map, a feature count, a rank, or
a sparsity pattern.

\textbf{Exact fast transforms.}
The FFT~\citep{cooley1965algorithm,frigo2005design} and algorithms exploiting
Toeplitz or circulant structure~\citep{gray2006toeplitz} achieve near-linear exact\space%
multiplication by exploiting fixed grid structure. For the Laplace kernel,
the uniform-grid choice $a_i=b_i=i\Delta$ makes $\our(a,b)$ Toeplitz, so
Toeplitz--FFT is a fast algorithm for this rigid slice of our family.
This is the exact baseline for the fixed-grid case; learned anchors ask for
a different primitive.

\textbf{Sketches and random features.}
CountSketch and related sketches~\citep{charikar2002finding,charikar2004finding,cormode2005improved}
compress coordinates through fixed or randomized routing, while Random
Fourier Features~\citep{rahimi2007random} approximate shift-invariant
kernels through Monte Carlo spectral features and underlie several linear
attention methods~\citep{choromanski2021rethinking,katharopoulos2020transformers,peng2021rfa}.
\our{} connects to both: CountSketch-like maps appear as zero-temperature
limits (App.~\ref{app:expressivity}), while RFF samples the Laplace kernel that \our{} computes exactly.

\textbf{Low-rank and sparse neural operators.}
Nystr\"om methods~\citep{drineas2005nystrom,gittens2016revisiting,williams2001using}
and low-rank adapters such as LoRA~\citep{hu2022lora} reduce cost by
introducing an explicit rank bottleneck. Sparse neural methods
\citep{frankle2019lottery,frantar2023sparsegpt,han2015learning,sun2023wanda}
reduce parameter count, but practical speedups depend on sparsity pattern
and hardware support. \our{} instead keeps a dense interaction implicit:
parameters live in the anchors, and computation proceeds through regular
sorted scans.

\textbf{Structured matrices and neural operators.}
Fast matrix--vector products and reduced parameter counts are established
benefits of structured matrices. Lingsch et al.~\citep{laplexcr_lingsch2024beyond}
extend Fourier-based neural operators to irregular sampling locations while
retaining a prescribed Fourier basis. Zhao et al.~\citep{laplexcr_zhao2017theoretical}
study neural networks whose weights have low displacement rank relative to
specified displacement operators. Aluvihare et al.~\citep{laplexcr_aluvihare2025lowcomplexity}
use structure imposed by a delay-Vandermonde factorization for wideband
beamforming. Semiseparable Gaussian-process solvers give exact linear-time algebra for
exponential kernels on irregular 1-D inputs~\citep{laplexcr_foremanmackey2017celerite}. These approaches learn within a prescribed algebraic family.
\our{} also fixes a kernel family; its contribution is to expose the
coordinate anchors themselves as trainable parameters and compute their
gradients exactly with the same near-linear asymptotic scaling as application
of the operator.

\section{Exact Laplace algebra: matvecs and Gram operators}
\label{sec:theory}

The representational picture would be much less interesting if every use of
$\our(a,b)$ required forming the dense matrix. The algorithmic fact that
makes the construction useful is that one-dimensional Laplace interactions
are dense but ordered. After sorting anchors, each side of an anchor sees
an exponentially rescaled prefix or suffix sum. This turns a dense kernel
operator into a scan computation.
The consequence is the central computational promise of \our{}: we can work
with dense, trainable kernel interactions at dimensions where the matrix
itself is not an object one can store. Matrix--vector products become
near-linear; weighted Gram matrices keep only the unavoidable quadratic
output cost.

\paragraph{Overview of the approach.}
We first show the symmetric case $a=b$, where sorted anchors reduce the
operator exactly to prefix and suffix scans. We then handle the rectangular
case $a\neq b$ by bucketing the $b$-anchors into intervals induced by
sorted $a$ and aggregating their contributions before applying the same
scan primitive. The weighted Gram decomposition follows the same logic:
the intermediate anchors are summarized once, and the full Gram matrix is
assembled from ordered cumulative statistics. All reductions are
differentiable except for the usual piecewise-constant sorting order, as in
sorting-based neural operators.

\paragraph{Symmetric case: a prefix-scan primitive.}
For $a=b$, the kernel matrix is
\[
A=\our(a) = [\exp(-|a_i - a_j|)]_{ij}.
\]
Assume $a_1 \le a_2 \le \cdots \le a_n$. Then
\[
A_{ij}=
\begin{cases}
\exp(a_j-a_i), & i \ge j, \\[4pt]
\exp(a_i-a_j), & i < j .
\end{cases}
\]
Decompose the matrix into lower and upper triangular parts:
$Ax = \mathrm{tril}(A)x + \mathrm{triu}(A,1)x$, with the diagonal in $\mathrm{tril}$.
For the lower triangular part,
\[
y_i = \sum_{j \le i} \exp(a_j - a_i)\, x_j.
\]
This satisfies the recurrence
\[
y_{i+1}
= \sum_{j \le i+1} \exp(a_j - a_{i+1})\, x_j
= x_{i+1} + \exp(a_i - a_{i+1}) \sum_{j \le i} \exp(a_j - a_i)\, x_j,
\]
or $y_{i+1} = x_{i+1} + \exp(a_i-a_{i+1})y_i$. A prefix scan evaluates
this recurrence in $\mathcal{O}(n)$ work and $\mathcal{O}(\log n)$
parallel depth; a reverse scan handles the upper triangular part. Thus
sorted symmetric multiplication is linear in time and memory, and sorting
raises the total cost to $\mathcal{O}(n\log n)$.

\paragraph{General case: exact aggregation over intervals.}
The asymmetric case is where \our{} departs from classical Toeplitz
algebra. The anchors $b$ need not lie on the grid of $a$; they can be
irregular, batched, or learned. Each $b_j$ falls into one interval between
neighbouring sorted $a$-anchors. Inside that interval, its contribution to
all points on the left and all points on the right has a separable
exponential form (see~Fig.\ref{fig:thm1}). We therefore aggregate the local contributions into two
vectors and propagate them through the symmetric scan primitive. Proofs of
the matrix--vector reductions are in App.~\ref{app:matvec-proofs}; the
Gram proof is in App.~\ref{app:gram-proof}.

\begin{theorem}\label{thm:1}
Assume $a_1 \leq \cdots \leq a_n$, let $b\in\mathbb{R}^k$, and extend with $a_0 = -\infty$, $a_{n+1} = +\infty$.\footnote{A convention: $a_{n+1}$ cancels in $e^{a_i-a_{n+1}}e^{a_{n+1}-b_j}$, so any finite $a_{n+1}>\max_j b_j$ works.}
For $i = 1, \ldots, n+1$ define the intervals
\[
P_i = [a_{i-1}, a_i).
\]
Given $x\in\mathbb{R}^k$, define vectors $u,v \in \mathbb{R}^{n+1}$ by
\[
u_i = \sum_{j : b_j \in P_i} x_j\, e^{\,a_i - b_j},
\qquad
v_i = \sum_{j : b_j \in P_i} x_j\, e^{\,b_j - a_i}.
\]
Then
\[
\our(a,b)x
= \bigl(\mathrm{triu}(K,1)\, u + \mathrm{tril}(K)\, v\bigr)_{1:n},
\]
where $K = \our(\bar a)$ for $\bar a = (a_1, \ldots, a_{n+1})$, and $(z)_{1:n}$ denotes the first $n$ entries of $z \in \mathbb{R}^{n+1}$.
\end{theorem}

\begin{figure}[tbp]
    \centering
    \def\svgwidth{\linewidth}
    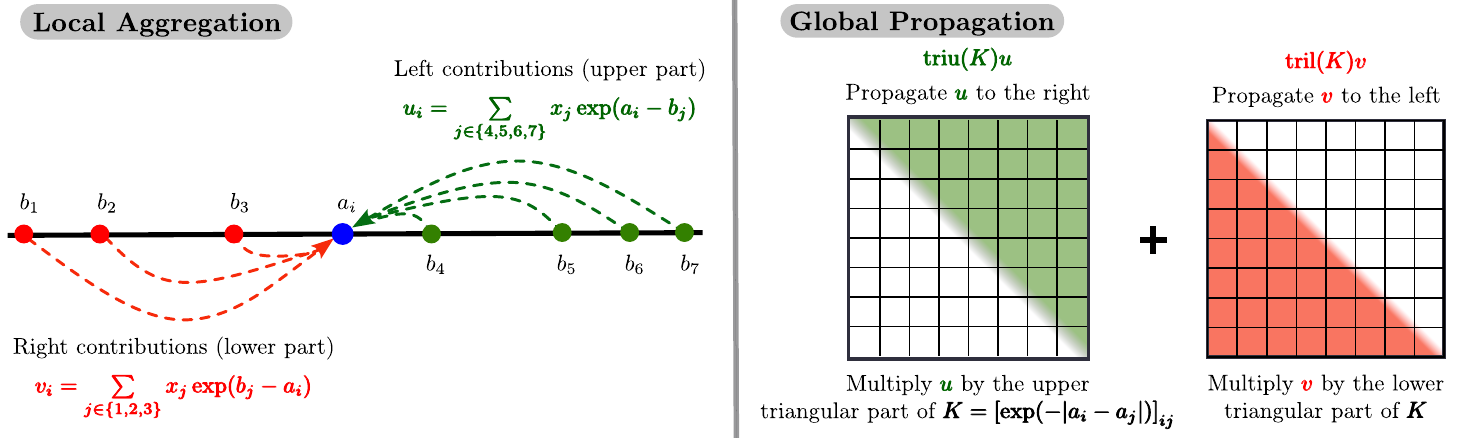
    \caption{Illustration of Theorem~\ref{thm:1}: exact decomposition of the general Laplace kernel matrix--vector product. The vector $b$ is partitioned into intervals defined by the sorted $a$, enabling local aggregation of contributions into vectors $u$ and $v$. These aggregated terms are then propagated globally using the upper and lower triangular parts of the symmetric kernel $K$, yielding an efficient and structured computation of $\our(a,b)x$ without explicitly forming the dense matrix.}
    \label{fig:thm1}
\end{figure}

\paragraph{Complexity.}
If sorted anchors and their interval assignment are already available, local
aggregation costs $\mathcal{O}(k)$, and the two triangular scan products
cost $\mathcal{O}(n)$. Sorting gives
$\mathcal{O}(n\log n+k\log n)$ time; storage stays linear. In training, the sorted order can be reused across
batches, and the dual reduction in App.~\ref{app:matvec-proofs} uses the shorter axis when $n$ and $k$ are imbalanced.

\paragraph{Computation of Gram matrix.}
The same order structure also gives exact weighted Gram matrices. This is
important because many downstream algorithms do not only apply a kernel
operator; they need expressions of the form
\[
M = A \operatorname{diag}(D) A^\top,
\quad
A=\our(a,b),
\quad D \in \mathbb{R}^k_+ .
\]
Such matrices appear in attention, kernel methods, Woodbury-accelerated
Gaussian likelihoods, and Fisher or Gauss--Newton curvature approximations
\citep{katharopoulos2020transformers,martens2020new,rahimi2007random,rasmussen2006gaussian,vaswani2017attention,williams2001using}.
We use this decomposition for Woodbury-evaluated Gaussian likelihoods at $n\approx3\cdot10^6$
(Sec.~\ref{sec:experiments-div2k}, App.~\ref{sec:density-model-app}) and for curvature in a second-order optimizer (App.~\ref{app:second-order}).
A naive computation pays $\mathcal{O}(n^2k)$. For the Laplace kernel, the
middle anchor set can again be summarized by interval-wise statistics,
after which every entry of $M$ is recovered exactly from scan products and
prefix sums.

We work in the same assumptions as in the previous theorem.

\begin{theorem}[Exact Gram decomposition]\label{thm:gram}
We define $u,v,w \in \R^{n+1}$ by
\[
u_i = \sum_{b_t \in P_i} e^{2(b_t - a_i)} D_t, \quad
v_i = \sum_{b_t \in P_{i}} e^{2(a_i - b_t)} D_t, \quad
w_i = \sum_{b_t \in P_i} D_t.
\]
Put
$U=\tril(K) u$, $V = \triu(K,1) v$ for $K=\our(2\bar a)$. We also define $W$ by $W_i=\sum_{j \leq i} w_j$.

Then for $n \ge i \ge j$, the entries of
$M = A \operatorname{diag}(D) A^\top$
are given by
\[
M_{ij}
=
\exp(a_j - a_i)
\Bigl(
U_j
+
(W_i-W_j)
+
V_i
\Bigr),
\]
with symmetric extension $M_{ji} = M_{ij}$.
\end{theorem}

Proof in App.~\ref{app:gram-proof}.  The cost is $\mathcal{O}(n^2 + k \log n)$: the $n^2$ term is the output size of the full Gram matrix, while the dependence on the intermediate anchor set remains near-linear  ($k$ points bucketed into the $n$ sorted $a$-anchors by binary search, i.e.\ $\mathcal{O}(k\log n)$; App.~\ref{app:gram-proof}).

\paragraph{Phased Laplace kernels.}
For neural layers, phase modulation can add flexibility while preserving the
fast algebra. Given phase anchors\space%
 $\phi\in\mathbb{R}^{n}, \psi \in \mathbb{R}^k$ we define
\[
  \our(a,\phi;b,\psi)_{ij} \;=\; \exp\!\bigl(-|a_i - b_j|\bigr) \cdot \cos(\phi_i - \psi_j),
\]
so the standard kernel is the special case $\phi\equiv\psi\equiv 0$. Angle-sum
identities reduce each phased matrix--vector or weighted Gram operation to
a constant number of non-phased \our{} calls plus $\mathcal{O}(n+k)$ pointwise
phase modulation. The phases are ordinary trainable parameters, and their
gradients flow through the same exact scan computation.

For a positive temperature (kernel length scale) $t$, we write
\[
\our_{t}(a,\phi;b,\psi)_{ij}
=\exp\!\left(-\frac{|a_i-b_j|}{t}\right)\cos(\phi_i-\psi_j).
\]
The definitions above use $t=1$. Other positive scales are handled by
rescaling both coordinate-anchor sets by $1/t$, so the same exact
algorithms apply; phase anchors are not rescaled. Smaller $t$ gives sharper
interactions, as in the zero-temperature limits in
Sec.~\ref{sec:learnable-routing} and App.~\ref{app:expressivity}.

\section{Scaling consequences: GPU primitives and high-dimensional Gramians}
\label{sec:scaling-consequences}

The algebra above is meant to change the feasible regime, not merely the
notation. We therefore separate the consequences from the derivation. The
first test asks whether the scan primitive behaves like a fast transform on
hardware; the second asks whether exact weighted Gramians can support a
meaningful Gaussian model in millions of dimensions.

\subsection{Primitive scaling on GPU}
\label{sec:kernel-algebra-bench}

The systems question is whether the scan algebra behaves like a
fast transform on real hardware. We measure matrix--vector multiplication
$XK(a,b)$ with $X\in\mathbb{R}^{B\times n}$, $K(a,b)=\our(a,b)$. The dense
baseline explicitly materializes the $n\times n$ kernel matrix and calls
\texttt{torch.matmul}. All experiments run on a single NVIDIA~H100 in
\texttt{float32}; full timing details, memory accounting, and OOM rules are
given in App.~\ref{sec:detailed-kernel-bench}.

\begin{figure}[ht]
    \centering
    \resizebox{\linewidth}{!}{\input{NeurIPS_2026/camera_ready/images/exp1.pgf}}
    \caption{\textbf{Forward time (left), forward+backward time (middle),
        and peak GPU memory of the forward call (right) for
        $X K(a, b)$ at batch size $B = 8$.}  The dense baseline runs
        until it no longer fits in GPU memory; red dots mark the last
        size that fits.  \our{} stops earlier, at the open square: at
        $n = 2^{28}$ the $B \times n$ input would hold $2^{31}$ elements,
        more than the 32-bit indexing of our CUDA kernels can address.
        This is a limit of the implementation, not of the method; open
        squares mark it in all figures.
        The explicit dense baseline (orange) stops at $n = 2^{16}$,
        while \our{} (blue) scales smoothly to $n = 2^{27}$.  At the
        largest size where both run, \our{} is more than two orders of
        magnitude faster, both forward and forward+backward, and uses
        over three orders of magnitude less peak memory.}
    \label{fig:kernel-algebra-bench}
\end{figure}

\paragraph{Scaling and memory.}
Figure~\ref{fig:kernel-algebra-bench} is the basic scalability test. The
dense baseline becomes infeasible once the $n\times n$ kernel no longer
fits in memory; \our{} continues as an implicit operator and scales smoothly
to problems more than three orders of magnitude larger ($n=2^{27}$ versus
$n=2^{16}$). At the largest size where both methods fit, \our{} is more than
two orders of magnitude faster on both forward and forward+backward
passes and uses over three orders of magnitude less peak memory.
This is the difference between storing a dense interaction and
computing it as a fast transform.

\paragraph{Numerical results.}
The scan is not merely faster than dense \texttt{float32}; it is also more
accurate against a dense \texttt{float64} reference. In the tested range, the error of dense matmul grows with $n$,
whereas the \our{} scan, which keeps exponential scale factors separate from
mantissas, stays at the level of \texttt{float32} rounding (App.~\ref{sec:app-exp-accuracy}).
The same pattern holds beyond the square forward pass. The weighted Gram
operator yields an asymptotic $\Theta(n^{2}/\log n)$ speed-up over dense
construction for large $k$ (App.~\ref{sec:app-exp-gram}); asymmetric problems dispatch to
the shorter axis (App.~\ref{sec:app-exp-asymm}); and even the pure-PyTorch
CPU reference remains far faster than dense (App.~\ref{sec:app-exp-cpu}).

\subsection{Density modeling on DIV2K at dim$\pmb{\approx 3 \cdot 10^6}$}
\label{sec:experiments-div2k}

To stress-test the Gram primitive at an extremely high data dimension, we
fit Gaussian density models to flattened DIV2K-$1024$ images
($800$ training, $100$ validation), where each example is a
$3{,}145{,}728$-dimensional vector ($1024\!\times\!1024\!\times\!3$),
again with no patches, convolutions, or positional encoding.\footnote{The
DIV2K protocol, an ImageNette-320 benchmark and matched-budget
permuted-coordinate comparisons are in App.~\ref{sec:density-model-app};
a raster-order ImageNette comparison with FFT-circulant and
Kronecker covariances and MAP denoising is in
App.~\ref{app:density-downstream}.%
} At this
scale even writing down an empirical covariance is hopeless (it would
require $\sim\!10^{13}$ parameters); the question is whether a
structured exact Gram operator still carries meaningful geometry once
the data dimension reaches the multi-million regime.

\paragraph{Setup.}
\begin{wrapfigure}{R}{0.5\textwidth}
\centering
\begin{subfigure}[b]{0.32\linewidth}
    \centering
    \includegraphics[width=\linewidth]{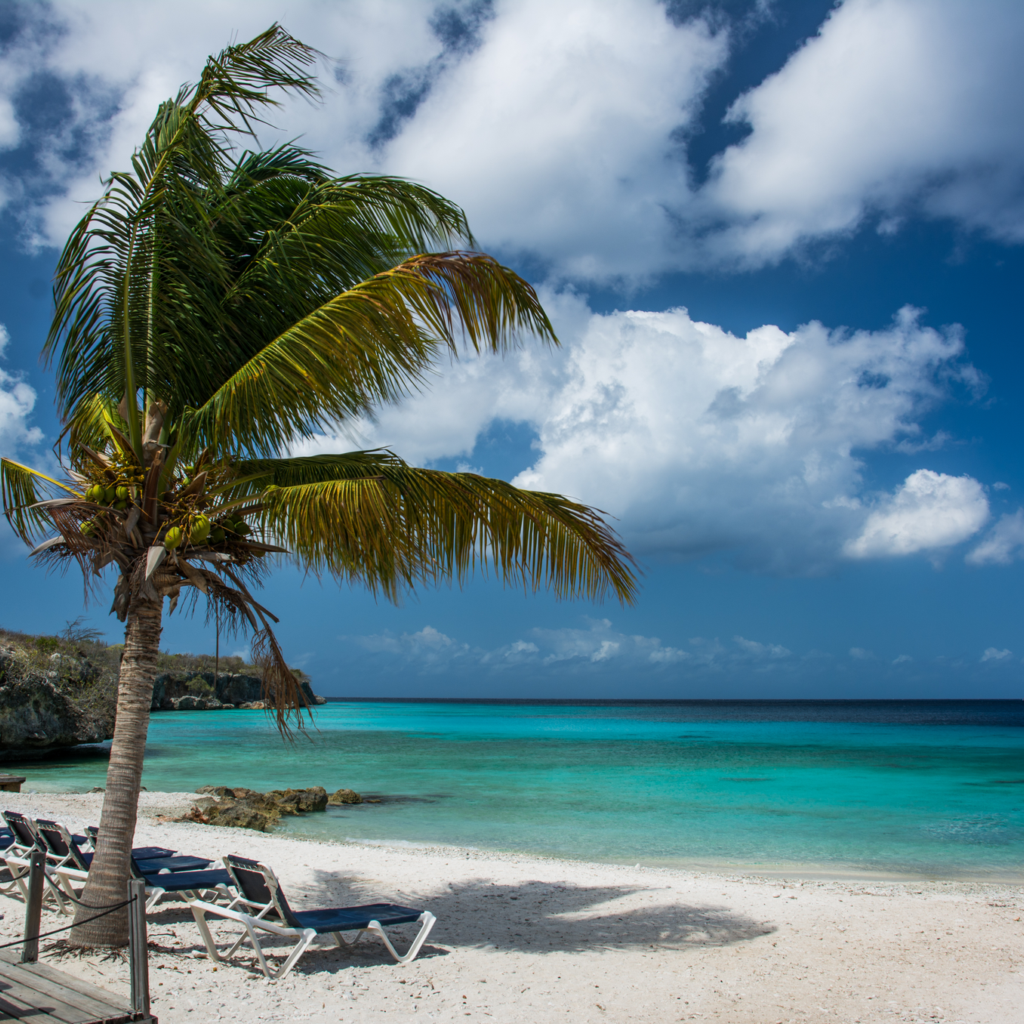}
    \caption*{\scriptsize input}
\end{subfigure}\hfill
\begin{subfigure}[b]{0.32\linewidth}
    \centering
    \includegraphics[width=\linewidth]{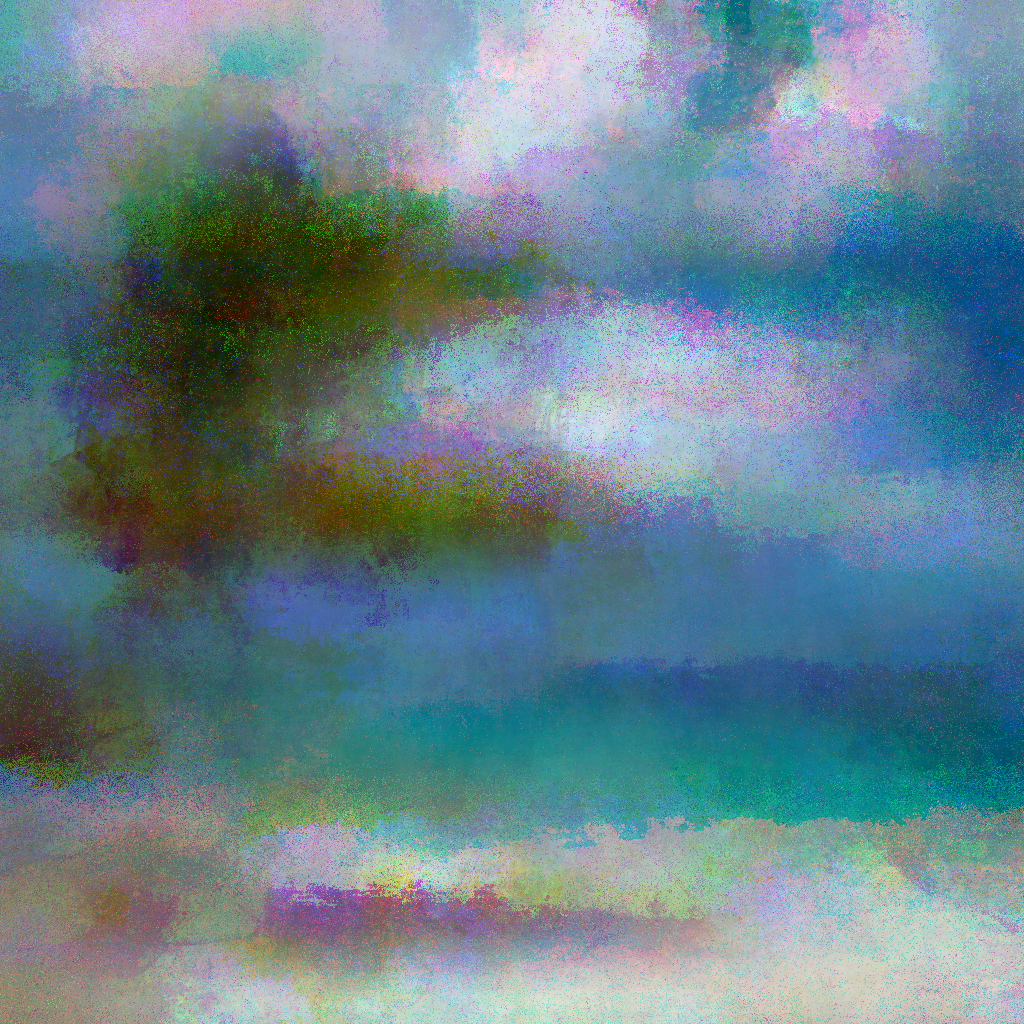}
    \caption*{\scriptsize \our{} MAP}
\end{subfigure}\hfill
\begin{subfigure}[b]{0.32\linewidth}
    \centering
    \includegraphics[width=\linewidth]{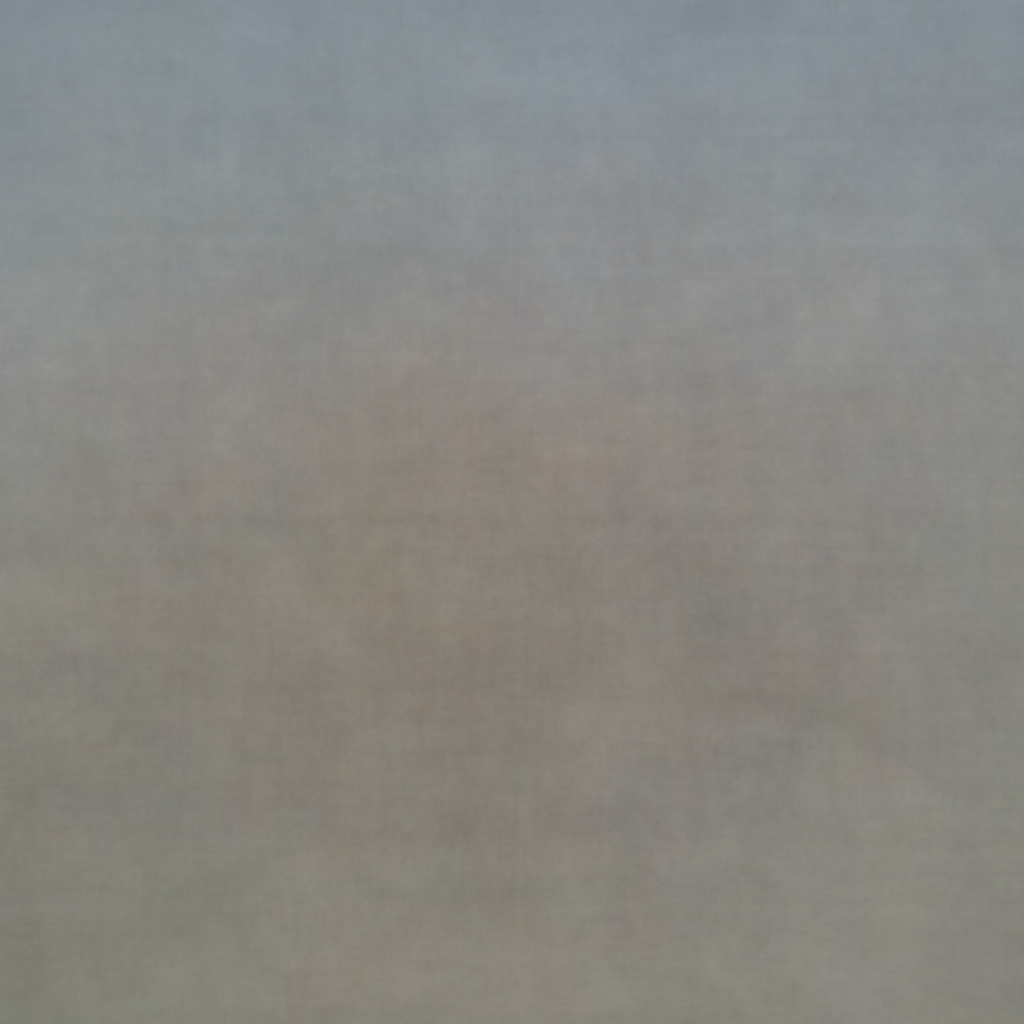}
    \caption*{\scriptsize low-rank MAP}
\end{subfigure}

\vspace{0.55ex}
\begin{subfigure}[b]{0.49\linewidth}
    \centering
    \includegraphics[width=\linewidth]{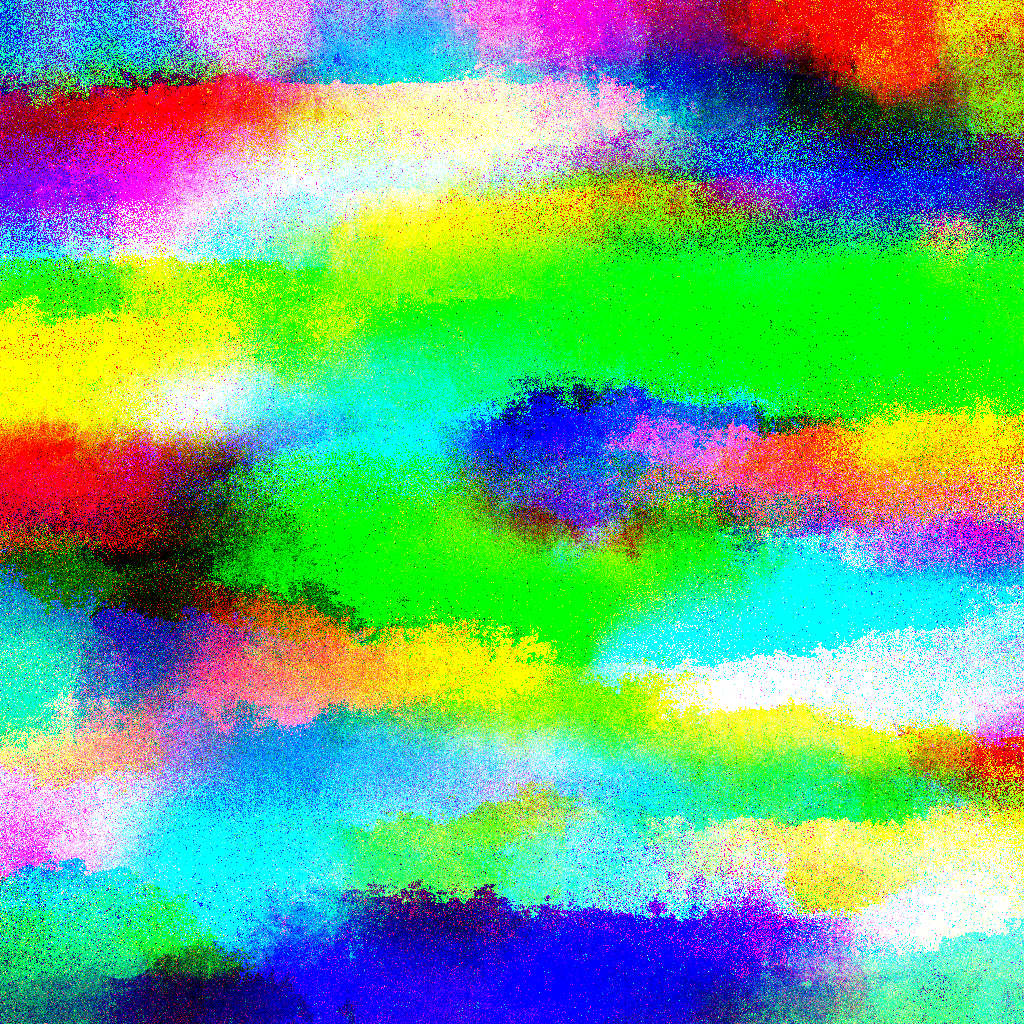}
    \caption*{\scriptsize \our{} sample\\\scriptsize LL $\mathbf{1{,}266{,}679}$}
\end{subfigure}\hfill
\begin{subfigure}[b]{0.49\linewidth}
    \centering
    \includegraphics[width=\linewidth]{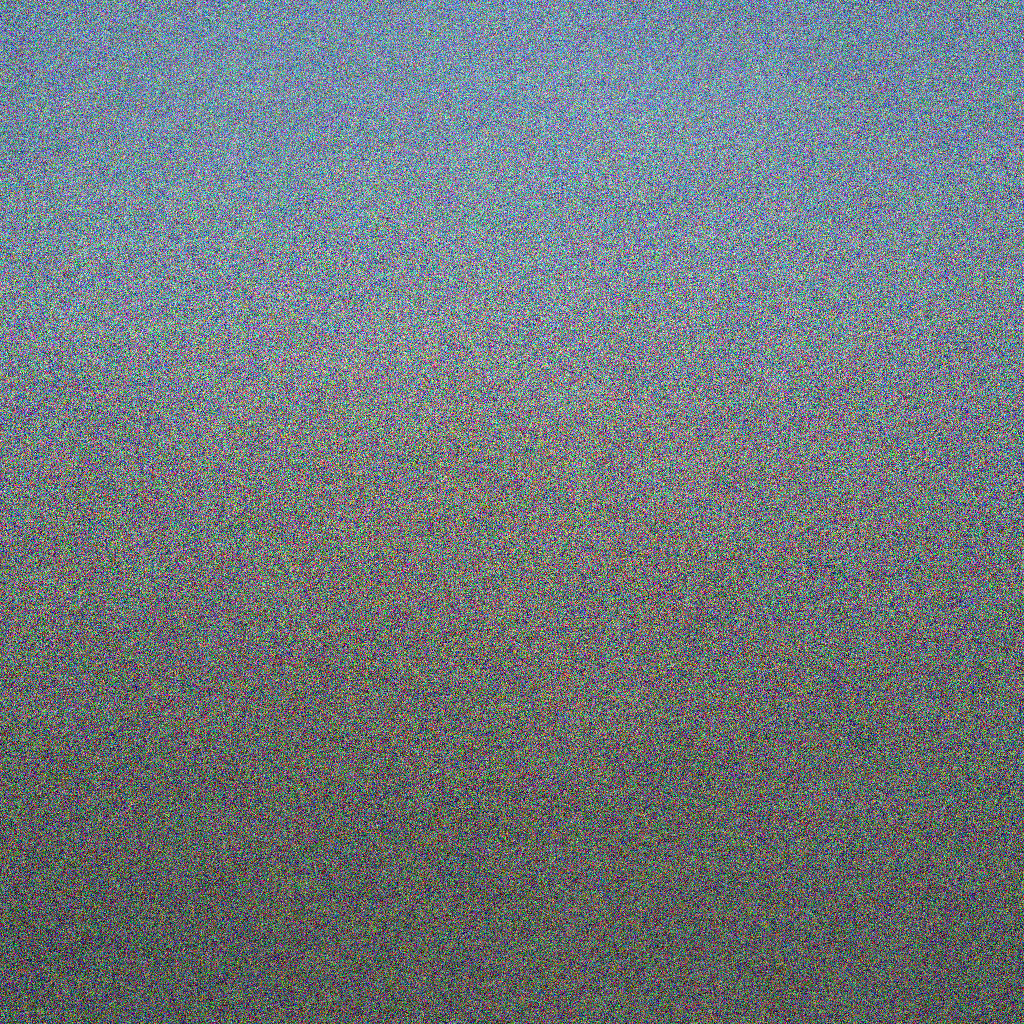}
    \caption*{\scriptsize low-rank sample\\\scriptsize LL $282{,}392$}
\end{subfigure}
\caption{DIV2K density at $n\!=\!3{,}145{,}728$. Matched
$\sim\!14$M-parameter budget: \our{} (rank $1000$) vs.\ dense low-rank
(rank $4$). \textbf{Top}: input and MAP reconstruction; \textbf{bottom}:
unconditional samples with validation log-likelihood. \our{} exceeds the dense baseline by $984{,}287$ nats in test LL.}
\label{fig:div2k-samples}
\vspace{-\baselineskip}
\end{wrapfigure}
We use a low-rank-plus-diagonal Gaussian template
\[
x \;=\; m \;+\; F\,z \;+\; d\odot\varepsilon,
\]
with $z\sim\mathcal{N}(0,I_k)$ and $\varepsilon\sim\mathcal{N}(0,I_n)$.
At a matched parameter budget we compare two parameterizations of
$F$: \emph{(i)} a phased \our{} factor $F=\mathrm{diag}(w)AL$ with $A$ a
width-$1000$ Laplace kernel and $L$ a lower-triangular Cholesky factor
($13.6$M trainable parameters); \emph{(ii)} a freely trained dense
low-rank factor $F\in\mathbb{R}^{n\times k_{\mathrm{low-rank}}}$
($15.7$M trainable parameters at $k_{\mathrm{low-rank}}=4$). Under this
budget the dense baseline can only afford rank $\sim\!4$, whereas
\our{} learns a rank-$1000$ structured factor. Both models are trained
for $20$ epochs with Adam and identical optimizer settings.

\paragraph{Takeaway.}
Figure~\ref{fig:div2k-samples} gives the most visual demonstration of
what the anchors learn. A \our{} covariance factor trained only on flat
vectors discovers coherent spatial color structure: samples have contiguous
regions, and the MAP reconstruction preserves the palm, beach and clouds silhouettes. The matched-budget dense low-rank baseline collapses
to nearly uniform noise because its rank budget is too small to represent
these correlations. With no convolution or positional encoding, this organization is
carried by the learned \our{} factor.

\section{Learnable routing: CountSketch limits and structured projections}
\label{sec:learnable-routing}

We first look at what the anchors represent. In \our{}, a matrix entry is
large when two learned coordinates are close and small when they are far
apart. At finite temperature this is a smooth matching rule; as the
temperature goes to zero it becomes discrete equality routing. Thus the
same object that behaves like a dense kernel in the forward pass also
contains permutations, hashes, and sparse routing maps on its boundary.

For the temperature-scaled kernel
\[
\our_t(a,b)=[\exp(-|a_i-b_j|/t)]_{ij},
\]
the limit $t\to0$ becomes an equality-routing matrix. Suitable anchor
multiplicities recover permutations and CountSketch-like maps, while
diagonal scalings and sparse replication recover arbitrary rectangular
dense matrices (App.~\ref{app:expressivity}). The important difference is
that \our{} does not ask us to choose the routing map before training: the
geometry is a continuous parameter of the model.

After the fast algebra is in place, the next question is what the anchors can
represent. If anchor geometry is a learnable version of sketching, then it
should preserve useful signal better than fixed
sketches when inserted as a projection from a very high-dimensional input.
We test this on flattened ImageNet10, where the input dimension is already
$150{,}528$.

\paragraph{Structured input projections versus CountSketch}
\label{sec:logistic-countsketch}

The test is intentionally unforgiving: each ImageNet10 image is flattened
to $d=224\!\times\!224\!\times\!3=150,528$ coordinates, and the classifier sees
only a projected vector of dimension $m\in\{25,100,500,1000\}$. In a frozen
regime, the projection is fixed and only the logistic head is trained, so
the projection is judged by how much signal it preserves. In a trainable
regime, \our{} anchors are optimized end-to-end and compared with an
unconstrained dense projection at the same output dimension. Frozen
competitors are CountSketch~\citep{charikar2002finding}, PCA, and Gaussian
random projections; the full $\tau$-sweep is in
App.~\ref{sec:logreg-frozen-app}.
\begin{wraptable}{R}{0.55\textwidth}
    \small
    \centering
    \caption{Logistic regression on ImageNet10.  Mean validation
        accuracy (\%) over $5$ seeds.  Cols.~2--5: frozen projection
        (only the linear head is trained); cols.~6--7: projection
        trained jointly with the head.  Best in each regime is shown
        in \textit{italic} (frozen) and \textbf{bold} (trainable).
        CS = CountSketch, RP = Random Projections.  Standard
        deviations in App.~\ref{sec:logreg-frozen-app}.}
    \label{tab:logreg-merged}
    \begin{tabular}{@{}r@{\quad}c@{\,}c@{\;\,}c@{\;\,}c@{\quad}c@{\;\,}c@{}}
        \toprule
        & \multicolumn{4}{c}{\textbf{frozen projection}}
        & \multicolumn{2}{c}{\textbf{trainable}} \\
        \cmidrule(lr){2-5}\cmidrule(lr){6-7}
        $m$ & CS & \our{} & PCA & RP & \our{} & Dense \\
        \midrule
          25 & 26.9 & 28.0 & \textit{34.8} & 28.0 & \textbf{43.5} & \textbf{43.5} \\
         100 & 31.0 & 31.8 & \textit{36.5} & 31.7 & \textbf{43.6} & 43.0 \\
         500 & 31.1 & 30.9 & \textit{32.4} & 30.2 & \textbf{43.8} & 43.0 \\
        1000 & 26.0 & \textit{31.3} & 27.2 & 25.2 & \textbf{44.2} & 43.8 \\
        \bottomrule
    \end{tabular}
\end{wraptable}
\paragraph{Results.}
Table~\ref{tab:logreg-merged} shows the expected hierarchy:
PCA is strongest at small
$m$ because it is fitted to the training set, while fixed random maps lose
signal quickly. Among data-agnostic maps--CountSketch, Gaussian random
projections, and frozen \our{}--\our{} is the most stable and becomes best
at the largest $m$. Even before learning the anchors, the smooth geometry
is a richer alternative to a fixed hash.
\par
The trainable columns are the more revealing part. Once anchors are learned
for the task, \our{} is statistically
indistinguishable from the unconstrained dense projection at every $m$ while
using $\Theta(m+d)$ rather than $\Theta(md)$ parameters: the differences are
$\le 0.84$~pp and are not significant over $5$ seeds\space%
(mean\,$\pm$\,std and Welch / Mann--Whitney tests in
App.~\ref{sec:logreg-frozen-app}, all $p\ge0.146$)%
. The operator does not merely compress the input; it learns a task-adapted
coordinate matching in a space where a dense projection is expensive.

\section{Restricted views: fixed grids and sampled features}
\label{sec:restricted-views}

The comparison with Toeplitz--FFT and Random Fourier Features is useful
because these methods are not distant competitors. They are what the same
Laplace interaction becomes after two different restrictions. If the
anchors are frozen to a uniform grid, the matrix is Toeplitz and FFT gives
an exact specialized solver. If the anchors are irregular but the kernel is
replaced by a finite spectral expansion, we obtain RFF. \our{} sits between
these views: the anchors can move, the kernel values remain exact, and the
gradients with respect to the anchors are still available.

\paragraph{Toeplitz--FFT: the fixed-grid slice}
\label{subsec:kerlap-fft-bench}

On the uniform grid $a_i=b_i=i\Delta$, the Laplace matrix
$[\exp(-|a_i-b_j|)]_{ij}$ is Toeplitz. Toeplitz--FFT is therefore the right
algorithm for this slice of the family. The comparison does not ask whether a
general learned-anchor primitive beats a solver hand-built for a fixed grid,
but a narrower question: when we leave the grid,
do we only pay a constant factor, or do we lose the fast-transform regime?

\begin{wrapfigure}{R}{0.55\textwidth}
    \centering
    \resizebox{\linewidth}{!}{\input{NeurIPS_2026/camera_ready/images/exp12a.pgf}}
    \caption{%
        GPU runtime for a single matrix--vector product on
        irregular anchors ($a,b\sim\mathcal{N}(0,1)$):
        \our{} (blue), Toeplitz--FFT on its fixed-grid specialization
        (orange), and a single cuFFT call as a hardware reference
        (green). All three curves follow the $\mathcal{O}(n\log n)$
        reference.
        Red dots mark the last size that fits in GPU memory; the open
        square marks the 32-bit indexing limit of the \our{} kernels
        (Fig.~\ref{fig:kernel-algebra-bench}).
    }
    \label{fig:kerlap-vs-fft}
    \vspace{-1.5\baselineskip}
\end{wrapfigure}
Figure~\ref{fig:kerlap-vs-fft} shows the answer. \our{}, Toeplitz--FFT,
and a raw cuFFT call lie on parallel $\mathcal{O}(n\log n)$ curves. The
extra constant is the price of sorting and scanning arbitrary anchors; what
it buys is a different object: not convolution on a prescribed grid, but a
dense kernel operator whose coordinate system can be learned. The numerical
picture is also favorable. On the very Toeplitz problem where the FFT
baseline is exactly applicable, Fig.~\ref{fig:app_kerlap-vs-fft}(b) (App.~\ref{sec:app-fft-gpu}) shows that
\our{} has lower \texttt{float32} error than both dense matmul and
Toeplitz--FFT. Thus the learned-anchor extension does not trade away
stability for flexibility.

The categorical difference is differentiability with respect to individual anchors.
Toeplitz--FFT differentiates through the input vector and shared parameters such as the grid
spacing, but individual anchors cannot move without breaking the Toeplitz structure.
\our{} differentiates through $x$, $a$, and $b$. For models that learn
coordinate systems, kernel nodes, attention anchors, or covariance factors,
this is the difference between tuning values on a grid and learning the
grid itself. CPU results, the differentiability table, and comparisons to
NUFFT and FMM are in App.~\ref{sec:detailed-fft-bench}.

\paragraph{RFF: the sampled surrogate}
\label{subsec:kerlap-vs-rff}
RFF~\citep{rahimi2007random} removes the grid restriction, but pays by
turning the exact kernel into a Monte Carlo feature map. For
$k(u-v)=\exp(-|u-v|)$, Bochner's theorem gives
$\widehat{k}_D(u, v) = \tfrac{2}{D} \sum_{d=1}^{D}
\cos(w_d u + c_d)\cos(w_d v + c_d)$ with Cauchy frequencies
$p(w)=[\pi(1+w^2)]^{-1}$. This is an approximation to the same Laplace
kernel that \our{} applies exactly. The feature count $D$ therefore plays
the role of an accuracy knob, while \our{} has no sampled feature dimension.

We test the favorable case for RFF: irregular anchors
$a,b\sim\mathcal{N}(0,1)$, $n=2^{16}$, $B=1$, and a shift-invariant
kernel with a known spectral density. Figure~\ref{fig:kerlap-vs-rff}
shows why this approximation is not a substitute for an exact learned
kernel primitive. RFF is slower than \our{} at every $D$; where it comes
closest (small $D$), its value error is still above $30\%$. Reducing the
error below $10\%$ requires $D\ge1024$, where RFF is over an order of
magnitude slower, and the gap grows linearly with $D$.
\begin{figure}[H]
    \centering
    \resizebox{\linewidth}{!}{\input{NeurIPS_2026/camera_ready/images/rff_val.pgf}}
    \caption{%
        \textbf{RFF vs.\ \our{}.}
        \textbf{(Left)} Runtime as a function of $D$. RFF is linear in
        the feature count; \our{} is independent of $D$ (dashed).
        RFF is slower than \our{} at every $D$.
        \textbf{(Right)} Relative $\ell_2$ approximation
        error of RFF (median and $[\min,\max]$ band over $10$ seeds) with
        the $1/\sqrt{D}$ reference. Where RFF comes closest in runtime
        (small $D$), its median error is still above $30\%$; below $10\%$
        error requires $D\ge1024$, where RFF is more than an order of
        magnitude slower than \our{}.  Red dots
        mark the largest $D$ that fits in GPU memory.
    }
    \label{fig:kerlap-vs-rff}
\end{figure}
Anchor gradients are worse. Differentiating the RFF estimator
with respect to anchors multiplies each term by the sampled Cauchy
frequency $w_d$, whose second moment is infinite. The resulting anchor
gradient estimator has an infinite second moment; empirically, the $b$-anchor
gradient error does not reliably decrease with $D$
(App.~\ref{sec:app-rff-grad}). For fixed features this may be acceptable.
For learnable kernels it is the central failure mode: RFF gives a sampled
view of the Laplace kernel, while \our{} keeps exact values and exact
anchor gradients at the same $\mathcal{O}(n\log n)$ cost.

\section{Learned anchor geometry in neural heads}
\label{sec:learnable-layers}

The projection experiment tests \our{} as a high-dimensional compression
map, and the density experiment tests it as a Gram/covariance primitive. We
also want a familiar neural-network interface where the operator must
replace a standard dense matrix. 

A classification head provides a clean testbed for whether anchor geometry
can replace dense class--feature interactions. We freeze the pretrained
backbone and train only the final head. For
Swin-T~\citep{liu2021swin}, the original head is a single linear map
$W \in \mathbb{R}^{c \times f}$ with $f=768$ and $c=1000$; for
MaxViT-T~\citep{tu2022maxvit}, it is the linear component of a
two-layer MLP.
\begin{wraptable}{R}{0.4\textwidth}
\centering
\caption{\small ImageNet classification with structured heads on
Swin-T and MaxViT-T. Mean top-1 validation accuracy over 5 runs. Best at each rank in \textbf{bold}.}
\label{tab:heads_results}

\footnotesize

\resizebox{\linewidth}{!}{%
\begin{tabular}{@{}l@{\quad}c@{\;\;}c@{\qquad}c@{\;\;}c@{}}
\toprule
 & \multicolumn{2}{c}{Swin-T} & \multicolumn{2}{c}{MaxViT-T} \\
\cmidrule(l{0pt}r{17pt}){2-3}
\cmidrule(l{0pt}r{0pt}){4-5}

Head & Params & ACC & Params & ACC \\

\cmidrule[1pt](l{0pt}r{0pt}){1-5}

Pretrained
    & $768$\,K & $81.5$
    & $774$\,K & $83.7$ \\

\cmidrule(l{2pt}r{2pt}){1-5}

FullRank
    & $768$\,K & 81.37
    & $774$\,K & 83.37 \\

\cmidrule(l{2pt}r{2pt}){1-5}

\our$^{(4)}$
    & $8.1$\,K & \textbf{54.04}
    & $11.7$\,K & \textbf{78.61} \\
LowRank$^{(4)}$
    & $8.1$\,K & 10.67
    & $11.7$\,K & 17.29 \\

\cmidrule(l{2pt}r{2pt}){1-5}

\our$^{(8)}$
    & $15.2$\,K & \textbf{71.72}
    & $21.8$\,K & \textbf{82.07}\\
LowRank$^{(8)}$
    & $15.1$\,K & 46.80
    & $21.8$\,K & 63.80\\

\cmidrule(l{2pt}r{2pt}){1-5}

\our$^{(16)}$
    & $29.3$\,K & \textbf{77.68}
    & $42.1$\,K & \textbf{82.65}\\
LowRank$^{(16)}$
    & $29.3$\,K & 67.87
    & $42.1$\,K & 78.19\\

\bottomrule
\end{tabular}}

\vspace{-33pt}
\end{wraptable}
As a full-parameter baseline, we train a randomly initialized dense
head. The \our{} head is parameterized as a weighted sum of $r$
Laplace-kernel operators,
\begin{equation*}
    \mathrm{LAPLEX}^{(r)}
    \;=\; \sum_{i=1}^{r} w_i\,L_i,
    \qquad
    L_i \in \mathbb{R}^{c \times f},
\end{equation*}
where each $L_i$ is a \our{} matrix parameterized by an independent anchor pair $(a_i,b_i)$ and $w_i$ are scalar mixing weights. As a low-parameter baseline, we compare against a
low-rank factorization
\begin{equation*}
    \mathrm{LowRank}^{(r)}
    = U V^{\!\top},
    \qquad
    U \in \mathbb{R}^{c \times r},
    \quad
    V \in \mathbb{R}^{f \times r}.
\end{equation*}
All heads are trained under the same optimization setup; full details
are provided in App.~\ref{sec:experiments-heads-app}.

\paragraph{Results.}
Table~\ref{tab:heads_results} shows a consistent separation between \our{}$^{(r)}$ and LowRank$^{(r)}$  across both backbones. At matched parameter budgets, \our{}$^{(r)}$ consistently achieves higher accuracy, with the largest gains at low rank. While LowRank$^{(r)}$ constrains representations through a shared low-dimensional bottleneck, \our{}$^{(r)}$ models class--feature interactions through a learned geometric kernel structure. Using only about $4$--$5\%$ of the original head parameters, \our{}$^{(16)}$ approaches full-rank performance within a few points. Full curves and training details are reported in App.~\ref{sec:experiments-heads-app}.

\section{Conclusions and Limitations}
\label{sec:conclusions}

The main message of this paper is that fast transforms do not have to live
on fixed coordinate systems. By parameterizing a dense operator through
learnable Laplace anchors, \our{} keeps exact matrix--vector computations, including Gram-operator products, in near-linear time and
linear memory, computes full weighted Gram matrices exactly in $\mathcal{O}(n^2+k\log n)$ time, while exposing the
geometry itself to gradient-based learning. Uniform-grid Toeplitz--FFT is a
rigid slice of this family; RFF is a Monte Carlo surrogate for the same
kernel; CountSketch-like routing appears as limiting cases.
\our{} connects these views with a single exact, differentiable operator.

The empirical results show why this matters. \our{} reaches scales where
dense kernels cannot be materialized, follows FFT-class scaling beyond the
uniform grid, improves \texttt{float32} accuracy over standard dense and
Toeplitz routes in the tested regime, and avoids the value/gradient
trade-offs of random features. As a learnable layer it gives compact
projections and classification heads; as a Gram primitive it enables
high-dimensional covariance models that preserve visible spatial structure
without convolutional bias. These are all instances of the same principle:
dense geometry can be learned without storing a dense matrix.

\textbf{Limitations.} The construction is deliberately narrow. It uses
one-dimensional Laplace anchors and sorting-based algorithms. Extending exactness to higher-dimensional anchor geometries, richer kernels, and
more general learned metrics is the main mathematical direction. A second
limitation is systems maturity: the implementation demonstrates the
algorithmic regime, but large-scale attention blocks, optimizers, and
generative models will require specialized kernels and careful architecture
design. The broader direction is to
make the geometry of fast matrix--vector
operators learnable.\space%

\section*{Acknowledgments}
The work of Jacek Tabor was supported by the Foundation for Polish Science (FNP) grant 'Centre for Credible AI' No. FENG.02.01-IP.05-0058/24. The work of Łukasz Struski was supported by the National Science Centre, Poland, grants no. 2023/49/B/ST6/01137. Some experiments were performed on servers purchased with funds from the flagship project entitled ``Artificial Intelligence Computing Center Core Facility'' from the DigiWorld Priority Research Area within the Excellence Initiative -- Research University program at Jagiellonian University in Kraków. We also gratefully acknowledge Polish high-performance computing infrastructure PLGrid (HPC Center: ACK Cyfronet AGH) for providing computer facilities and support within computational grant no. PLG/2025/018975.

\bibliographystyle{plain}
\bibliography{ref,revision}

\newpage

\appendix

\begin{center}
    \Large \textbf{Appendix}
\end{center}

\paragraph{Guide to the Appendix.}
The appendix extends the main paper by providing theoretical, algorithmic, and empirical details underlying LAPLEX. Section~\ref{app:expressivity} establishes expressivity, showing that LAPLEX recovers discrete routing (e.g., CountSketch) in limiting regimes and is universal for dense linear maps. Sections~\ref{app:matvec-proofs} and~\ref{app:gram-proof} present proofs of the fast scan-based algorithms for exact matrix--vector and weighted Gram computations, achieving near-linear complexity. Sections~\ref{sec:detailed-kernel-bench} and~\ref{sec:detailed-fft-bench} give a full empirical and systems-level evaluation, including GPU/CPU benchmarks and comparisons to FFT-based methods, highlighting both efficiency and differentiability advantages. Section~\ref{sec:detailed-rff} analyzes Random Fourier Features, demonstrating their limitations in value accuracy and especially gradient estimation compared to LAPLEX. Finally, Sections~\ref{sec:logreg-frozen-app}--\ref{sec:density-model-app} provide extended experimental results, covering projection tasks (Section~\ref{sec:logreg-frozen-app}), structured classification heads (Section~\ref{sec:experiments-heads-app}), and high-dimensional density modeling (Section~\ref{sec:density-model-app}), illustrating the practical benefits of learnable Laplace-kernel operators.

\section{Expressivity: CountSketch and dense matrices as limits of \our{}}
\label{app:expressivity}

The main text treats \our{} as a smooth, learnable fast transform. This
section identifies what lies on the boundary of that smooth family. For
the temperature-scaled Laplace kernel
\[
\our_t(a,b)\in\mathbb{R}^{m\times n},
\qquad
\bigl[\our_t(a,b)\bigr]_{ij}
\;=\;
\exp\!\left(-\frac{|a_i-b_j|}{t}\right),
\qquad t > 0,
\]
with $a\in\mathbb{R}^m$ and $b\in\mathbb{R}^n$, the limit $t\to0$ is an
equality-routing matrix. This immediately recovers permutations and
CountSketch-style hashing when anchors are chosen as discrete labels
(Prop.~\ref{prop:countsketch_appendix}). With diagonal scalings and sparse
replication, the same family is universal for arbitrary rectangular dense
matrices (Thm.~\ref{thm:rectangular_universality}). Thus, \our{} should not be viewed as a narrowly specialized matrix class, but rather as a continuous family interpolating between sparse routing structures and dense linear operators, while admitting fast exact computation throughout the smooth Laplace regime.

\subsection{Zero-temperature limit}
In the low-temperature limit, the kernel matrix converges to a binary
incidence matrix encoding the equality structure between the entries of
$a$ and~$b$, as stated in the following lemma.

\begin{lemma}[Zero-temperature limit]
\label{lem:zero_temp_appendix}
Let $a\in\mathbb{R}^m$ and $b\in\mathbb{R}^n$. Then, for every
$i\in\{1,\dots,m\}$ and $j\in\{1,\dots,n\}$,
\[
\lim_{t\to 0}\bigl[\our_t(a,b)\bigr]_{ij}
=
\mathbf{1}\{a_i=b_j\}.
\]
Equivalently, $\our_t(a,b)$ converges entrywise to the matching matrix
$M(a,b)$ defined by
\[
M(a,b)_{ij}=\mathbf{1}\{a_i=b_j\}.
\]
\end{lemma}

\begin{proof}
If $a_i = b_j$ then $\exp(-|a_i-b_j|/t) = 1$ for every $t > 0$.
If $a_i \neq b_j$ then $|a_i - b_j| = c > 0$ and
$\exp(-c/t) \to 0$ as $t \to 0$.
\end{proof}

\subsection{CountSketch as a special case}
In the low-temperature limit, the kernel recovers a hashing-based
aggregation scheme, as formalized in the following proposition. 

\begin{proposition}[CountSketch realization]
\label{prop:countsketch_appendix}
Let $n$ denote the input dimension and $m$ the sketch dimension,
typically with $m \ll n$. Let $h:\{1,\dots,n\}\to\{1,\dots,m\}$ be a
hash function and let $s\in\{\pm 1\}^{n}$ be a sign vector. Choose
pairwise distinct scalars $c_{1},\dots,c_{m}$ and define
\[
a_{i} = c_{i},
\qquad
b_{j} = c_{h(j)} .
\]
Then for every $x\in\mathbb{R}^{n}$,
\[
\lim_{t\to 0}\our_{t}(a,b)\,(s\odot x) \;=\; y,
\qquad
y_{i} \;=\; \sum_{j:\,h(j)=i} s_{j}\,x_{j}.
\]
That is, CountSketch arises as the zero-temperature limit of the
kernel architecture (with sign multiplication folded into the input).
\end{proposition}

\begin{proof}
By Lemma~\ref{lem:zero_temp_appendix} and the fact that
$\{c_{r}\}$ are distinct,
$\lim_{t\to 0}[\our_{t}(a,b)]_{ij} = \mathbf{1}\{i = h(j)\}$.
For the $i$-th output coordinate,
$\sum_{j} \mathbf{1}\{i=h(j)\} s_{j} x_{j} = \sum_{j:\,h(j)=i} s_{j} x_{j}$,
which is precisely CountSketch.
\end{proof}

\subsection{Universality for rectangular matrices}

\begin{theorem}[Rectangular universality]
\label{thm:rectangular_universality}
Let $W\in\mathbb{R}^{m\times n}$ be arbitrary.  There exist a
replication operator $R_m:\mathbb{R}^{n}\!\to\!\mathbb{R}^{mn}$, a
diagonal matrix $D_{W}\in\mathbb{R}^{mn\times mn}$, and vectors
$a\in\mathbb{R}^{m}$, $b\in\mathbb{R}^{mn}$ such that
\[
W \;=\; \Bigl(\lim_{t\to 0}\our_{t}(a,b)\Bigr)\, D_{W}\, R_{m}.
\]
The kernel architecture is therefore universal for all linear maps
$\mathbb{R}^{n}\!\to\!\mathbb{R}^{m}$.
\end{theorem}

\begin{proof}
Choose pairwise distinct scalars $c_{1},\dots,c_{m}$ and set
$a_i = c_i$. Index replicas by pairs $(i,j)$ with
$i\in\{1,\dots,m\}$ and $j\in\{1,\dots,n\}$, and define
\[
(R_m x)_{(i,j)} = x_j, \qquad
(D_W)_{(i,j),(i,j)} = W_{ij}, \qquad
b_{(i,j)} = c_i .
\]
By Lemma~\ref{lem:zero_temp_appendix},
\[
\lim_{t\to 0} [\our_t(a,b)]_{r,(i,j)}
= \mathbf{1}\{c_r = c_i\}
= \mathbf{1}\{r = i\}.
\]
Thus the kernel groups all replicas with the same first index $i$, and
the $r$-th output coordinate becomes
\[
\sum_{i,j} \mathbf{1}\{r=i\}\,W_{ij}\,x_j
= \sum_{j} W_{rj}\,x_j
= (W x)_r.
\]
\end{proof}

The proof shows that no \emph{output} aggregation operator is needed:
the kernel itself sums all replicas with the same target index.

\subsection{Interpretation}

The two results above sketch a continuous bridge between smooth
structured kernels at positive temperature and discrete combinatorial
routing in the limit $t \to 0$.  At low temperature, $\our_{t}(a,b)$
acts as a sparse hash table determined by equality patterns --
recovering CountSketch and arbitrary discrete routings as special
cases.  Adding replication and diagonal scaling makes the family
universal for arbitrary rectangular matrices.  In this construction, replication and the diagonal map cost $\mathcal{O}(mn)$,
the size of $W$, and $\our_{t}(a,b)$ admits the fast multiplication
algorithm of App.~\ref{app:matvec-proofs}.

\section{Fast matrix--vector multiplication: proofs and complexity}
\label{app:matvec-proofs}

This section proves the two reduction theorems announced in
Sec.~\ref{sec:theory} and gives the exact complexity in both regimes.
Everything rests on one primitive: an exponentially weighted prefix/suffix
scan over a sorted vector. We describe it first in
Sec.~\ref{app:matvec:symmetric}, then handle the two dispatch branches:
Sec.~\ref{app:matvec:short-a} reduces the general
$\mathbb{R}^{k}\!\to\!\mathbb{R}^{n}$ matvec to that primitive when
$n \le k$ (sort and scan along the shorter $a$-axis),
and Sec.~\ref{app:matvec:short-b} treats the dual case $n \ge k$
(sort and scan along the shorter $b$-axis).  The combined complexity
is summarized in Sec.~\ref{app:matvec:complexity}.

\subsection{Symmetric case: prefix-scan reduction}
\label{app:matvec:symmetric}

Let $a\in\mathbb{R}^{n}$ with $a_1 \le \cdots \le a_n$ and write
\[
A \;=\; \our(a) \;=\; \bigl[\exp(-|a_i - a_j|)\bigr]_{ij}.
\]
For $i \ge j$, $A_{ij} = \exp(a_j - a_i)$; for $i < j$,
$A_{ij} = \exp(a_i - a_j)$.  Decompose
$Ax = \mathrm{tril}(A)\,x + \mathrm{triu}(A,1)\,x$ and consider the
lower triangular part
\[
y_i \;=\; \sum_{j \le i}\exp(a_j - a_i)\,x_j
   \;=\; \exp(-a_i)\sum_{j \le i}\exp(a_j)\,x_j .
\]
The inner sum $s_i = \sum_{j \le i}\exp(a_j)\,x_j$ satisfies the
recurrence $s_1 = \exp(a_1)\,x_1$, $s_i = s_{i-1} + \exp(a_i)\,x_i$.
This recurrence is associative, hence \emph{parallel}: a prefix-scan
computes $(s_i)$ in $\mathcal{O}(n)$ work and $\mathcal{O}(\log n)$
depth, after which $y_i = \exp(-a_i)\,s_i$ is a pointwise rescaling.
The above recurrence, written as $y_i=e^{\delta_i}y_{i-1}+x_i$ with $\delta_1=0$ and
$\delta_i=a_{i-1}-a_i$, can also be evaluated by an associative scan over the pairs
$(\delta_i,x_i)$ with the binary operator
\[
(A_1, B_1) \oplus (A_2, B_2)
=
\left(
A_1 + A_2,\;
\exp(A_2)\, B_1 + B_2
\right),
\]
which is compatible with parallel prefix-scan.

Our portable PyTorch implementation computes the prefix sums with \texttt{logcumsumexp} in the log domain.
Starting from
\[
y_i = \exp(-a_i) \sum_{j \leq i} \exp(a_j) x_j,
\]
we observe that the sum can be stably computed via a cumulative complex log-sum-exp: for $x_j<0$ we use $\log x_j=\log|x_j|+\mathrm{i}\pi$, and $x_j=0$ is mapped to $-\infty$. 

We compute prefix values via
\[
\log s_i
= \texttt{logcumsumexp}_{j \le i}\!\bigl(a_j + \log x_j\bigr).
\]

The final result is then obtained as
\[
y_i = \exp(-a_i + \log s_i).
\]
This approach replaces the explicit prefix-sum with a numerically stable cumulative log-sum-exp, yielding a fully vectorized and GPU-efficient implementation within the standard PyTorch API.

\subsection{General case via the shorter axis ($n \le k$)}
\label{app:matvec:short-a}

We now prove the main-text Theorem on the
$\mathbb{R}^{k}\!\to\!\mathbb{R}^{n}$ matvec with the kernel
$A = \our(a, b) \in \mathbb{R}^{n \times k}$ in the regime
$n \le k$, and locate its complexity.

\begin{theorem}[Reduction via the $a$-axis]
\label{thm:app-short-a}
Assume $a_1 \leq \cdots \leq a_n$, let $b\in\mathbb{R}^k$, and extend with $a_0 = -\infty$, $a_{n+1} = +\infty$, with the convention of Theorem~\ref{thm:1}.
For $i = 1, \ldots, n+1$ define the intervals
\[
P_i = [a_{i-1}, a_i).
\]
Given $x \in \mathbb{R}^k$, define
$u, v \in \mathbb{R}^{n+1}$ by:
\[
u_i = \sum_{j : b_j \in P_i} x_j\, e^{\,a_i - b_j},
\qquad
v_i = \sum_{j : b_j \in P_i} x_j\, e^{\,b_j - a_i}.
\]
Let $K = \our(\bar a) \in \mathbb{R}^{(n+1) \times (n+1)}$, $\bar a = (a_1, \ldots, a_{n+1})$, be the kernel matrix with entries
$K_{ij} = e^{-|a_i - a_j|}$. Then the product
$y = \our(a,b)\, x$ satisfies
\[
y_i \;=\; \bigl(\mathrm{triu}(K,1)\,u\bigr)_i \;+\; \bigl(\mathrm{tril}(K)\,v\bigr)_i,
\qquad i = 1, \ldots, n,
\]
where $\mathrm{triu}(K,1)$ is the strictly upper-triangular part of $K$ and
$\mathrm{tril}(K)$ is the lower-triangular part including the diagonal.
\end{theorem}

\begin{proof}
By definition, for $i = 1, \ldots, n$,
\[
y_i = \sum_{j=1}^{k} e^{-|a_i - b_j|}\, x_j
    = \sum_{r=1}^{n+1}\; \sum_{j : b_j \in P_r} e^{-|a_i - b_j|}\, x_j.
\]
We split the outer sum at $r = i$ and analyze the two regions separately.

\emph{Region $r > i$ (right of $a_i$).}
For $b_j \in P_r$ with $r > i$ we have $b_j \ge a_{r-1} \ge a_i$, hence
$|a_i - b_j| = b_j - a_i$. Therefore
\[
e^{-|a_i - b_j|} = e^{a_i - b_j} = e^{a_i - a_r}\, e^{a_r - b_j}.
\]
Summing over all $j$ with $b_j \in P_r$ yields $e^{-(a_r - a_i)}\, u_r$.
Summing over all $r > i$ gives
\[
\sum_{r=i+1}^{n+1} e^{-(a_r - a_i)}\, u_r
= \sum_{r > i} K_{ir}\, u_r
= \bigl(\mathrm{triu}(K,1)\,u\bigr)_i.
\]

\emph{Region $r \le i$ (left of or at $a_i$).}
For $b_j \in P_r$ with $r \le i$ we have $b_j < a_r \le a_i$, hence
$|a_i - b_j| = a_i - b_j$. Therefore
\[
e^{-|a_i - b_j|} = e^{b_j - a_i} = e^{-(a_i - a_r)}\, e^{b_j - a_r}.
\]
Summing over all $j$ with $b_j \in P_r$ yields $e^{-(a_i - a_r)}\, v_r$.
Summing over all $r \le i$ gives
\[
\sum_{r=1}^{i} e^{-(a_i - a_r)}\, v_r
= \sum_{r \le i} K_{ir}\, v_r
= \bigl(\mathrm{tril}(K)\,v\bigr)_i.
\]

Adding the two contributions yields $y_i = (\mathrm{triu}(K,1)u)_i + (\mathrm{tril}(K)v)_i$,
which completes the proof.
\end{proof}

\paragraph{Complexity, case $n \le k$.}
Sorting $a$ costs $\mathcal{O}(n\log n)$.  Each $b_j$ is placed in its
bucket $P_r$ via binary search, costing
$\mathcal{O}(k\log n)$ in total.  Forming $u$ and $v$ from the
bucketed contributions is $\mathcal{O}(k)$ pointwise work.  The two
triangular kernel multiplications are $\mathcal{O}(n)$ each by
Sec.~\ref{app:matvec:symmetric}.  In total, we have
\[
\mathcal{O}(n\log n)\;+\;\mathcal{O}(k\log n)\;+\;\mathcal{O}(k)\;+\;\mathcal{O}(n)
\;=\;
\mathcal{O}(k\log n) \quad\text{when } n \le k.
\]
If the sort of $a$ can be cached across multiple matvecs (as in
attention or training, where the same anchors are reused over many
batches), it amortises away and the per-call cost drops to
$\mathcal{O}(n + k\log n) = \mathcal{O}(k\log n)$.

\subsection{Dual reduction via the shorter axis ($n \ge k$)}
\label{app:matvec:short-b}

In the regime $n \ge k$ the symmetric construction operates on
$b$ rather than $a$: we sort $b$, scan the symmetric kernel
$\our(b)$ over $b$, and propagate the result to each $a_i$ by a
binary search and a pair of pointwise products.

\begin{theorem}[Dual reduction via the $b$-axis]
\label{thm:app-short-b}
Assume $b_1 \le \cdots \le b_k$ and let $K = \our(b)$.  Define
\[
u \;=\; \mathrm{tril}(K)\,x,
\qquad
v \;=\; \mathrm{triu}(K)\,x .
\]
For each $i$, set $j_i = \max\{j : b_j \le a_i\}$, with $j_i = 0$ if
no such $j$ exists.  Then $y = \our(a, b)\, x$ satisfies
\[
y_i
\;=\;
\begin{cases}\exp(b_{j_i} - a_i)\, u_{j_i} & j_i \ge 1,\\ 0 & \text{otherwise},\end{cases}
\;+\;
\begin{cases}\exp(a_i - b_{j_i+1})\, v_{j_i+1} & j_i < k,\\ 0 & \text{otherwise}.\end{cases}
\]
\end{theorem}

\begin{proof}
Fix $i$ and split the sum at the threshold $j_i$
\[
y_i \;=\; \sum_{j \le j_i}\exp(b_j - a_i)\,x_j \;+\; \sum_{j > j_i}\exp(a_i - b_j)\,x_j.
\]
Factoring out the dependence on $a_i$ we have
\[
y_i
\;=\;
\exp(-a_i)\sum_{j \le j_i}\exp(b_j)\,x_j
\;+\;
\exp(a_i)\sum_{j > j_i}\exp(-b_j)\,x_j.
\]

The triangular kernel applied to $b$ yields
\[
u_j = \sum_{l \le j} \exp(b_l - b_j)\, x_l,
\qquad
v_j = \sum_{l \ge j} \exp(b_j - b_l)\, x_l,
\]
which, after multiplying by $\exp(b_j)$ and $\exp(-b_j)$, respectively, gives
\[
\sum_{l \le j} \exp(b_l)\, x_l = \exp(b_j)\, u_j,
\qquad
\sum_{l \ge j} \exp(-b_l)\, x_l = \exp(-b_j)\, v_j.
\]
Evaluating these identities at $j = j_i$ and $j = j_i + 1$
gives the desired formula.
\end{proof}

\paragraph{Complexity, case $n \ge k$.}
Sort of $b$ costs $\mathcal{O}(k\log k)$; the two triangular kernel
multiplications cost $\mathcal{O}(k)$ via the symmetric scan.  Each
of the $n$ split points $j_i$ is found by binary search in
$\mathcal{O}(\log k)$, giving $\mathcal{O}(n\log k)$ total.  Final
evaluation of $y$ is $\mathcal{O}(n)$.  Adding the components,
\[
\mathcal{O}(k\log k) \;+\; \mathcal{O}(k) \;+\; \mathcal{O}(n\log k) \;+\; \mathcal{O}(n)
\;=\;
\mathcal{O}(n\log k) \quad\text{when }n \ge k.
\]
As before, the sort of $b$ amortises away when the same anchors are
reused.

\subsection{Combined complexity}
\label{app:matvec:complexity}

Combining both branches, the implementation dispatches automatically
on $\min(n, k)$:
\[
\boxed{\quad
T_{\our}(n, k)
\;=\;
\mathcal{O}\bigl(\max(n, k)\,\log\min(n, k)\bigr)
\quad}
\]
for a \emph{single} matrix--vector product, against the
$\mathcal{O}(n k)$ of the dense baseline.  In a batch of $B$ vectors
sharing the same anchors, the sort and the binary-search results
are computed once and reused across the batch, so the per-vector cost drops to $\mathcal{O}(n+k)$ and the total cost is
\[
T_{\our}(n, k, B)
\;=\;
\mathcal{O}\bigl(\max(n, k)\log\min(n, k) + B\,(n+k)\bigr).
\]
Memory usage is $\mathcal{O}(B(n + k))$ for the working tensors, with
no $n\!\times\!k$ matrix ever materialized.

\section{Fast weighted Gram matrix: proof and complexity}
\label{app:gram-proof}

The matrix--vector reductions explain how to apply a single Laplace
operator.  The weighted Gram result is the next step: it shows that the
same one-dimensional ordering also controls the second-order object that
appears in kernel methods, curvature approximations, and Woodbury
likelihoods.  We prove the main-text theorem for
\[
M \;=\; A\,\mathrm{diag}(D)\,A^{\!\top},
\qquad
A_{ij} \;=\; \exp(-|a_i - b_j|),
\qquad D \in \mathbb{R}^{k}_{+},
\]
in $\mathcal{O}(n^{2} + k\log n)$ rather than the
$\mathcal{O}(n^{2} k)$ of the naive matmul (see Fig.~\ref{fig:gram_decomposition}). 

\begin{figure}
    \centering
    \def\svgwidth{\linewidth}
    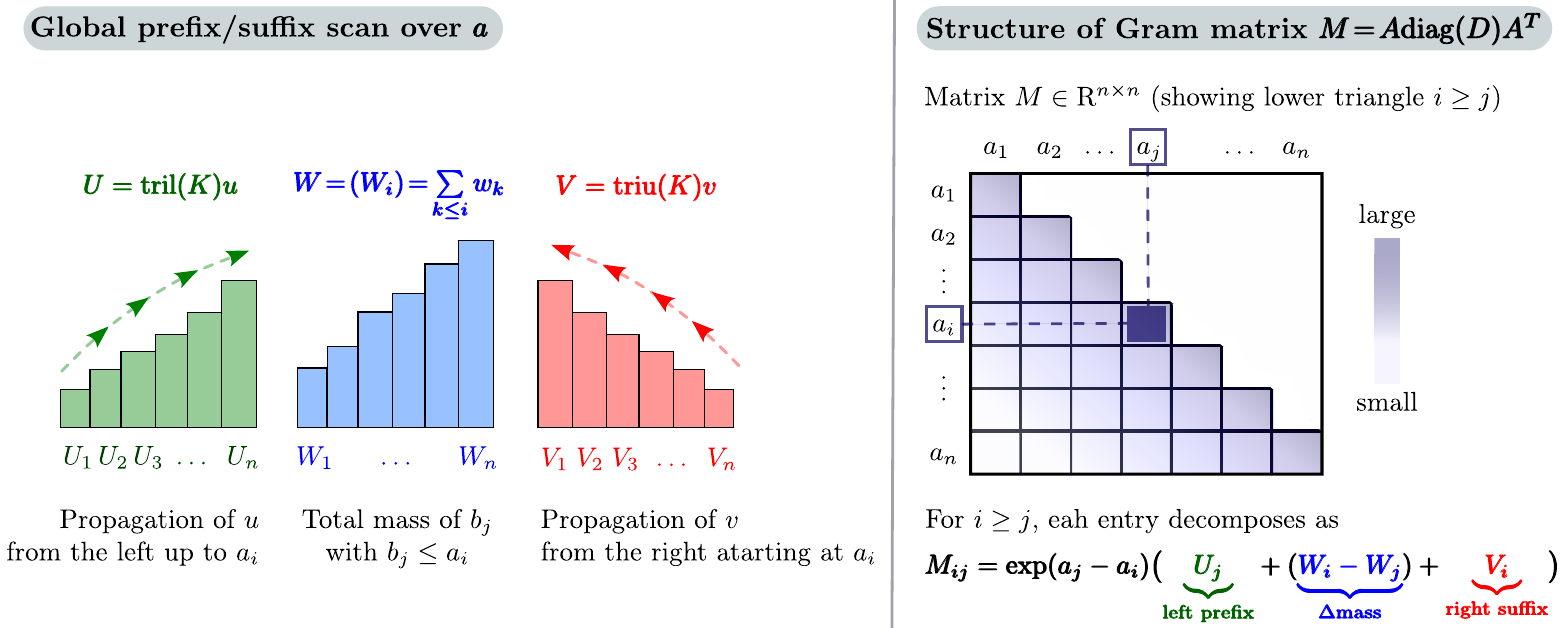
    \caption{Structure of the exact Gram matrix $M = A \operatorname{diag}(D) A^\top$ for the Laplace kernel. Each entry $M_{ij}$ (shown for $i \ge j$) admits a decomposition into three one-dimensional aggregated contributions: a left prefix term $U_j$ capturing accumulated influence from indices up to $a_j$, an interval mass $(W_i - W_j)$ corresponding to a prefix-suffix decomposition over the interval $[a_j, a_i)$, and a right suffix term $V_i$ encoding contributions propagated from the right. All components are modulated by the exponential envelope $\exp(a_j - a_i)$, reflecting the ordering-induced decay structure of the Laplace kernel. The dense Gram matrix can be exactly represented using prefix- and suffix-scan statistics over a single sorted coordinate axis, rather than full pairwise interactions.}
    \label{fig:gram_decomposition}
\end{figure}

\begin{theorem}[Exact Gram decomposition]
\label{thm:app-gram}
Let $a_1 \le \cdots \le a_n$, let $b\in\mathbb{R}^k$, and set $a_0=-\infty$, $a_{n+1}=+\infty$, with the convention of Theorem~\ref{thm:1}.
Partition $\mathbb{R}$ into the intervals
\[
P_r := [a_{r-1}, a_r), \qquad r=1,\dots,n+1.
\]
Given weights $D\in\mathbb{R}^k$, define
$u,v,w \in \mathbb{R}^{n+1}$ by
\[
u_i := \sum_{b_t \in P_i} e^{2(b_t - a_i)}\,D_t,
\qquad
v_i := \sum_{b_t \in P_i} e^{2(a_i - b_t)}\,D_t,
\qquad
w_i := \sum_{b_t \in P_i} D_t.
\]
Let $W_i := \sum_{r \le i} w_r$ be the cumulative weights.
Define the kernel matrix $K \in \mathbb{R}^{(n+1)\times(n+1)}$ by
$K_{ij} := e^{-2|a_i-a_j|}$, $i, j = 1, \ldots, n+1$, and let $A :=\our(a,b)$. Set additionally
\[
U := \operatorname{tril}(K)\,u,
\qquad
V := \operatorname{triu}(K,1)\,v.
\]
Then the Gram matrix $M := A\,\operatorname{diag}(D)\,A^\top$
satisfies, for all $n \ge i \ge j$,
\[
M_{ij}
=
e^{a_j-a_i}\,\bigl( U_j + (W_i - W_j) + V_i \bigr),
\]
with symmetry $M_{ji}=M_{ij}$.
\end{theorem}

\begin{proof}
By definition, $M_{ij} = \sum_{t=1}^{k} D_t \, e^{-|a_i - b_t|} \, e^{-|a_j - b_t|}$, and
assume $n \ge i \ge j$. We partition the inner sum into three disjoint regions according to the position of $b_t$ relative to $a_j$ and $a_i$; they are unions of the intervals $P_r$, so every $b_t$ is counted exactly once.

\emph{Left region ($b_t < a_j$).}
Here $b_t < a_j \le a_i$, so $|a_i - b_t| = a_i - b_t$ and $|a_j - b_t| = a_j - b_t$. Hence
\[
e^{-|a_i - b_t|} e^{-|a_j - b_t|} 
= e^{-(a_i - b_t)} e^{-(a_j - b_t)} 
= e^{-(a_i + a_j)} e^{2b_t}.
\]
Summing over all $b_t < a_j$ (i.e., over intervals $P_r$ with $r \le j$) yields
\[
e^{-(a_i + a_j)} \sum_{r \le j} \sum_{b_t \in P_r} D_t e^{2b_t}
= e^{-(a_i + a_j)} \sum_{r \le j} e^{2a_r} u_r.
\]
By definition, $U_j = \sum_{r \le j} K_{jr} u_r = \sum_{r \le j} e^{-2(a_j - a_r)} u_r = e^{-2a_j} \sum_{r \le j} e^{2a_r} u_r$. 
Thus the left contribution simplifies to $e^{a_j - a_i} U_j$.

\emph{Middle region ($a_j \le b_t < a_i$, i.e., intervals $P_r$ with $j < r \le i$).}
Here $|a_i - b_t| = a_i - b_t$ and $|a_j - b_t| = b_t - a_j$. The product becomes
\[
e^{-(a_i - b_t)} e^{-(b_t - a_j)} = e^{-(a_i - a_j)} = e^{a_j - a_i},
\]
which is independent of $b_t$. Summing over this region gives
\[
e^{a_j - a_i} \sum_{a_j \le b_t < a_i} D_t = e^{a_j - a_i} (W_i - W_j).
\]

\emph{Right region ($b_t \ge a_i$).}
Here $b_t \ge a_i \ge a_j$, so $|a_i - b_t| = b_t - a_i$ and $|a_j - b_t| = b_t - a_j$. Hence
\[
e^{-|a_i - b_t|} e^{-|a_j - b_t|} 
= e^{-(b_t - a_i)} e^{-(b_t - a_j)} 
= e^{a_i + a_j} e^{-2b_t}.
\]
Summing over all $b_t \ge a_i$ (intervals $P_r$ with $i < r \le n+1$) yields
\[
e^{a_i + a_j} \sum_{r > i} \sum_{b_t \in P_r} D_t e^{-2b_t}
= e^{a_i + a_j} \sum_{r > i} e^{-2a_r} v_r.
\]
Using $V_i = \sum_{r > i} K_{ir} v_r = \sum_{r > i} e^{-2(a_r - a_i)} v_r = e^{2a_i} \sum_{r > i} e^{-2a_r} v_r$, 
the right contribution simplifies to $e^{a_j - a_i} V_i$.

Adding the three contributions and factoring out $e^{a_j - a_i}$ yields the claimed formula. 
Symmetry $M_{ji} = M_{ij}$ follows directly from the definition of $M$ and the absolute value in the exponent.
\end{proof}

\paragraph{Complexity.}
Sorting $a$ and bucketing $b$ cost $\mathcal{O}(n\log n + k\log n)$.
The per-interval aggregates $u_i, v_i, w_i$ are each a one-pass scatter
over $b$, $\mathcal{O}(k)$.  The cumulative quantities $U_i, V_i, W_i$
are triangular scans and a prefix sum, $\mathcal{O}(n)$.  Once $U, V, W$ are tabulated,
each entry $M_{ij}$ is evaluated in $\mathcal{O}(1)$ from the formula of Theorem~\ref{thm:app-gram}, so the dominant cost is the explicit construction of the
$n \times n$ matrix:
\[
\boxed{\quad
T_{\mathrm{Gram}}(n, k)
\;=\;
\mathcal{O}\bigl(n^{2} + k\log n\bigr),
\quad}
\]
versus $\mathcal{O}(n^{2} k)$ for the dense reference.  The
$\Theta(n^{2})$ component is the size of the output matrix and is
unavoidable.  The asymptotic speed-up over the dense baseline is
$\Theta(n^{2}\,k / (n^{2} + k\log n))$, which reduces to
$\Theta(n^{2}/\log n)$ when $k\log n \gtrsim n^{2}$ and to
$\Theta(k)$ when $n^{2} \gtrsim k\log n$. Memory is
$\mathcal{O}(n^{2}+k)$ for the output and intermediate-anchor data.\space%

\paragraph{Phased kernels and signed weights.}\label{app:phased-gram}
Theorem~\ref{thm:app-gram} is linear in $D$, so it also holds for signed weights;
nonnegative weights are needed only for $M$ to be positive semidefinite. For the phased
kernel of Sec.~\ref{sec:theory}, the angle-sum identity gives
$A=\our(a,\phi;b,\psi)=C_\phi KC_\psi+S_\phi KS_\psi$ with $K=\our(a,b)$,
$C_\phi=\mathrm{diag}(\cos\phi)$, $S_\phi=\mathrm{diag}(\sin\phi)$, and $C_\psi,S_\psi$
defined analogously. Writing $G(q)=K\,\mathrm{diag}(q)\,K^{\top}$,
\[
A\,\mathrm{diag}(D)\,A^{\top}
= C_\phi G_{cc} C_\phi + S_\phi G_{ss} S_\phi + C_\phi G_{cs} S_\phi + S_\phi G_{cs} C_\phi,
\]
where $G_{cc}=G(D\odot\cos^2\psi)$, $G_{ss}=G(D\odot\sin^2\psi)$ and
$G_{cs}=G(D\odot\cos\psi\odot\sin\psi)$. A phased weighted Gram therefore reduces to three
non-phased ones, one of them with signed weights, plus $\mathcal{O}(n^2)$ diagonal scaling.

\section{Kernel-algebra benchmarks: full empirical record}
\label{sec:detailed-kernel-bench}

The main text compresses the kernel-algebra story into one central
figure. Here we unpack it into the individual stress tests behind
Sec.~\ref{sec:kernel-algebra-bench}. We benchmark the three
kernel-algebra primitives provided by \our{}: matrix--vector
multiplication $X\,K(a,b)$ (and its transpose), and the weighted Gram
matrix $K(a,b)\,\mathrm{diag}(D)\,K(a,b)^{\!\top}$,
where $K(a,b)_{ij} = \exp(-|a_i - b_j|)$. 
Each timing is the best of three independent runs of $30$ timed iterations after
$10$ warm-up iterations, with explicit \texttt{cuda.synchronize()}
barriers around the timed region. For backward measurements we form
$\mathcal{L} = \mathbf{1}^{\!\top}\!Y$ with $Y$ the operator output and
back-propagate through the inputs (data tensor and both anchor sets).
The \emph{explicit-dense} baseline materializes the full $n\times k$
kernel matrix and uses a single \texttt{torch.matmul}\ call; this is the
strongest PyTorch baseline whenever memory permits. There is no time or
memory budget: the problem size is doubled until the baseline no longer
fits in device memory (an allocation fails with an out-of-memory error),
and the last size that fits is marked with a red dot on the corresponding
plot. \our{} is run the same way, but its CUDA kernels index tensors with
32-bit integers, so it stops at the last size whose tensors have fewer than
$2^{31}$ elements (open squares; Fig.~\ref{fig:kernel-algebra-bench}). All GPU measurements use a single NVIDIA~H100 (80\,GB) in
\texttt{float32} (Figs.~\ref{fig:comp_runtime} and~\ref{fig:runtime_memory} also show an
NVIDIA~RTX~2080~Ti with 11\,GB); so that this limit reflects device memory rather than
allocator fragmentation, PyTorch's caching allocator runs with expandable
segments.

\subsection{Cost of learning the anchors}
\label{sec:anchor-cost}
Table~\ref{tab:anchor-cost} separates the anchor-gradient, sorting and
permutation costs from the remaining work in a forward+backward call to
$XK(a,b)$, with $n=k$ and $B=8$. The entries are mean percentages and
standard deviations over three seeds.

\begin{table}[htbp]
\centering\small
\caption{Relative component costs in the LAPLEX forward+backward primitive
($n=k$, $B=8$), mean $\pm$ standard deviation over three seeds.}
\label{tab:anchor-cost}
\begin{tabular}{@{}rcccc@{}}
\toprule
$n=k$ & anchor gradients (\%) & sorting (\%) & permutation (\%) & remaining (\%)\\
\midrule
$2^{10}$ & $16.3\pm3.4$ & $12.4\pm0.7$ & $8.8\pm2.4$ & $62.6\pm5.1$\\
$2^{13}$ & $17.1\pm1.9$ & $14.7\pm4.9$ & $15.3\pm2.1$ & $53.0\pm1.1$\\
$2^{16}$ & $14.0\pm5.7$ & $13.5\pm1.6$ & $14.9\pm3.1$ & $57.6\pm10.0$\\
$2^{19}$ & $6.2\pm0.0$ & $3.8\pm0.1$ & $49.8\pm0.0$ & $40.1\pm0.1$\\
$2^{22}$ & $3.6\pm0.1$ & $1.1\pm0.0$ & $37.6\pm0.0$ & $57.7\pm0.1$\\
$2^{25}$ & $2.9\pm0.0$ & $0.8\pm0.0$ & $32.9\pm0.1$ & $63.5\pm0.0$\\
\bottomrule
\end{tabular}
\end{table}

At the largest measured size, $n=2^{25}$, the three anchor-specific parts
together account for less than $40\%$ of the primitive's step time. This
breakdown implies less than a twofold cost relative to the remaining work;
it is not a timing guarantee for an entire downstream model. Permutation is
memory-bound: it reads the $B\times n$ input in forward and scatters a
corresponding array in backward. Its relative share peaks at $n=2^{19}$.

Anchor gradients reuse the prefix/suffix scans and sorted order from forward,
with $\mathcal{O}(n+k)$ pointwise reductions and no additional sorting.
Backward therefore has the same
$\mathcal{O}(\max(n,k)\log\min(n,k))$ asymptotic cost as forward. Exact
kernel evaluation avoids changing the objective or its gradients through a
kernel approximation; it does not by itself guarantee convergence or an
unchanged number of optimizer steps. Numerical agreement is assessed in
Sec.~\ref{sec:app-exp-accuracy}, and empirical crossing diagnostics are
reported in App.~\ref{app:anchor-dynamics}.

\subsection{GPU runtime, peak memory, and forward+backward}\space%
\label{sec:app-exp-gpu}

\paragraph{Setup.}
Square kernel ($n = k$), batch size $B = 8$.  We sweep
$n = 2^{10}, 2^{11}, \dots$ for both methods.
The explicit-dense baseline runs out of the device memory at
$n = 2^{17}$, where the $n\times n$ matrix $K$ alone takes $64$\,GiB;
\our{} is run up to $n = 2^{27}$.

\begin{figure}[ht]
    \centering
    \resizebox{\linewidth}{!}{\input{NeurIPS_2026/camera_ready/images/exp1_app.pgf}}
    \caption{\textbf{GPU benchmark of $X K(a,b)$ as a function of $n=k$ at
        $B=8$.}  Forward time (left), forward+backward time (middle), and
        peak GPU memory of the forward call (right).  \our{} is shown in
        blue; the explicit-dense baseline is shown in orange.  Red dots
        mark the last size that fits in GPU memory.
        At $n = 2^{16}$, the largest size the dense baseline fits, \our{}
        is more than two orders of magnitude faster on both passes and uses
        over three orders of magnitude less memory.  At $n = 2^{20}$, \our{}
        completes the forward+backward pass in a few milliseconds with a
        few hundred megabytes of working memory, far below the device limit.}
    \label{fig:app-exp-gpu}
\end{figure}

\paragraph{Observations.}
The forward time of \our{} grows roughly linearly in $n$ from
$n \approx 2^{18}$ on (consistent with the
$\mathcal{O}(B(n+k))$ working set, dominated by memory bandwidth at
these sizes), while smaller sizes are dominated by
constant kernel-launch overhead and stay flat, well below a millisecond (see~Fig.~\ref{fig:app-exp-gpu}). The
forward+backward time is a small constant multiple of the forward time
(one extra scan launch plus the fused $\partial / \partial a$ and $\partial / \partial b$
reduction), nearly independent of $n$.  Peak memory grows as $\mathcal{O}(B
n)$ with a small constant (the working set is the input tensor plus a
few $(B,n)$ scratch buffers), and at $n = 2^{20}$ is a few hundred
megabytes.  The dense baseline uses $\mathcal{O}(n^{2})$ memory
and runs out of device memory at $n = 2^{17}$.

\subsection{Numerical accuracy}
\label{sec:app-exp-accuracy}

\paragraph{Setup.}
We compare \our{}'s \texttt{float32} forward output and the explicit
\texttt{float32} dense matmul \emph{both against an
\texttt{float64} dense ground truth}.  The same input tensors $a, b, X$
in \texttt{float32} are cast to \texttt{float64} for the reference;
the difference therefore reflects only the algorithmic precision of the
two \texttt{float32} routes, not input rounding.  Square kernel,
$B = 8$, $n = 2^{10}, \dots, 2^{16}$; at $n = 2^{17}$ the \texttt{float64}
reference ($n \times n$ entries, $128$\,GiB) no longer fits in device
memory, so no error can be computed beyond that size.

\begin{figure}[ht]
    \centering
    \resizebox{\linewidth}{!}{\input{NeurIPS_2026/camera_ready/images/exp2.pgf}}
    \caption{\textbf{Relative error against an \texttt{float64} dense
        reference.}  Left: $\ell_2$ relative error.  Right: $\ell_\infty$
        relative error.  The dashed line marks the
        \texttt{float32}\ machine epsilon $\varepsilon \approx 1.19 \cdot 10^{-7}$.
        \our{} sits at the unit-of-last-place noise floor across the entire
        range, while the explicit dense \texttt{float32}\ matmul has an error that grows with $n$, exceeding $\our{}$ by a factor that
        grows with $n$.  Open circles mark the largest size at which the
        \texttt{float64} reference fits in GPU memory.}
    \label{fig:app-exp-accuracy}
\end{figure}

\paragraph{Observations.}
Although \our{}'s \texttt{float32} path uses an internally rescaled
mantissa--exponent representation, its output relative error is
\emph{lower} than that of the corresponding $X @ K$ matmul performed in
the same precision.
At all tested sizes the
$\ell_\infty$ error of \our{} stays below $5\cdot10^{-7}$, and the
$\ell_2$ error stays at $\sim\!1.5\cdot10^{-7}$ (see~Fig.~\ref{fig:app-exp-accuracy}).

\subsection{Batch-size scaling}
\label{sec:app-exp-batch}

\paragraph{Setup.}
We fix $n = k = 2^{14}$ and double $B$ from $1$.
The explicit-dense baseline fits in device memory up to $B = 2^{18}$
(forward+backward) and $B = 2^{19}$ (forward); \our{} is run up to
$B = 2^{16}$.

\begin{figure}[ht]
    \centering
    \resizebox{\linewidth}{!}{\input{NeurIPS_2026/camera_ready/images/exp3.pgf}}
    \caption{\textbf{Batch-size scaling at $n = k = 2^{14}$.}  Forward
        (left) and forward+backward (right).  The dense baseline is
        nearly $B$-independent at small and moderate $B$ because its cost
        is dominated by building $K$ (independent of $B$); only at large
        $B$ does the actual matmul become the bottleneck.
        \our{} is flat at small $B$ (kernel-launch overhead) and grows
        linearly with $B$ once $B \gtrsim 2^{6}$, staying faster than the
        dense baseline at every $B$.  Red dots mark the last batch size at
        which the dense baseline fits in GPU memory.}
    \label{fig:app-exp-batch}
\end{figure}

\paragraph{Observations.}
At small and moderate $B$ the dense baseline's runtime is essentially the
cost of forming the $n \times n$ kernel matrix; the subsequent $X @ K$
matmul is so inexpensive next to that operation that batch size has very
little effect.  This pattern flips for \our{}: there is no $K$ to form, so
the cost grows as $\mathcal{O}(B \cdot n)$ once the anchors are sorted.
Below $B \approx 2^{6}$ kernel-launch overhead dominates and \our{} is
flat in $B$; beyond that it grows linearly.  The dense forward+backward
plateau at small $B$ is the time of forming the explicit kernel matrix and
\emph{computing its gradient} -- a cost that \our{} sidesteps entirely.
At this size $K$ takes only $1$\,GiB, so at large $B$ the memory of both
methods is dominated by the $\mathcal{O}(B n)$ activations; the dense
baseline runs out of device memory at $B = 2^{19}$ (forward+backward) and
$B = 2^{20}$ (forward) (see~Fig.~\ref{fig:app-exp-batch}).

\subsection{Asymmetric kernels: fixed short axis}
\label{sec:app-exp-asymm}

\paragraph{Setup.}
The $\mathcal{O}(\max(n,k)\log\min(n,k))$ scan-based algorithm we use selects the
shorter axis automatically: when $n \le k$ it scans along $a$ and
gathers per $b$ (branch~A), otherwise it scatters from $a$ into
buckets along sorted $b$ and scans along $b$ (branch~B).  Both
branches share the same asymptotic cost but expose different sets of
operations.  We test each branch with the corresponding axis fixed at
$1024$ and the other doubled from $2^{10}$ upwards.

\begin{figure}[ht]
    \centering
    \resizebox{\linewidth}{!}{\input{NeurIPS_2026/camera_ready/images/exp4.pgf}}
    \caption{\textbf{Asymmetric kernels.}  Left: $n = 1024$ fixed,
        $k$ varies (branch~A).  Right: $k = 1024$ fixed, $n$ varies
        (branch~B).  \our{} is essentially constant in the swept axis
        up to $\sim\!2^{17}$--$2^{21}$ (cost is dominated by fixed
        overheads), then grows linearly with the long axis.  The dense
        baseline grows linearly throughout until it runs out of GPU
        memory at $2^{24}$; red dots mark the last size that fits.}
    \label{fig:app-exp-asymm}
\end{figure}

\paragraph{Observations.}
With $n = 1024$ fixed, \our{}'s forward time stays essentially constant up
to $k \approx 2^{21}$ and below a millisecond up to $k = 2^{25}$, while the
dense baseline grows linearly from the start.  The mirror-image experiment
($k = 1024$ fixed, $n$ varying) shows the same behavior with a larger
overhead from the scatter-add step.  At the largest size where the dense
baseline still runs ($2^{23}$; it runs out of device memory at $2^{24}$ in
both settings), \our{} is roughly an order of magnitude faster with
$k = 1024$ fixed and more than two orders of magnitude faster with
$n = 1024$ fixed (see~Fig.~\ref{fig:app-exp-asymm}).

\subsection{Weighted Gram matrix}
\label{sec:app-exp-gram}

\paragraph{Setup.}
For the weighted-Gram operation
$M = K(a,b)\,\mathrm{diag}(D)\,K(a,b)^{\!\top}$ we
fix the \emph{output} dimension at $n = 1024$ (so the result $M$ is
always a $1024 \times 1024$ matrix, $4$\,MB in \texttt{float32}) and
sweep the \emph{inner} dimension
$k = 2^{10}, 2^{11}, \dots$.  This isolates the algorithmic
scaling: we want to measure how the cost depends on the dimension we
sum over, not on the size of the result.

\begin{figure}[ht]
    \centering
    \resizebox{\linewidth}{!}{\input{NeurIPS_2026/camera_ready/images/exp5.pgf}}
    \caption{\textbf{Weighted Gram matrix scaling.}  Left: forward time.
        Right: peak GPU memory.  $n = 1024$ is fixed; $k$ varies.  Up to
        $k \approx 2^{17}$ the cost of \our{} is dominated by the
        $\mathcal{O}(n^2)$ output construction and is therefore
        \emph{independent of} $k$.  The dense baseline grows linearly
        in $k$ in both panels until it runs out of GPU memory at
        $k = 2^{24}$, where its $n \times k$ matrix $K$ alone takes
        $64$\,GiB; red dots mark the last size that fits.}
    \label{fig:app-exp-gram}
\end{figure}

\paragraph{Observations.}
The flat plateau in the forward-time panel is the predicted behavior
of the algorithm.  Working backwards from the analytical
cost~$\mathcal{O}(n^2 + k\log n)$: the $n^2$ term is
$1024^2 \approx 10^{6}$ FLOPs and dominates until the inner-dimension
term $k\log n = 10\,k$ catches up around $k \approx 10^{5}$, at which
point \our{} begins to grow linearly in $k$.  The dense baseline is
$\mathcal{O}(n^{2} k)$ in time and uses $\mathcal{O}(n k)$ memory, so
its time grows linearly with $k$ and so does its memory.
At the largest size that fits the dense baseline ($k = 2^{23}$), \our{} is
over two orders of magnitude faster and uses over two orders of magnitude
less memory (tens of gigabytes versus a few hundred megabytes,
see~Fig.~\ref{fig:app-exp-gram}).

\paragraph{Asymptotic speed-up.}
For any fixed $n$, the dense Gram costs $\mathcal{O}(n^{2} k)$ while
\our{} costs $\mathcal{O}(n^{2} + k \log n) =
\mathcal{O}(\max(n^2, k \log n))$.  In the regime where the inner
dimension dominates ($k \log n \gg n^{2}$), the speed-up of \our{}
over the dense baseline is asymptotically
$\Theta(n^{2} / \log n)$.

\subsection{CPU benchmark (no CUDA backend)}
\label{sec:app-exp-cpu}

\paragraph{Setup.}
The \our{} package ships with a portable pure-PyTorch implementation
(\texttt{laplex\_noCUDA}) built on a single \texttt{torch.logcumsumexp}
call per scan; this version supports CPU, CUDA, MPS, and ROCm uniformly
and requires no compilation.  We benchmark it on CPU only, so the
custom CUDA path is not invoked.  Square kernel, $B = 1$, 8 CPU threads;
$n$ is doubled from $2^{10}$ until the next size would exceed the
$100$\,GB memory limit of the run.

\begin{figure}[ht]
    \centering
    \resizebox{.7\linewidth}{!}{\input{NeurIPS_2026/camera_ready/images/exp6.pgf}}
    \caption{\textbf{CPU benchmark.}  Pure-PyTorch \our{} versus
        explicit dense, both running on CPU. The pure-PyTorch \our{}
        scales close to linearly in $n$ (\texttt{logcumsumexp}\ is
        well-vectorised in PyTorch's CPU backend) while the dense
        baseline grows quadratically.  Red dots mark the last size that
        fits in the $100$\,GB memory limit of the run.
        At $n = 2^{16}$, the last size the dense baseline fits, the
        \our{} forward is more than two orders of magnitude faster, and
        \our{} continues to $n = 2^{28}$.}
    \label{fig:app-exp-cpu}
\end{figure}

\paragraph{Observations.}
The pure-PyTorch path is much slower than the
custom CUDA backend on a per-call basis (compare the
sub-millisecond forward time at $n=2^{14}$ in Section~\ref{sec:app-exp-gpu}
against a few milliseconds here), but it is portable to any device
PyTorch supports and requires no compilation step on first import.
The dense baseline's CPU time grows quadratically, from milliseconds at
$n = 2^{10}$ to more than ten seconds at $n = 2^{16}$, the last size that
fits; there even the pure-PyTorch \our{} is more than two orders of
magnitude faster.  \our{} continues to $n = 2^{28}$, which makes it the
only path of the two that is usable for large-$n$ kernel work without a
GPU (see~Fig.~\ref{fig:app-exp-cpu}).

\subsection{Summary of empirical findings}
\label{sec:app-exp-summary}

Across all six stress tests, the same pattern repeats: avoiding the
explicit kernel matrix is not just an asymptotic win, but a practical
wall-clock and memory win on commodity hardware.  Concretely:

\begin{itemize}
    \item On GPU at the largest size where the dense baseline still
        runs ($n = k = 2^{16}$, $B = 8$), $\our{}$ is more than two
        orders of magnitude faster on both the forward and the
        forward+backward pass and uses over three orders of magnitude
        less peak VRAM (Sec.~\ref{sec:app-exp-gpu}).

    \item \our{} scales far past the point where the dense baseline
        runs out of memory, to problems more than three orders of
        magnitude larger; at $n = 2^{20}$ the forward+backward call
        takes a few milliseconds and a few hundred megabytes of VRAM
        (Sec.~\ref{sec:app-exp-gpu}).

    \item Despite running internally in \texttt{float32}, \our{}
        delivers \emph{better} relative-error numbers than the explicit
        dense \texttt{float32} matmul against an \texttt{float64}
        ground truth, at all tested $n$
        (Sec.~\ref{sec:app-exp-accuracy}).

    \item For asymmetric workloads with a short axis fixed at
        $1024$, \our{} is between roughly one and more than two orders
        of magnitude faster at the largest size that fits the dense
        baseline, depending on which axis is fixed
        (Sec.~\ref{sec:app-exp-asymm}).

    \item The weighted-Gram operation enjoys an asymptotic
        $\Theta(n^{2} / \log n)$ speed-up over the dense baseline,
        with over two orders of magnitude measured at the largest size
        that fits the dense baseline ($n = 1024$, $k = 2^{23}$;
        Sec.~\ref{sec:app-exp-gram}).

    \item Even on CPU, the pure-PyTorch reference implementation is
        more than two orders of magnitude faster than the explicit dense
        matmul at the largest size the dense baseline fits ($n = 2^{16}$),
        and it scales to $n = 2^{28}$ (Sec.~\ref{sec:app-exp-cpu}).
\end{itemize}

\noindent
These measurements are the empirical backbone for the main-text claim
that \our{} behaves like a fast transform in practice while retaining
exact Laplace-kernel semantics.

\section{FFT and Toeplitz--FFT baselines: full comparison}
\label{sec:detailed-fft-bench}

This appendix section spells out the relationship between \our{} and
FFT-based multiplication. The important structural fact is simple:
Toeplitz--FFT is the uniform-grid specialization of the Laplace operator
studied in this paper. If $a_i=b_i=i\Delta$, then $\our(a,b)$ is a
Toeplitz matrix and circulant embedding gives an exact FFT-based matvec.
If the anchors are irregular or learned, that diagonalization disappears.
The experiments below therefore measure the cost and benefit of leaving
the fixed-grid slice.

We separate four questions: runtime on GPU, accuracy on the uniform-grid
case where Toeplitz--FFT is exactly applicable, runtime without a custom GPU
kernel, and differentiability with respect to anchor locations. This is the
comparison that matters for a learnable fast transform: FFT methods provide
the best constants on their slice, while \our{} keeps the same asymptotic
class on a strictly larger and trainable family.
\subsection{GPU runtime (recap)}
\label{sec:app-fft-gpu}

\begin{figure}[t]
    \centering
    \resizebox{\linewidth}{!}{\input{NeurIPS_2026/camera_ready/images/exp12_full.pgf}}
    \caption{%
        \textbf{FFT comparison, \texttt{float32}.}
        \textbf{(a)} GPU runtime for a single matrix--vector product on
        irregular anchors ($B=1$, $n=k$, $a,b\sim\mathcal{N}(0,1)$):
        \our{} (blue), Toeplitz--FFT on its fixed-grid specialization
        (orange), and a single cuFFT call as a hardware reference
        (green). All three curves follow the $\mathcal{O}(n\log n)$
        reference. \textbf{(b)} Relative $\ell_2$ error against a
        \texttt{float64} dense reference on the uniform-grid Toeplitz
        problem where the FFT baseline is exactly applicable
        (App.~\ref{sec:app-fft-accuracy}).
        \our{} gives the smallest
        \texttt{float32} error at every size.
        Red dots mark the last size that fits in GPU memory and the open
        square the 32-bit indexing limit of the \our{} kernels
        (Fig.~\ref{fig:kernel-algebra-bench}); in (b), open circles mark the
        largest size at which the \texttt{float64} reference fits.
    }
    \label{fig:app_kerlap-vs-fft}
\end{figure}

Figure~\ref{fig:app_kerlap-vs-fft} reports
single-vector ($B = 1$) wall-clock times for
\our{}, Toeplitz--FFT, and a single \texttt{cuFFT} call
from $n = 2^{10}$ up to the largest size that fits in device memory (for \our{},
the largest size within the 32-bit indexing of its kernels).  All three curves are parallel on a
log--log plot and follow the $\mathcal{O}(n\log n)$ reference, confirming the matching
asymptotic complexity class. 
This behaviour is further corroborated by
Fig.~\ref{fig:runtime_memory}, which extends the analysis to
larger problem sizes and additionally reports memory usage, showing
consistent scaling trends across architectures. 
The constant factors are modest: at $n = 2^{20}$ \our{} is about two times
slower than the specialised circulant-embedding code, which itself is
several times slower than a single (non-multiplying)
\texttt{cuFFT}\ call -- a ratio that simply reflects the fact that
Toeplitz--FFT performs three FFTs on a vector of length $2n$ plus a
pointwise multiplication and a padding operation.  The crucial
observation is that all three have the same slope on the log--log
plot.

\begin{figure}[t]
    \centering
    \resizebox{\linewidth}{!}{\input{NeurIPS_2026/camera_ready/images/exp11.pgf}}
    \caption{%
        Runtime and peak memory usage of \textsc{Laplex}, Dense, and \texttt{torch.fft.fft} (cuFFT backend) as a function of the number of elements~$n$ (log\textsubscript{2} scale), evaluated on an NVIDIA H100 (solid lines) and an RTX~2080~Ti with 11\,GB (dashed lines). Shaded regions denote the standard deviation across runs; red dots (dashed red circles for the RTX~2080~Ti) mark the last size that fits in GPU memory.
        \textbf{Left:} \textsc{Laplex} has a near-constant, sub-millisecond runtime on the H100 up to $n \approx 2^{19}$ and then grows linearly, staying below one second at $n = 2^{29}$, the largest size that fits. The \texttt{torch.fft.fft} baseline (single FFT call) exhibits the lowest runtime across all $n$, reflecting the cost of a single transform without additional operations. Dense runs out of memory beyond $n = 2^{16}$.
        On the RTX~2080~Ti, \textsc{Laplex} stays at about a millisecond up to $n \approx 2^{19}$ and fits up to $n = 2^{27}$; \texttt{torch.fft.fft} again provides the fastest baseline, while Dense fails already beyond $n = 2^{14}$.
        \textbf{Right:} \textsc{Laplex} memory grows linearly in $n$, to about $40$\,GB at $n = 2^{29}$. The \texttt{torch.fft.fft} baseline has the lowest memory footprint, corresponding to FFT workspace and input/output buffers only. In contrast, Dense memory grows quadratically and reaches about $64$\,GB at $n = 2^{16}$, after which it exceeds the device memory.
        On the RTX~2080~Ti, the same ordering holds: \textsc{Laplex} fits up to $n = 2^{27}$ and \texttt{torch.fft.fft} up to $n = 2^{28}$, while Dense is limited to $n = 2^{14}$.
    }
    \label{fig:runtime_memory}
\end{figure}

\subsection{Numerical accuracy on a uniform-grid (Toeplitz) problem}
\label{sec:app-fft-accuracy}

\paragraph{Setup.}
We pick \emph{uniform} anchors $a_i = b_i = i \Delta$ with
$\Delta = 1/n$, so the kernel matrix $K(a,b)_{ij}\!=\!\exp(-|i-j|\Delta)$
is symmetric Toeplitz and Toeplitz--FFT is exactly applicable.
$X \!\in\! \mathbb{R}^{B \times n}$ is sampled $\mathcal{N}(0,1)$ in
\texttt{float32}.  We compute $X K$ via four routes:
\emph{(i)} dense \texttt{float64} (the ground truth);
\emph{(ii)} dense \texttt{float32} matmul;
\emph{(iii)} Toeplitz--FFT \texttt{float32} (circulant embedding +
$3 \times$ FFT);
\emph{(iv)} \our{} \texttt{float32} (prefix/suffix scan).
We report the relative $\ell_2$ and $\ell_\infty$ errors of routes
(ii)--(iv) against the (i) reference, $B = 8$,
$n = 2^{10}, \dots, 2^{16}$ (at $n = 2^{17}$ the \texttt{float64} reference
no longer fits in device memory).

\begin{figure}[ht]
    \centering
    \resizebox{\linewidth}{!}{\input{NeurIPS_2026/camera_ready/images/exp7_toeplitz_acc.pgf}}
    \caption{\textbf{Numerical accuracy on a uniform-grid Toeplitz
        problem.}  Relative $\ell_2$ (left) and $\ell_\infty$ (right)
        error against an \texttt{float64} dense reference.  All three
        \texttt{float32} routes operate on \emph{the same} input
        tensors.
        Both fast methods improve on the dense \texttt{float32}
        matmul, whose error grows with $n$; \our{} has a lower error than
        Toeplitz--FFT at every size, despite making no use of the
        uniformity of the grid.  The dashed line marks the \texttt{float32} machine
        epsilon $\varepsilon \approx 1.19 \cdot 10^{-7}$; open circles mark
        the largest size at which the \texttt{float64} reference fits in
        GPU memory.}
    \label{fig:app-fft-accuracy}
\end{figure}

\paragraph{Observations.}
At every tested size the dense \texttt{float32} matmul has the largest
error and \our{} the smallest, in both norms; the gap to dense grows with
$n$ (see~Fig.~\ref{fig:app-fft-accuracy}).  Three remarks.

(a) The error of the dense \texttt{float32} matmul grows with $n$, visible as the
upward-tilting orange curve.  Both fast methods avoid this
growth.

(b) Toeplitz--FFT inherits the well-known
$\mathcal{O}(\log n \cdot \varepsilon)$ accuracy of a butterfly FFT
plus a small extra constant from the complex-valued IFFT and from the
single pointwise multiplication; the result is roughly flat in $n$
and noticeably better than dense \texttt{float32}.

(c) \our{} keeps the mantissa--exponent representation in
\texttt{float32} and combines partial sums with $\mathcal{O}(\varepsilon)$
rounding per element, giving an essentially $n$-independent error that
is smaller than that of Toeplitz--FFT at every tested size.  The integer-exponent rescaling means that the scan never accumulates values whose magnitudes differ by many orders of
magnitude.

We emphasise that this experiment is a fair comparison \emph{on FFT's
home turf}: a uniform grid is exactly the configuration where
Toeplitz--FFT is applicable as a baseline.  On a non-uniform grid
the FFT-based comparator simply does not run.

\subsection{CPU runtime (no GPU required)}
\label{sec:app-fft-cpu}

\paragraph{Setup.}
The equal-complexity-class claim of Section~\ref{subsec:kerlap-fft-bench}
relies on the custom CUDA scan compiled by \texttt{prefix\_suffix.py}.
A natural worry is that the conclusion may change once that specialised
kernel is unavailable.  To test this, we run
all three methods on CPU using only PyTorch built-ins
(\texttt{torch.logcumsumexp} for \our{}, \texttt{torch.fft} for the FFT
baselines), $B = 1$, \texttt{float32}, a single CPU thread; $n$ is doubled
from $2^{10}$ until the next size would exceed the $100$\,GB memory limit
of the run.

\begin{figure}[ht]
    \centering
    \resizebox{0.7\linewidth}{!}{\input{NeurIPS_2026/camera_ready/images/exp8_nocuda_fft.pgf}}
    \caption{\textbf{CPU runtime.}  All three methods on a single CPU
        thread, log--log axes.  The dashed line is an
        $\mathcal{O}(n \log n)$ reference fitted to the last
        \texttt{torch.fft.fft} measurement.  All curves are parallel
        to that reference, confirming that the matching-complexity
        result of the GPU experiment is hardware-independent.  Red dots
        mark the last size that fits in the $100$\,GB memory limit of the
        run.}
    \label{fig:app-fft-cpu}
\end{figure}

\paragraph{Observations.}
The constants on CPU are larger than on GPU, but the scaling pattern is
identical to Figure~\ref{fig:app_kerlap-vs-fft}: three parallel lines on
the log--log plot up to the largest sizes that fit ($n = 2^{28}$ for
\our{}, $2^{29}$ for Toeplitz--FFT and $2^{30}$ for the FFT), with \our{}
about an order of magnitude slower than Toeplitz--FFT and nearly two
orders of magnitude slower than a single FFT call.  Notably, while \our{}
pays a larger constant penalty over Toeplitz--FFT on CPU than on GPU, the
dependence on $n$ is the same on both
devices, so the relative cost stays roughly constant over the tested range.

The wider gap on CPU is mostly a property of PyTorch's
\texttt{torch.logcumsumexp}, which is single-threaded and not as
heavily optimized as the MKL-backed FFT.  The gap shrinks substantially
on GPU (where our custom CUDA kernel beats a generic primitive) and is
therefore a property of the reference implementation in
\texttt{laplex\_noCUDA.py}, not of the algorithm itself (see~Fig.~\ref{fig:app-fft-cpu}).

\subsection{Differentiability with respect to the anchor sets}
\label{sec:app-fft-grad}

A subtler difference between \our{} and the FFT-based baselines is
that the FFT path supports gradients only with respect to the
\emph{input vector} and shared kernel parameters such as the grid spacing or length scale, not with respect to individual anchor positions.  This is intrinsic, not an oversight: the Toeplitz
structure is built from a fixed grid spacing $\Delta$ rather than an
explicit pair of tensors $(a, b)$, so there is nothing
``differentiable'' on the anchor side.  In the language of
\texttt{torch.autograd}, attempting
\texttt{loss.backward()} with \texttt{a.requires\_grad\_(True)}
returns \texttt{None} or a zero gradient, depending on how the grid
spacing was constructed.

\begin{table}[ht]
    \centering
    \small
    \caption{Differentiability and applicability of the three
        $\mathcal{O}(n\log n)$ comparators.  ``$\checkmark$'': supported
        and used in our experiments.  ``$-$'': not supported.}
    \label{tab:app-fft-grad}
    \begin{tabular}{lccc}
        \toprule
        \textbf{Method} &
            \makecell{Irregular grid \\ $a, b$ arbitrary} &
            \makecell{$\partial / \partial x$ \\ (input vec.)} &
            \makecell{$\partial / \partial a$, $\partial / \partial b$ \\ (anchors)} \\
        \midrule
        Single \texttt{cuFFT}        & not a kernel matvec   & $\checkmark$ & -- \\
        Toeplitz--FFT                & $-$                   & $\checkmark$ & -- \\
        \our{}                       & $\checkmark$          & $\checkmark$ & $\checkmark$ \\
        \bottomrule
    \end{tabular}
\end{table}

This category-level difference is summarized in Table~\ref{tab:app-fft-grad} and matters in any setting where the
anchors are themselves trainable parameters: feature lookup tables,
node embeddings in graph kernel methods or end-to-end-trained Gaussian process models.  The
Toeplitz--FFT comparator is simply unavailable in those settings.

\subsection{Other $\mathcal{O}(n\log n)$ comparators}
\label{sec:app-fft-others}

For completeness we briefly mention two further fast comparators that
do \emph{not} appear in our benchmark, and explain why.

\paragraph{Non-uniform FFT (NUFFT).}
The non-uniform FFT~\citep{greengard2004accelerating, dutt1993fastfourier}
extends the FFT to irregular grids in $\mathcal{O}(n\log n)$ time, but
only for \emph{Fourier-type} kernels of the form
$\exp(2\pi i\, a_i\, b_j)$.  The Laplace kernel $\exp(-|a_i - b_j|)$
is not in this family, so NUFFT is not directly applicable: it would
require approximating the Laplace decay by a finite Fourier series
with bounded support, which introduces truncation and quadrature errors.
NUFFT is therefore the right tool for kernels where the FFT spectrum
is the natural representation, but not for the operator we study.

\paragraph{Fast Multipole Method (FMM) and Hierarchical Semi-Separable matrices.}
The 1-D Laplace kernel falls within the broad class to which the
Fast Multipole Method~\citep{greengard1987fastalgorithm} and
Hierarchically Semi-Separable matrix
representations~\citep{xia2010fastalgorithms} apply.  In principle these
methods achieve $\mathcal{O}(n \log n)$ matvec on irregular grids by
hierarchical low-rank factorization of off-diagonal blocks.  For the
one-dimensional exponential kernel these blocks are exactly rank one, since
$e^{-|a_i-b_j|}=e^{a_i}e^{-b_j}$ for $a_i\le b_j$, so a semiseparable
representation is exact~\citep{laplexcr_foremanmackey2017celerite}; \our{}
exploits the same ordered structure through sorting, interval aggregation and
parallel scans.  Modern,
PyTorch-friendly, GPU-accelerated, and \emph{differentiable} FMM
implementations for the Laplace kernel are not, to our knowledge,
publicly available, and recreating one is a non-trivial software
project orthogonal to the contribution of this paper.  The advantage
of \our{}'s scan-based decomposition is that it is mechanically
simple, end-to-end differentiable, and exact (within
\texttt{float32}\ rounding) by construction; FMM is more general
(arbitrary smooth kernels) but, for general kernels, achieves $\mathcal{O}(n\log n)$ only
through controlled approximation, with an accuracy/cost trade-off
that is non-trivial to tune.

\subsection{Summary}
\label{sec:app-fft-summary}

The FFT comparison is best summarized as a containment statement plus a
constant-factor measurement. Uniform-grid Laplace Toeplitz multiplication
is a special case of \our{} evaluated by a highly optimized FFT route.
\our{} pays a moderate constant factor to move from that fixed slice to
arbitrary trainable anchors, while preserving exactness and improving
\texttt{float32} accuracy in the tested Toeplitz setting.

\begin{itemize}
    \item \textbf{Same asymptotic class.} On GPU
        (Sec.~\ref{sec:app-fft-gpu}) and CPU
        (Sec.~\ref{sec:app-fft-cpu}), \our{}, Toeplitz--FFT, and a single
        FFT call follow parallel $\mathcal{O}(n\log n)$ curves.

    \item \textbf{Special-case constants.} Toeplitz--FFT is faster because
        it assumes a uniform grid and uses a circulant embedding. \our{} is
        slower by a roughly constant factor over the tested range but accepts arbitrary anchors and can
        train them.

    \item \textbf{Numerical accuracy.} On the uniform-grid problem where
        Toeplitz--FFT is valid, \our{} gives lower relative \texttt{float32}
        error than both dense matmul and Toeplitz--FFT
        (Sec.~\ref{sec:app-fft-accuracy}).

    \item \textbf{Differentiability.} Toeplitz--FFT has no gradients with
        respect to arbitrary anchor locations, because those locations are
        not model parameters in the FFT formulation. \our{} differentiates
        through values and anchors.

    \item \textbf{Other comparators.} NUFFT targets Fourier-type kernels,
        not the Laplace kernel directly; FMM/HSS methods are broader; for this kernel a semiseparable representation is exact, but they
        lack a standard differentiable GPU implementation
        for this setting (Sec.~\ref{sec:app-fft-others}).
\end{itemize}
\section{Random Fourier Features: value and gradient behavior}
\label{sec:detailed-rff}

RFF is the natural irregular-grid comparator because it approximates the
same shift-invariant Laplace kernel without requiring a uniform grid. The
comparison is therefore not between unrelated models. It is between an
exact ordered-kernel algorithm (\our{}) and a finite Monte Carlo
factorization of the same kernel (RFF). This section shows why that
distinction becomes decisive when the anchor locations are learnable: value
errors shrink slowly with the feature count, while anchor-gradient errors
inherit the heavy tails of the Cauchy spectral density.
\subsection{Setup and value accuracy}
\label{sec:app-rff-value}

\paragraph{The estimator.}
Random Fourier Features~\citep{rahimi2007random} approximate a
shift-invariant kernel $k(u-v)$ via Bochner's theorem: if $p$ is the
spectral density of $k$, then
$k(u-v) = \mathbb{E}_{w \sim p}\,e^{i w(u-v)}
       = \mathbb{E}_{w \sim p,\, c \sim \mathcal{U}[0, 2\pi]}\,
        2 \cos(w u + c)\cos(w v + c)$,
which yields the unbiased Monte Carlo estimator
\begin{equation}
    \widehat{k}_D(u, v)
    \;=\; \tfrac{2}{D}\sum_{d = 1}^{D}
          \cos(w_d u + c_d)\, \cos(w_d v + c_d),
    \qquad
    w_d \sim p,\ \ c_d \sim \mathcal{U}[0, 2\pi].
    \label{eq:rff-est}
\end{equation}
The corresponding feature matrices $\Phi(a) \in \mathbb{R}^{n \times D}$
and $\Phi(b) \in \mathbb{R}^{k \times D}$ are dense; the forward pass
costs $\mathcal{O}((n + k) D)$.  For the Laplace kernel
$k(u - v) = \exp(-|u - v|)$ in one dimension, $p$ is the standard
Cauchy distribution
$p(w) = \pi^{-1} (1 + w^{2})^{-1}$, sampled in our implementation via
the inverse-CDF transform $w_d = \tan(\pi (U_d - \tfrac12))$,
$U_d \sim \mathcal{U}[0, 1]$.  This is the only spectral baseline that
applies to the irregular node sets that motivate \our{}.

\paragraph{Experimental protocol.}
$n = k = 2^{16}$, $B = 1$, $a, b \sim \mathcal{N}(0, 1)$ independently
and irregularly.  For each $D = 1, 2, 4, \ldots$ (until the feature
matrices no longer fit in device memory) we evaluate
the RFF estimator $10$ times with independent random features
$(w_d, c_d)_{d=1}^{D}$ and report the median and the $[\min, \max]$
range of the relative $\ell_2$ error
$\|\widehat{y} - y\|_2 / \|y\|_2$, with $y = x^{\!\top}\!K(a, b)$
computed by \our{} (which we verified against an
$\mathcal{O}(n^{2})$ dense matmul up to $n = 2048$, relative error
$< 10^{-6}$).  Gradients
$\partial L / \partial a$ and $\partial L / \partial b$ are computed
by back-propagating $L = \mathbf{1}^{\!\top}\!\widehat{y}$ through
either the RFF estimator or \our{}.  

\paragraph{Value error.}
The empirical value error tracks the classical
$\mathcal{O}(1/\sqrt{D})$ concentration rate
(Fig.~\ref{fig:kerlap-vs-rff}(b) of the main text):
the median error falls roughly as $1/\sqrt{D}$, from above $100\%$ at
$D = 1$ to a few percent at $D = 2048$.  The
$[\min, \max]$ band stays narrow because the summands in
\eqref{eq:rff-est} are bounded by $1$ and Hoeffding's inequality
applies; the worst-case-to-median ratio across seeds stays within a small
constant (at most $\sim\!3\times$) for every tested $D$.

\subsection{Gradient with respect to anchor positions}
\label{sec:app-rff-grad}

\paragraph{A theoretical obstruction.}
Differentiating the RFF estimator with respect to the node locations
introduces a frequency factor:
\begin{equation}
    \frac{\partial \widehat{k}_D(u,v)}{\partial u}
    \;=\; -\tfrac{2}{D} \sum_{d=1}^{D}
          w_d\, \sin(w_d u + c_d)\, \cos(w_d v + c_d).
    \label{eq:rff-grad}
\end{equation}
Each summand in \eqref{eq:rff-grad} is bounded only by $|w_d|$, and
for the Laplace kernel $w_d$ is drawn from the standard Cauchy
distribution, whose second moment is infinite.  The trigonometric factors do not remove this: for one summand
$g_d=-2w_d\sin(w_du+c_d)\cos(w_dv+c_d)$, averaging over the random phase gives
\[
\mathbb{E}_{c_d}\bigl[g_d^{2}\mid w_d\bigr]
= w_d^{2}\bigl(\sin^{2}(w_d(u-v))+\tfrac12\bigr)\;\ge\;\tfrac12\,w_d^{2},
\]
so every summand has an infinite second moment.

The estimator therefore has an \emph{infinite second moment}: the
standard concentration tools (Hoeffding, Bernstein) do not bound its
deviations, and the canonical $1/\sqrt{D}$ rate of the central limit theorem does not apply.
The error distribution is therefore heavy-tailed, dominated by
occasional large-frequency draws, and concentration as
$D \to \infty$ is not guaranteed
\citep{sutherland2015error}.  \our{}, in contrast, delivers gradients
that are exact by construction (up to \texttt{float32} round-off) at
the same $\mathcal{O}(n \log n)$ cost as the forward pass.

\paragraph{Empirical verification.}
Figure~\ref{fig:app-rff-grad} reports the relative error of
$\widehat{y}$, $\widehat{\nabla_a L}$, and $\widehat{\nabla_b L}$
(medians and $[\min, \max]$ ranges across $10$ seeds).
Three qualitative differences from the value case are visible.

\emph{(i) Broken $1/\sqrt{D}$ scaling.}
The gradient error with respect to $a$ decreases more slowly than the
value error and stops decreasing at a few percent for the largest $D$;
the gradient with respect to $b$ does not measurably decrease
at all, saturating well above $100\%$ relative error across
the entire sweep.

\emph{(ii) Heavy tails.}
The worst-case-to-median ratio across seeds exceeds $10\times$ for
$\widehat{\nabla_a L}$ at several values of $D$ and reaches over
$100\times$ for $\widehat{\nabla_b L}$, consistent with the
Cauchy-sampled frequency factor in~\eqref{eq:rff-grad}.

\emph{(iii) Asymmetry between $\nabla_a$ and $\nabla_b$.}
Because the loss $L = \mathbf{1}^{\!\top}(x^{\!\top}\!K)$ sums the
$x$-weighted entries over rows, the two gradient norms differ by two
orders of magnitude in our setup
($\|\nabla_a\| \approx 2.5 \cdot 10^{6}$ versus
$\|\nabla_b\| \approx 2.3 \cdot 10^{4}$), so the same absolute
Cauchy-driven noise produces dramatically larger \emph{relative}
error on $\nabla_b$.

\begin{figure}[t]
    \centering
    \resizebox{\linewidth}{!}{\input{NeurIPS_2026/camera_ready/images/rff_grad.pgf}}
    \caption{%
        \textbf{RFF gradient accuracy on the Laplace kernel.}
        $n = k = 2^{16}$, median (line) with $[\min, \max]$ range
        (shaded) over $10$ seeds.
        \emph{Left:} relative error of the value, $\nabla_a L$, and
        $\nabla_b L$ versus $D$; the dashed line is the
        $\mathcal{O}(1/\sqrt{D})$ reference established for the value.
        The value error tracks the reference, the $a$-gradient lags,
        and the $b$-gradient does not measurably improve with $D$.
        \emph{Right:} worst-case-to-median error ratio across seeds.
        Values concentrate (within $\sim\!3\times$); gradients exhibit
        heavy tails ($10\times$ to over $100\times$), consistent with the
        unbounded variance of the Cauchy-sampled frequencies.  Red dots
        mark the largest $D$ that fits in GPU memory.
    }
    \label{fig:app-rff-grad}
\end{figure}

\paragraph{Takeaway.}
RFF is a useful approximation when one needs a fixed-kernel value estimate
and can choose a feature count large enough for the desired tolerance. It is
a poor replacement for \our{} as a learnable operator over Laplace anchors.
It is slower than \our{} at every tested feature count: at small $D$ its value
error is still large, and at feature counts where the value becomes
accurate it is much slower; and the anchor-gradient
estimator has an infinite second moment because the Laplace spectrum is Cauchy.
\our{} computes the same kernel product exactly, including anchor
gradients, at a cost independent of any feature dimension.

\section{Logistic regression: full $\pmb{\tau}$-sweep for the frozen baseline}
\label{sec:logreg-frozen-app}

The main text reports the frozen-projection experiment at a single
representative temperature (kernel length scale), $\tau = 0.1$.  The frozen \our{}
projection maps $x\in\mathbb{R}^{d}$ to $(s\odot x)^{\top}\our_{\tau}(a,b)\in\mathbb{R}^{m}$,
where $\our_{\tau}$ is the temperature-scaled kernel of
Sec.~\ref{sec:learnable-routing} with $t=\tau$, the anchors are $a\sim\mathcal{N}(0,I_{d})$
and $b\sim\mathcal{N}(0,I_{m})$, and $s\in\{\pm1\}^{d}$ is a random sign
vector; all three are drawn once and kept fixed.  Here we expose the
temperature sweep; see~Table~\ref{tab:logreg-frozen-tau-sweep}.  The point is not to tune aggressively, but to show
that the behavior is stable: sharp kernels behave like hard
gathering, smoother kernels behave more like random projections, and a
wide middle range remains competitive.  CountSketch, PCA, and RP are
repeated from Table~\ref{tab:logreg-merged} for reference.

\paragraph{Standard deviations and significance for the trainable columns.}
The main-text Table~\ref{tab:logreg-merged} reports the mean over $5$ seeds; we
also report dispersion and unpaired significance tests for the two \emph{trainable}
columns. The seeds are independent across the \our{} and dense runs, so we use
Welch's $t$-test (Mann--Whitney $U$ agrees); means reproduce
Table~\ref{tab:logreg-merged}.
\begin{center}
\begin{tabular}{@{}rccc@{}}
\toprule
$m$ & \our{} (trainable) & Dense (trainable) & Welch $p$ \\
\midrule
$25$   & $43.52\pm0.77$ & $43.52\pm1.52$ & $1.000$ \\
$100$  & $43.56\pm0.46$ & $43.04\pm0.55$ & $0.146$ \\
$500$  & $43.84\pm0.89$ & $43.00\pm0.79$ & $0.153$ \\
$1000$ & $44.16\pm1.51$ & $43.80\pm0.47$ & $0.634$ \\
\bottomrule
\end{tabular}
\end{center}
The differences ($\le 0.84$~pp) are not significant at any $m$ ($p\ge0.146$):\space%
\our{} is statistically indistinguishable from the unconstrained dense
projection while using $\Theta(m+d)$ instead of $\Theta(md)$ parameters.

\begin{table}[htpb]
    \small
    \centering
    \caption{Frozen-projection logistic regression on ImageNet10:
        full sweep over the temperature (kernel length scale) $\tau$ of \our{}.
        Validation accuracy (\%) $\pm$ standard deviation over $5$
        seeds.  Only $\tau = 0.1$ is reported in the main text
        (Table~\ref{tab:logreg-merged}).}
    \label{tab:logreg-frozen-tau-sweep}
\begin{tabular}{@{}l@{\quad}c@{\quad}c@{\quad}c@{\quad}c@{\quad}c@{\quad}c@{\quad}c@{}}
\toprule
Method & 25 & 50 & 100 & 250 & 500 & 750 & 1000 \\
\midrule
CountSketch & 26.88{\tiny$\pm$2.17} & 30.92{\tiny$\pm$1.97} & 31.00{\tiny$\pm$1.65} & 32.04{\tiny$\pm$1.40} & 31.12{\tiny$\pm$1.95} & 27.76{\tiny$\pm$1.79} & 25.96{\tiny$\pm$1.09} \\
\our{} ($\tau\!=0.01$) & 29.08{\tiny$\pm$0.54} & 30.44{\tiny$\pm$1.89} & 32.20{\tiny$\pm$1.74} & 30.92{\tiny$\pm$0.58} & 30.24{\tiny$\pm$2.09} & 28.56{\tiny$\pm$1.52} & 27.64{\tiny$\pm$1.88} \\
\our{} ($\tau\!=0.05$) & 26.28{\tiny$\pm$1.51} & 30.08{\tiny$\pm$2.51} & 30.24{\tiny$\pm$1.81} & 32.68{\tiny$\pm$1.32} & 31.48{\tiny$\pm$2.32} & 30.08{\tiny$\pm$1.22} & 30.08{\tiny$\pm$0.91} \\
\our{} ($\tau\!=0.1$)  & 27.96{\tiny$\pm$1.17} & 29.88{\tiny$\pm$1.04} & 31.84{\tiny$\pm$2.05} & 32.20{\tiny$\pm$1.29} & 30.92{\tiny$\pm$1.82} & 31.28{\tiny$\pm$1.57} & 31.28{\tiny$\pm$1.48} \\
\our{} ($\tau\!=0.5$)  & 27.76{\tiny$\pm$1.77} & 29.44{\tiny$\pm$0.62} & 29.12{\tiny$\pm$1.22} & 31.60{\tiny$\pm$1.22} & 31.24{\tiny$\pm$0.80} & 30.60{\tiny$\pm$1.14} & 31.64{\tiny$\pm$0.71} \\
\our{} ($\tau\!=1.0$)  & 27.28{\tiny$\pm$1.94} & 27.76{\tiny$\pm$0.85} & 29.72{\tiny$\pm$1.91} & 29.56{\tiny$\pm$1.19} & 29.56{\tiny$\pm$2.46} & 30.12{\tiny$\pm$1.06} & 29.64{\tiny$\pm$1.44} \\
PCA                 & 34.80{\tiny$\pm$0.51} & 35.92{\tiny$\pm$1.68} & 36.48{\tiny$\pm$2.09} & 35.84{\tiny$\pm$2.19} & 32.40{\tiny$\pm$1.00} & 29.32{\tiny$\pm$1.83} & 27.16{\tiny$\pm$1.95} \\
RP                  & 27.96{\tiny$\pm$0.52} & 30.28{\tiny$\pm$1.43} & 31.68{\tiny$\pm$1.52} & 32.36{\tiny$\pm$1.92} & 30.16{\tiny$\pm$0.86} & 28.04{\tiny$\pm$1.37} & 25.16{\tiny$\pm$0.86} \\
\bottomrule
\end{tabular}
\end{table}

\paragraph{Observations.}
The qualitative picture extends the main-text reading:
\emph{(i)} small $\tau$ ($0.01$, $0.05$) acts approximately like a
sharp gather and tracks CountSketch closely, with a peak near
$m = 250$ and a slow decay afterwards;
\emph{(ii)} intermediate $\tau$ ($0.1$, $0.5$) interpolates between
sharp and smooth and is the most robust to changes in $m$, peaking
later and not collapsing;
\emph{(iii)} large $\tau$ ($1.0$) behaves like a smoothed RP and
plateaus around $30\%$.  At $m = 1000$ the best frozen-\our{}
configurations ($\tau \in \{0.1, 0.5\}$) are the only methods that
maintain $> 31\%$ accuracy -- every
non-\our{} baseline (CountSketch, PCA, RP) is below $30\%$ at that
size.  The choice $\tau = 0.1$ used in the main text is a
reasonable default; in practice $\tau$ can simply be made
trainable.

\section{Classification heads: full details}
\label{sec:experiments-heads-app}

The structured-head experiment is intentionally simple: the backbone is kept fixed while only the final classification layer is replaced, allowing us to study which low-parameter head most faithfully reproduces the original classifier. The main text reports the headline
comparison; here we include the training configuration, exact parameter
counts, validation accuracies, and optimization curves. 

As backbone models, we use Swin-T \cite{liu2021swin} and MaxViT-T \cite{tu2022maxvit}, both with publicly available pretrained weights on ImageNet-1k.

\paragraph{Setup.} For Swin-T, the original head is a single linear layer $W \in \mathbb{R}^{c \times f}$; for MaxViT-T, it is the linear component of a two-layer MLP with $W_1 \in \mathbb{R}^{f_2 \times f_1}$ and $W_2 \in \mathbb{R}^{c \times f_2}$, where $c=1000$ denotes the ImageNet-1k classes. As baselines, we report both the original pretrained head and a randomly initialized full-rank dense head. We then compare two structured low-parameter alternatives at matched parameter budgets, indexed by $r$ (the number of \our{} components, or the rank of the low-rank factorization):

\noindent\emph{(i)} the \our{} head,
\begin{equation*}
    \mathrm{LAPLEX}^{(r)}
    \;=\; \sum_{i=1}^{r} w_i\,L_i,
    \qquad
    L_i \in \mathbb{R}^{c \times f},
\end{equation*}
where each $L_i$ is a LAPLEX matrix parameterized by an independent anchor pair $(a_i,b_i)$ and $w_i$ are trained scalar mixing weights;

\noindent\emph{(ii)} the low-rank factorization,
\begin{equation*}
    \mathrm{LowRank}^{(r)}
    = U V^{\!\top},
    \qquad
    U \in \mathbb{R}^{c \times r},
    \quad
    V \in \mathbb{R}^{f \times r}.
\end{equation*}

\noindent All structured heads are trained with the backbone frozen on the full ImageNet-1k training split. To ensure fair comparison, the optimization setup is kept identical across all ranks and both backbones (Table~\ref{tab:heads-training-config}).

\begin{table}[ht]
\centering
\small
\caption{\small Training configuration shared by all classification-head experiments.}
\label{tab:heads-training-config}
\begin{tabular}{l c}
\toprule
\textbf{Setting} & \textbf{Value} \\
\cmidrule(l{0pt}r{0pt}){1-2}

Optimizer & AdamW ($\beta_1=0.9$, $\beta_2=0.999$, $\epsilon=10^{-8}$) \\
Epochs & 30 \\
Batch size & 128 \\

\cmidrule(lr){1-2}
\our{}$^{(r)}$ \& LowRank$^{(r)}$ lr / weight decay& $10^{-4}$ / $0$ \\
FullRank \& bias lr / weight decay & $10^{-5}$ / $0.05$ \\

\cmidrule(lr){1-2}
Scheduler & cosine, $\eta_{\min}=10^{-6}$, no warmup \\

\cmidrule(lr){1-2}
Seeds & $0$, $1$, $2$, $3$, $4$ (deterministic) \\

\bottomrule
\end{tabular}
\end{table}

\paragraph{Results.}
Table~\ref{tab:head_comparison} reports parameter counts and achieved accuracy for both Swin-T and MaxViT-T. Across both backbones, the \our{}$^{(r)}$ head consistently outperforms LowRank$^{(r)}$ at matched parameter budgets, with the largest margin observed at low rank ($r=4$). As the rank increases, performance improves for both methods; however, \our{}$^{(r)}$ remains superior across all settings. At $r=16$, \our{}$^{(r)}$ reduces the gap to the original pretrained head to approximately $1$ percentage point on MaxViT-T and $4$ percentage points on Swin-T, while using only $\sim$5\% and $\sim$4\% of the original parameters, respectively.

Figure~\ref{fig:swin_maxvitl-app} further illustrates training dynamics. Across all ranks and both backbones, \our{}$^{(r)}$ exhibits faster convergence and consistently lower training and validation loss compared to LowRank$^{(r)}$,  reaching equal or better final values.

\begin{table}[h]
\centering
\small
\caption{\small Comparison of \our{}$^{(r)}$ and LowRank$^{(r)}$ classification heads on ImageNet-1k. \our{}$^{(r)}$ consistently outperforms LowRank$^{(r)}$ at matched parameter count, while validation accuracy increases with rank for both methods.}
\label{tab:head_comparison}

\vspace{0.5em} 

\begin{subtable}[t]{0.48\linewidth}
\centering
\caption{Swin-T}
\begin{tabular}{l l c}
\toprule
\textbf{Head} & \textbf{Params} & \textbf{ACC} \\
\midrule
Original pretrained & $768$\,K & $81.5$ \\
\midrule
Random FullRank & $768$\,K & 81.37{\tiny$\pm$0.03} \\
\midrule
\our{}$^{(4)}$  & $8.1$\,K  & 54.04{\tiny$\pm$0.29} \\
LowRank$^{(4)}$    & $8.1$\,K  & 10.67{\tiny$\pm$0.26} \\
\midrule
\our$^{(8)}$  & $15.2$\,K & 71.72{\tiny$\pm$0.17} \\
LowRank$^{(8)}$    & $15.1$\,K & 46.80{\tiny$\pm$0.15} \\
\midrule
\our$^{(16)}$ & $29.3$\,K & 77.68{\tiny$\pm$0.09} \\
LowRank$^{(16)}$   & $29.3$\,K & 67.87{\tiny$\pm$0.14} \\
\bottomrule
\end{tabular}
\end{subtable}
\hfill
\begin{subtable}[t]{0.48\linewidth}
\centering
\caption{MaxViT-T}
\begin{tabular}{l l c}
\toprule
\textbf{Head} & \textbf{Params} & \textbf{ACC} \\
\midrule
Original pretrained  & $774$\,K & $83.7$ \\
\midrule
Random FullRank & $774$\,K & 83.37{\tiny$\pm$0.04} \\
\midrule
\our$^{(4)}$  & $11.7$\,K & 78.61{\tiny$\pm$0.16} \\
LowRank$^{(4)}$  & $11.7$\,K & 17.29{\tiny$\pm$2.31} \\
\midrule
\our$^{(8)}$  & $21.8$\,K & 82.07{\tiny$\pm$0.06} \\
LowRank$^{(8)}$  & $21.8$\,K & 63.80{\tiny$\pm$1.31} \\
\midrule
\our$^{(16)}$ & $42.1$\,K & 82.65{\tiny$\pm$0.03} \\
LowRank$^{(16)}$ & $42.1$\,K & 78.19{\tiny$\pm$0.24} \\
\bottomrule
\end{tabular}
\end{subtable}

\end{table}

\begin{figure}[ht]
    \centering
    \includegraphics[width=\textwidth]{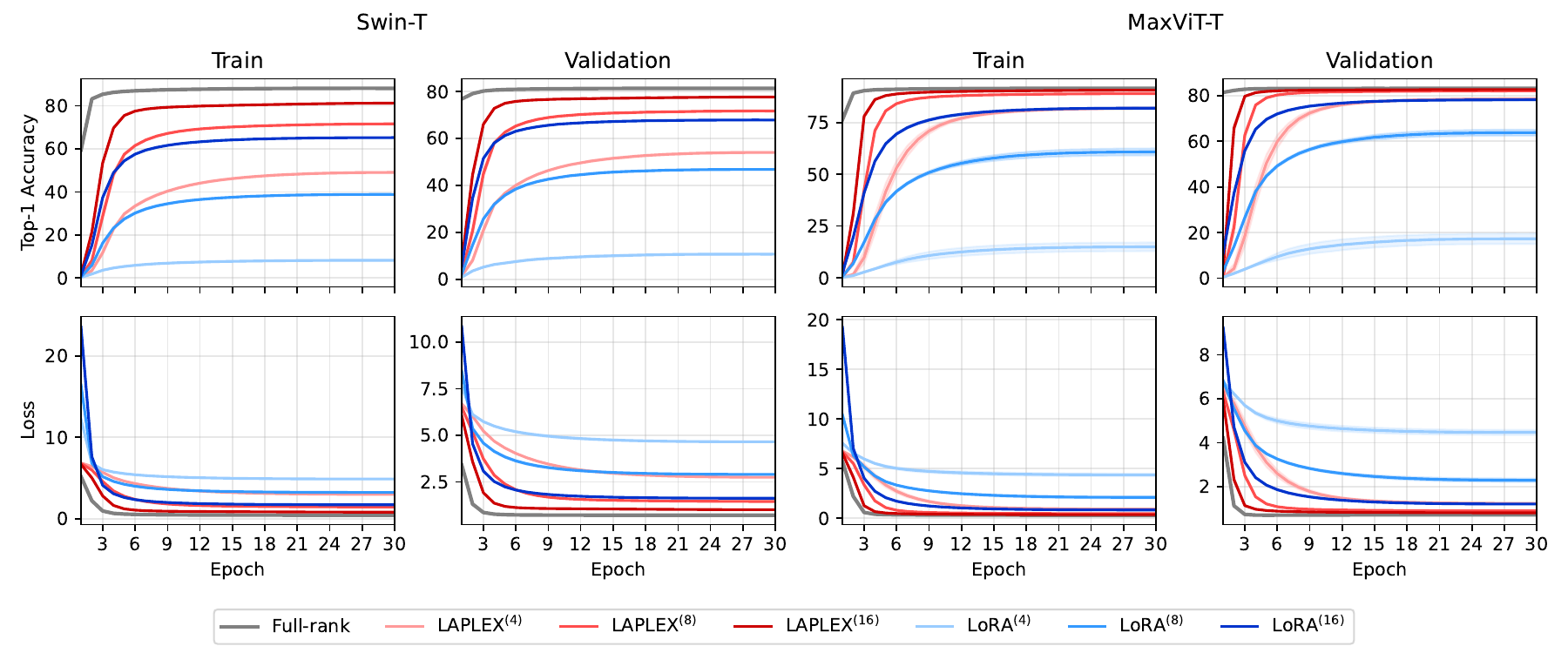}
    \caption{\small Training and validation curves on Swin-T and MaxViT-T for a FullRank (gray), \our{} (red) and LowRank (blue) heads at ranks ${4, 8, 16}$. Each curve represents the mean over 5 runs, with shading denoting standard deviation. \our{} converges faster and to consistently lower training and validation loss than LowRank across the entire schedule, irrespective of rank.}
    \label{fig:swin_maxvitl-app}
\end{figure}

\section{\texorpdfstring{Density modeling: DIV2K protocol and ImageNette benchmark}{Density modeling: DIV2K protocol and ImageNette benchmark}}\space%
\label{sec:density-model-app}

\paragraph{Relation between DIV2K and ImageNette.}
DIV2K-1024 in Sec.~\ref{sec:experiments-div2k} is the scaling stress test at
$n=3{,}145{,}728$, using a rank-$1000$ structured factor. The ImageNette-320
benchmark below studies multi-component variants and admits an exact
empirical-Gaussian reference through the sample-space SVD. These are separate,
complementary experiments, with different training splits and schedules.

\paragraph{DIV2K-1024 protocol.}
Images are resized and cropped to $1024\times1024$, flattened into RGB
vectors and preprocessed using the empirical mean. We use the $800/100$
training/validation split and the low-rank-plus-diagonal Gaussian template
of Sec.~\ref{sec:experiments-div2k}. Both the LAPLEX and dense low-rank
factors are trained with Adam at learning rate $2\times10^{-3}$, batch size
$8$, for $20$ epochs. The Woodbury likelihood is evaluated in fp64. The
LAPLEX factor has latent width $1000$ and $13.6$M trainable parameters;
the dense factor has rank $4$ and $15.7$M parameters. The ImageNette
training configuration below is specific to that benchmark and does not
replace this DIV2K protocol.

\subsection{ImageNette-320 benchmark}
\label{sec:imagenette-density-benchmark}

The density experiment is the most concrete demonstration of what the
Gram primitive buys: a high-dimensional covariance model whose structure
is learned from flat data, rather than imposed by image geometry.  
The point is to show how \our{} enables \emph{efficient and exact}
computation of the weighted Gramian
$A^{\!\top}\!\mathrm{diag}(D)A$ on high-dimensional data, the central
primitive that unlocks second-order methods (Gauss--Newton, natural
gradient, Fisher-based curvature, Woodbury-accelerated Gaussian
likelihoods) at scales where forming the corresponding
$n\!\times\!n$ covariance is infeasible. We use \emph{Gaussian density
modeling} as the running illustration: a Gaussian density on flattened
images is fully specified by the mean and a structured covariance, and
the quality of that covariance is directly visible through samples and
MAP reconstructions. The same primitive underlies every other
second-order application of \our{} in this paper.

\paragraph{All models operate on the flattened pixel vector.}
Throughout this section, every model -- \our{}, the dense low-rank
baseline, the empirical-covariance truncation, and the unconstrained
empirical Gaussian reference -- treats each image as a flat
$n$-dimensional vector: pixels are not
arranged on a grid, no convolution or patching is applied, and no
positional encoding is provided. Spatial-color structure in the learned covariance is therefore
learned from the data rather than imposed by the architecture.

\paragraph{Why ImageNette-320.}
We benchmark on \emph{flattened} ImageNette-$320$
($n{=}320^{2}\!\cdot\!3{=}307{,}200$, $N{=}9{,}469$ training images,
$10$ classes). The pixel dimension is large enough that natural
images display rich spatial-color structure \emph{and} the
unconstrained empirical covariance
$\hat\Sigma\!\in\!\mathbb{R}^{n\times n}$ is no longer storable in
practice ($\sim\!377$\,GB in fp32). At the
same time, the modest number of training images ($N{\ll}n$) lets us
still \emph{evaluate} the full empirical Gaussian exactly through an
$N\!\times\!N$ ``SVD trick'': the top eigenvectors of $\hat\Sigma$
are recovered from the centered-data Gram
$X_{c}X_{c}^{\!\top}/(N{-}1)$ without ever materializing the
$n\!\times\!n$ matrix. This combination makes ImageNette-$320$ a
regime where structured-covariance parameterizations such as \our{}
face a non-trivial empirical-covariance reference, and where
matched-budget dense families are severely crippled by their
effective-rank constraint.

\paragraph{Low-rank Gaussian modeling in high dimensions.}
A standard recipe for building a multivariate Gaussian density
$\mathcal{N}(m,\Sigma)$ in high-dimensional $\mathbb{R}^{n}$ when the
full $n\!\times\!n$ covariance is neither storable nor factorisable
is to map a standard-normal latent $z\in\mathbb{R}^{k}$, $k\ll n$, through a linear map
$F\in\mathbb{R}^{n\times k}$ and add
diagonal noise:
\begin{equation}
\label{eq:generative-standard-app}
x \;=\; m \;+\; F\,z \;+\; d\odot\varepsilon,
\quad z\sim\mathcal{N}(0,I_{k}),\;
\varepsilon\sim\mathcal{N}(0,I_{n}),
\end{equation}
which induces
\[
\Sigma \;=\; D + FF^{\!\top},
\qquad
D = \mathrm{diag}(d^{2}).
\]
A latent covariance factor, such as the triangular $L$ below, is absorbed into $F$. This
diagonal-plus-low-rank factor-analysis template underlies
probabilistic PCA \citep{tipping1999probabilistic}, mixtures of PPCA
and factor analyzers
\citep{tipping1999mixtures,ghahramani1996algorithm}, EPLL-style
patch-based image priors \citep{zoran2011learning}, and efficient
GMMs of natural images \citep{richardson2018gans}.

The key practical reason for the template is the Woodbury identity,
which reduces log-likelihood evaluation from operations of size
$n\!\times\!n$ to matrix--vector operations of size $n\!\times\!k$
and a determinant of size $k\!\times\!k$:
\begin{align}
\Sigma^{-1}
&= D^{-1} - D^{-1} F \bigl(I_{k} + F^{\!\top} D^{-1} F\bigr)^{-1} F^{\!\top} D^{-1},\\
\log\det\Sigma
&= \log\det D + \log\det\bigl(I_{k} + F^{\!\top} D^{-1} F\bigr).
\end{align}
The bottleneck is therefore the $k\!\times\!k$ capacitance matrix
$M = I_{k} + F^{\!\top} D^{-1} F$.

The key design choice is the parameterization of $F$. We compare
\emph{two} matched-budget regimes on the same template
\eqref{eq:generative-standard-app}: a small budget
($\sim\!2.2$M trainable parameters) at which we run the simplest
\our{} instantiation alongside two dense baselines, and a larger
budget ($\sim\!3.5$M) at which we contrast the natural
multi-component extension of \our{} against the same two dense
baselines re-trained at the larger rank. Across both budgets we
evaluate four kinds of $F$:
\begin{itemize}
\item \our{} (single-component, $I{=}1$):
$F{=}\mathrm{diag}(w)\,A\,L^{\!\top}$, where
$A=\our(a,\phi;b,\psi)\!\in\!\mathbb{R}^{n\times k_{\mathrm{laplex}}}$ is the phased Laplace kernel (Sec.~\ref{sec:theory}) with trainable anchors $a,b$, phases $\phi,\psi$ and temperature $\tau$, and
$L\!\in\!\mathbb{R}^{k_{\mathrm{laplex}}\times k_{\mathrm{laplex}}}$
is lower-triangular with positive diagonal;
\item \our{} (multi-component, $I{=}2$):
$F{=}\sum_{i=1}^{I}\mathrm{diag}(w_{i})\,A\,L_{i}^{\!\top}$ with $A$
shared and per-component weights $w_{i}$ and \emph{arbitrary}
factors $L_{i}\!\in\!\mathbb{R}^{k_{\mathrm{laplex}}\times k_{\mathrm{laplex}}}$
(no triangular constraint);
\item dense low-rank baseline: fully free
$F\!\in\!\mathbb{R}^{n\times k_{\text{low-rank}}}$;
\item empirical-covariance truncation (probabilistic PCA): the top
$k_{\mathrm{cov}}$ eigenvectors of $\hat\Sigma$, with
$\Sigma{=}V_{k_{\mathrm{cov}}}\Lambda V_{k_{\mathrm{cov}}}^{\!\top}+\sigma^{2}(I{-}V_{k_{\mathrm{cov}}}V_{k_{\mathrm{cov}}}^{\!\top})$
and $\sigma^{2}$ from \citep{tipping1999probabilistic};
\item \emph{(reference, not matched-budget)} the full empirical
Gaussian $\mathcal{N}(\hat m, \hat\Sigma+\varepsilon I)$ with optimal
ridge $\varepsilon$, evaluated exactly via the SVD trick.
\end{itemize}
The matched ranks for each budget regime are determined
algebraically by the budget formulas
$P_{\mathrm{laplex}}^{(I)}{=}n(3{+}I)+I\,k_{\mathrm{laplex}}^{\,2}+\mathcal{O}(k_{\mathrm{laplex}})$,
$P_{\mathrm{low\text{-}rank}}{=}n+n\,k_{\text{low-rank}}+\mathcal{O}(k_{\text{low-rank}}^{2})$,
$P_{\mathrm{cov}}{=}n+n\,k_{\mathrm{cov}}+k_{\mathrm{cov}}$. At
$\sim\!2.2$M they pin
$k_{\mathrm{laplex}}{=}1000$ ($I{=}1$), $k_{\text{low-rank}}{=}6$
and $k_{\mathrm{cov}}{=}6$. At $\sim\!3.5$M (matched to multi
\our{}, which has exactly $3{,}538{,}001$ trainables) the dense
families lift to $k_{\text{low-rank}}{=}11$ and
$k_{\mathrm{cov}}{=}11$, an effective rank still nearly two orders
of magnitude below the $k_{\mathrm{laplex}}{=}1000$ that \our{}
sustains in either regime.

\paragraph{Sampling and MAP reconstruction.}
From the generative form~\eqref{eq:generative-standard-app} both
sampling and reconstruction are natural. \emph{Sampling} draws
$z\sim\mathcal{N}(0,I_{k})$ and $\varepsilon\sim\mathcal{N}(0,I_{n})$
and returns $x{=}m+Fz+d\odot\varepsilon$; for the PCA reference (no
noise term, latent $z\sim\mathcal{N}(0,I_{k_{\mathrm{cov}}})$ and
$\Lambda{=}\mathrm{diag}(\lambda_{1},\ldots,\lambda_{k_{\mathrm{cov}}})$
the top-$k_{\mathrm{cov}}$ eigenvalues) this reduces to
$x{=}\hat m+V_{k_{\mathrm{cov}}}\Lambda^{1/2}z$.
\emph{MAP reconstruction} of an observed image~$x$ under the
generative model is, for our kernel and dense-LR variants, a
Woodbury-accelerated posterior mean: the MAP latent
$z^{\!*}=\arg\max_{z}p(z\mid x)$ is the unique solution of
\begin{equation}
\label{eq:map-app}
(I_{k}+F^{\!\top}D^{-1}F)\,z^{\!*}
\;=\; F^{\!\top}D^{-1}(x-m),
\qquad
\hat x_{\mathrm{MAP}} \;=\; m + F\,z^{\!*},
\end{equation}
obtained by one $k\!\times\!k$ Cholesky solve. For the
empirical-covariance truncation the noise term vanishes and the MAP
coincides with the orthogonal projection
$\hat x_{\mathrm{MAP}}=\hat m + V_{k_{\mathrm{cov}}}V_{k_{\mathrm{cov}}}^{\!\top}(x-\hat m)$.

\paragraph{Training and evaluation.}
We train all variants of \our{} and the dense low-rank baseline on
flattened ImageNette-$320$ for $25$ epochs with Adam
(lr~$2{\times}10^{-3}$, batch~$64$); the Gram matrices $G_{ij}$ and
the Cholesky of the capacitance $M$ run in fp64 via a custom
decoupled-lookback scan, the rest in fp32. The empirical-covariance
truncation is computed in closed form via the SVD trick -- no
training -- and the same eigendecomposition serves both budget
regimes ($k_{\mathrm{cov}}{=}6$ and $k_{\mathrm{cov}}{=}11$ are
just two truncation depths of the same spectrum). At the
$\sim\!2.2$M-parameter budget the three matched models have
$2{,}231{,}801$ (\our{} $I{=}1$), $2{,}150{,}442$
(dense $k_{\text{low-rank}}{=}6$) and $2{,}150{,}406$
(PCA $k_{\mathrm{cov}}{=}6$) trainables (matched within $4\%$).
At the $\sim\!3.5$M-parameter budget they have $3{,}538{,}001$
(\our{} $I{=}2$, two arbitrary $L_{i}$), $3{,}686{,}466$ (dense
$k_{\text{low-rank}}{=}11$) and $3{,}686{,}401$ (PCA
$k_{\mathrm{cov}}{=}11$) trainables (matched within $5\%$). The
unconstrained ridge-regularized empirical Gaussian reference attains higher test LL than both: with optimal $\varepsilon{=}10^{-2}$ (selected by val
LL) it effectively requires $\sim\!2.9$\,B parameters to store the
centered training data $X_{c}$ and the recovered spectrum -- roughly
$1{,}300{\times}$ more than the small matched set and
$\sim\!820{\times}$ more than the larger matched set.

\paragraph{Construction of the multi-component variant.}
At the $\sim\!2.2$M-parameter budget \our{} reduces to a single
$F{=}\mathrm{diag}(w)\,A\,L^{\!\top}$ with shared phased Laplace
kernel $A$ and one lower-triangular factor
$L\!\in\!\mathbb{R}^{k_{\mathrm{laplex}}\times k_{\mathrm{laplex}}}$.
At the larger $\sim\!3.5$M-parameter budget we evaluate the natural
\emph{multi-component} enrichment in which we allow a sum of $I$
structured factors, sharing $A$ but with per-component weights and
per-component $L_{i}$:
\begin{equation}
\label{eq:multi-component}
F \;=\; \sum_{i=1}^{I} \mathrm{diag}(w_{i})\, A\, L_{i}^{\!\top},
\quad L_{i}\!\in\!\mathbb{R}^{k_{\mathrm{laplex}}\times k_{\mathrm{laplex}}}\ \text{(arbitrary, no triangular constraint)},
\end{equation}
which yields
$\Sigma=D+FF^{\!\top}=D+\sum_{i,j}\mathrm{diag}(w_{i})\,A L_{i}^{\!\top}L_{j}\,A^{\!\top}\mathrm{diag}(w_{j})$.
The $k_{\mathrm{laplex}}\!\times\!k_{\mathrm{laplex}}$ capacitance
matrix factorizes as
\[
M = I_{k_{\mathrm{laplex}}} + F^{\!\top}D^{-1}F
  \;=\; I_{k_{\mathrm{laplex}}} + \sum_{i,j=1}^{I} L_{i}\,G_{ij}\,L_{j}^{\!\top},
\quad
G_{ij} = A^{\!\top}\mathrm{diag}(w_{i}w_{j}/d^{2})\,A.
\]
By symmetry $G_{ij}{=}G_{ji}$ there are only $I(I{+}1)/2$ distinct
Grams. Each has output dimension $k_{\mathrm{laplex}}$ and sums over
$n$ image coordinates, so the three-real-Gram identity (App.~\ref{app:phased-gram}) and
Thm.~\ref{thm:gram}, with the dimensions interchanged, give
$\mathcal{O}(k_{\mathrm{laplex}}^{2}+n\log k_{\mathrm{laplex}})$ per Gram.
Forming all distinct $G_{ij}$ therefore costs
$\mathcal{O}\!\bigl(I(I{+}1)/2\cdot
(k_{\mathrm{laplex}}^{2}+n\log k_{\mathrm{laplex}})\bigr)$.
The subsequent dense products in latent space cost
$\mathcal{O}(I^{2}k_{\mathrm{laplex}}^{3})$ with conventional matrix
multiplication, and the final Cholesky factorization costs
$\mathcal{O}(k_{\mathrm{laplex}}^{3})$. The\space%
rank of $\Sigma{-}D$ is at most $k_{\mathrm{laplex}}$ (since
$F\!\in\!\mathbb{R}^{n\times k_{\mathrm{laplex}}}$ in spite of the
sum), but the family of representable $F$ is strictly richer than
the single-component case: the sum cannot in general be
realised by a single $(w,L)$ pair. The parameterization carries an
$\mathrm{O}(k_{\mathrm{laplex}})$ gauge symmetry
$L_{i}\!\mapsto\!Q L_{i}$ for any
$Q\!\in\!\mathrm{O}(k_{\mathrm{laplex}})$ applied uniformly, which sends
$F$ to $FQ^{\!\top}$ and leaves\space%
$FF^{\!\top}$ invariant but does not obstruct Adam empirically. We
instantiate the construction at $I{\in}\{2,3\}$, both with
\emph{full $k_{\mathrm{laplex}}\!\times\!k_{\mathrm{laplex}}$}
$L_{i}$.

\begin{figure}[t]
\centering
\begin{subfigure}[b]{0.24\textwidth}
    \centering
    \includegraphics[width=\linewidth]{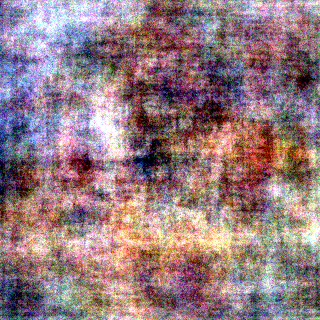}
    \caption*{$\mathcal{N}(m,\Sigma)$ sample \emph{(reference)}\\$\sim\!2.9$\,B params, LL $315{,}549$}
\end{subfigure}\hfill
\begin{subfigure}[b]{0.24\textwidth}
    \centering
    \includegraphics[width=\linewidth]{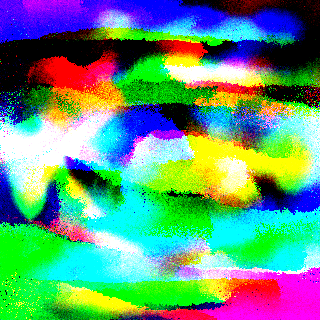}
    \caption*{\our{} ($I{=}1$) sample\\$\sim\!2.2$M params, LL $198{,}549$}
\end{subfigure}\hfill
\begin{subfigure}[b]{0.24\textwidth}
    \centering
    \includegraphics[width=\linewidth]{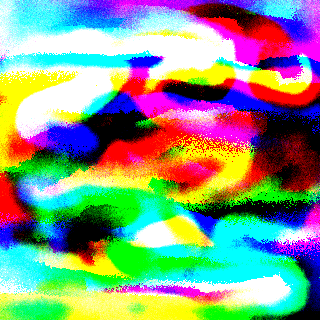}
    \caption*{\our{} ($I{=}2$) sample\\$\sim\!3.54$M params, LL $229{,}607$}
\end{subfigure}\hfill
\begin{subfigure}[b]{0.24\textwidth}
    \centering
    \includegraphics[width=\linewidth]{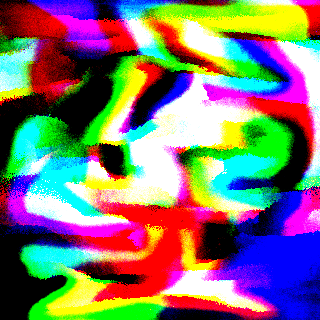}
    \caption*{\our{} ($I{=}3$) sample\\$\sim\!4.84$M params, LL $\mathbf{238{,}578}$}
\end{subfigure}
\caption{ImageNette-$320$ density: random samples drawn from four
Gaussian models, all conditioned on the same Welford-frozen
$307{,}200$-vector mean $m$, ordered left to right by trainable
budget and (equivalently for \our{}) by component count.
\textbf{Left}: full-empirical-Gaussian reference
$\Sigma_{\varepsilon}{=}\hat\Sigma+\varepsilon I$ with val-optimal
$\varepsilon{=}10^{-2}$ ($\sim\!2.9$\,B effective parameters).
\textbf{Centre-left}: \our{} single-component ($I{=}1$,
$\sim\!2.2$M trainable). \textbf{Centre-right}: \our{}
multi-component ($I{=}2$, two arbitrary $L_{i}$ factors of
Eq.~\eqref{eq:multi-component}, $\sim\!3.54$M trainable).
\textbf{Right}: \our{} multi-component ($I{=}3$, three arbitrary
$L_{i}$, $\sim\!4.84$M trainable). All three \our{} variants
reproduce the soft natural-image color blobs of the
empirical-Gaussian reference at about three orders of magnitude
fewer parameters; test LL grows monotonically with $I$, with diminishing returns between
$I{=}2$ and $I{=}3$ (gain $+31{,}058$ nats vs.\ $+8{,}971$ nats).
Matched-budget dense baselines (low-rank with
$k_{\text{low-rank}}{\in}\{6,11\}$ and PPCA truncation with
$k_{\mathrm{cov}}{\in}\{6,11\}$) produce qualitatively unstructured
samples at either of the matched budgets and are reported in
Tab.~\ref{tab:imagenette-density-app} but not shown here. MAP
reconstructions for the same models on held-out images appear in
Fig.~\ref{fig:imagenette-recon-gallery}.}
\label{fig:imagenette-samples-app}
\end{figure}

\begin{figure}[t]
\centering
\begin{subfigure}[b]{0.095\textwidth}
    \centering
    \includegraphics[width=\linewidth]{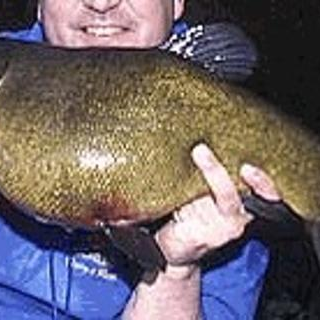}
    \caption*{\tiny tench}
\end{subfigure}\hfill
\begin{subfigure}[b]{0.095\textwidth}
    \centering
    \includegraphics[width=\linewidth]{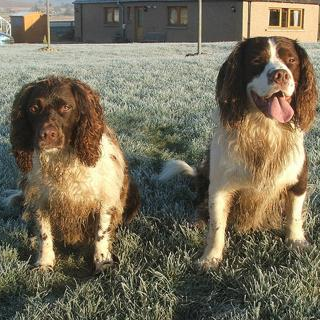}
    \caption*{\tiny springer}
\end{subfigure}\hfill
\begin{subfigure}[b]{0.095\textwidth}
    \centering
    \includegraphics[width=\linewidth]{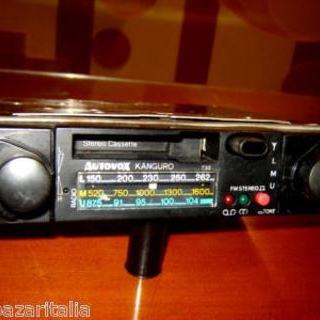}
    \caption*{\tiny cassette}
\end{subfigure}\hfill
\begin{subfigure}[b]{0.095\textwidth}
    \centering
    \includegraphics[width=\linewidth]{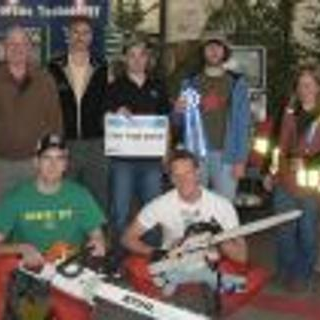}
    \caption*{\tiny chainsaw}
\end{subfigure}\hfill
\begin{subfigure}[b]{0.095\textwidth}
    \centering
    \includegraphics[width=\linewidth]{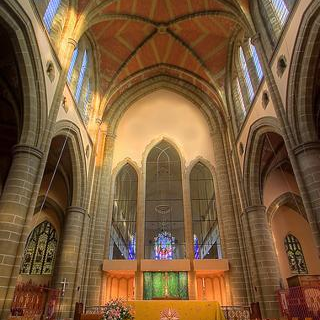}
    \caption*{\tiny church}
\end{subfigure}\hfill
\begin{subfigure}[b]{0.095\textwidth}
    \centering
    \includegraphics[width=\linewidth]{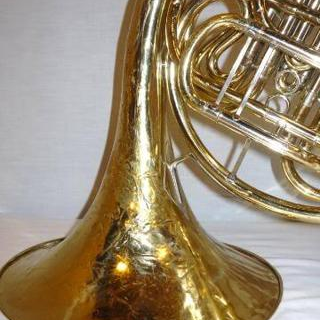}
    \caption*{\tiny horn}
\end{subfigure}\hfill
\begin{subfigure}[b]{0.095\textwidth}
    \centering
    \includegraphics[width=\linewidth]{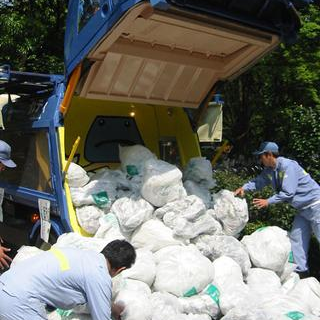}
    \caption*{\tiny truck}
\end{subfigure}\hfill
\begin{subfigure}[b]{0.095\textwidth}
    \centering
    \includegraphics[width=\linewidth]{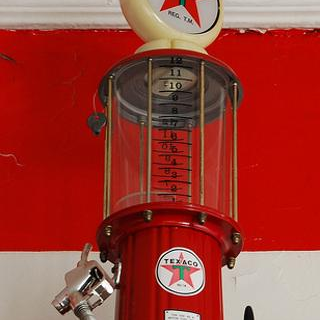}
    \caption*{\tiny gas pump}
\end{subfigure}\hfill
\begin{subfigure}[b]{0.095\textwidth}
    \centering
    \includegraphics[width=\linewidth]{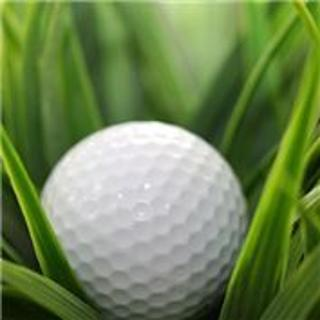}
    \caption*{\tiny golf ball}
\end{subfigure}\hfill
\begin{subfigure}[b]{0.095\textwidth}
    \centering
    \includegraphics[width=\linewidth]{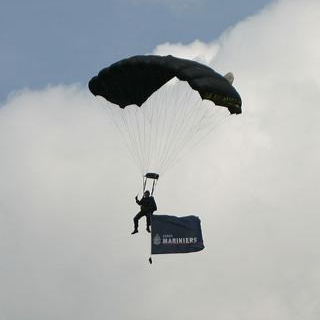}
    \caption*{\tiny parachute}
\end{subfigure}

\vspace{0.5ex}
\begin{subfigure}[b]{0.095\textwidth}
    \centering
    \includegraphics[width=\linewidth]{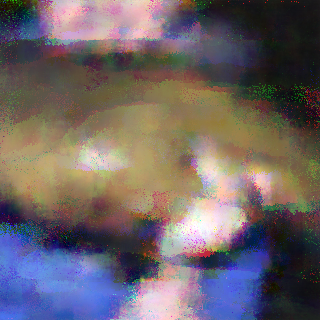}
    \caption*{\tiny $I{=}1$\\$0.121$}
\end{subfigure}\hfill
\begin{subfigure}[b]{0.095\textwidth}
    \centering
    \includegraphics[width=\linewidth]{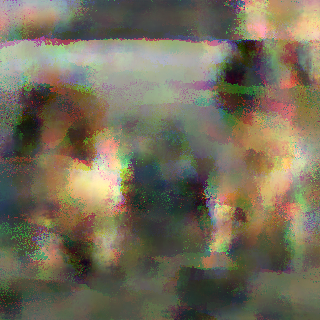}
    \caption*{\tiny $I{=}1$\\$0.147$}
\end{subfigure}\hfill
\begin{subfigure}[b]{0.095\textwidth}
    \centering
    \includegraphics[width=\linewidth]{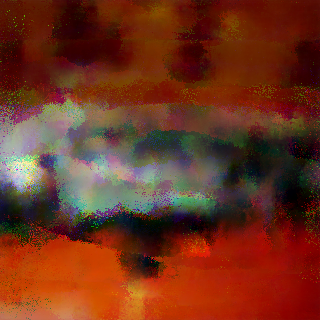}
    \caption*{\tiny $I{=}1$\\$0.128$}
\end{subfigure}\hfill
\begin{subfigure}[b]{0.095\textwidth}
    \centering
    \includegraphics[width=\linewidth]{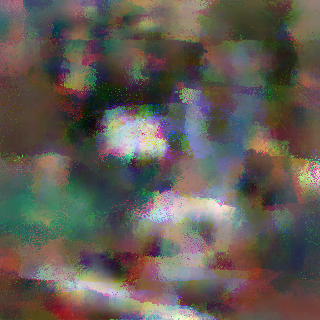}
    \caption*{\tiny $I{=}1$\\$0.122$}
\end{subfigure}\hfill
\begin{subfigure}[b]{0.095\textwidth}
    \centering
    \includegraphics[width=\linewidth]{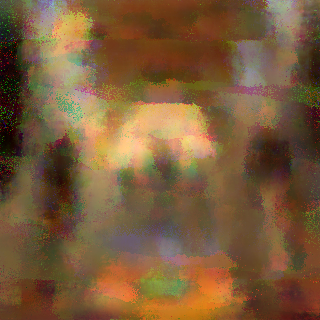}
    \caption*{\tiny $I{=}1$\\$0.109$}
\end{subfigure}\hfill
\begin{subfigure}[b]{0.095\textwidth}
    \centering
    \includegraphics[width=\linewidth]{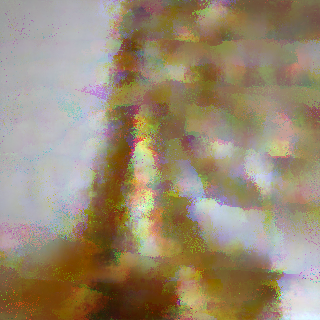}
    \caption*{\tiny $I{=}1$\\$0.141$}
\end{subfigure}\hfill
\begin{subfigure}[b]{0.095\textwidth}
    \centering
    \includegraphics[width=\linewidth]{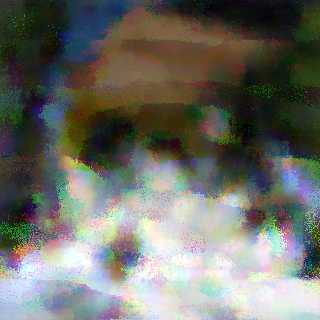}
    \caption*{\tiny $I{=}1$\\$0.143$}
\end{subfigure}\hfill
\begin{subfigure}[b]{0.095\textwidth}
    \centering
    \includegraphics[width=\linewidth]{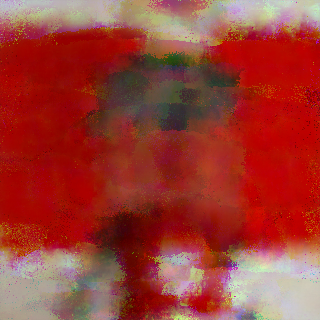}
    \caption*{\tiny $I{=}1$\\$0.110$}
\end{subfigure}\hfill
\begin{subfigure}[b]{0.095\textwidth}
    \centering
    \includegraphics[width=\linewidth]{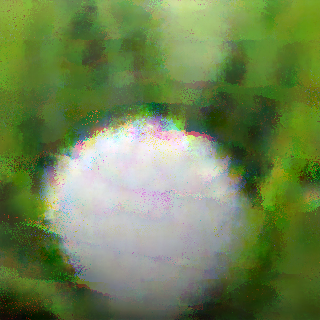}
    \caption*{\tiny $I{=}1$\\$0.081$}
\end{subfigure}\hfill
\begin{subfigure}[b]{0.095\textwidth}
    \centering
    \includegraphics[width=\linewidth]{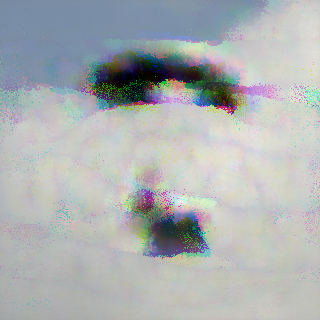}
    \caption*{\tiny $I{=}1$\\$0.086$}
\end{subfigure}

\vspace{0.5ex}
\begin{subfigure}[b]{0.095\textwidth}
    \centering
    \includegraphics[width=\linewidth]{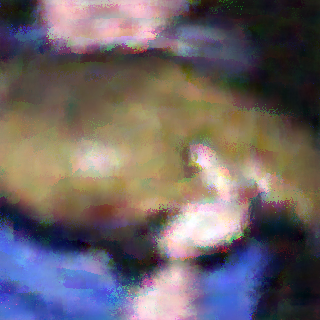}
    \caption*{\tiny $I{=}2$\\$\mathbf{0.106}$}
\end{subfigure}\hfill
\begin{subfigure}[b]{0.095\textwidth}
    \centering
    \includegraphics[width=\linewidth]{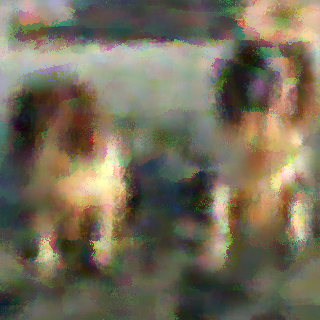}
    \caption*{\tiny $I{=}2$\\$\mathbf{0.132}$}
\end{subfigure}\hfill
\begin{subfigure}[b]{0.095\textwidth}
    \centering
    \includegraphics[width=\linewidth]{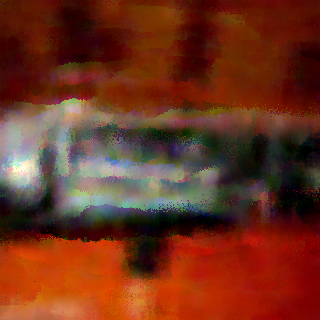}
    \caption*{\tiny $I{=}2$\\$\mathbf{0.115}$}
\end{subfigure}\hfill
\begin{subfigure}[b]{0.095\textwidth}
    \centering
    \includegraphics[width=\linewidth]{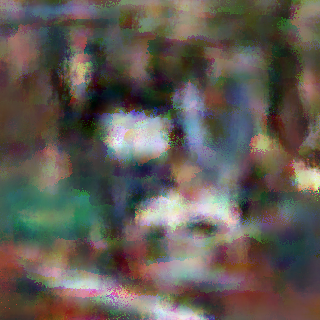}
    \caption*{\tiny $I{=}2$\\$\mathbf{0.104}$}
\end{subfigure}\hfill
\begin{subfigure}[b]{0.095\textwidth}
    \centering
    \includegraphics[width=\linewidth]{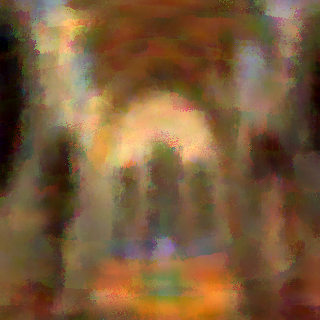}
    \caption*{\tiny $I{=}2$\\$\mathbf{0.098}$}
\end{subfigure}\hfill
\begin{subfigure}[b]{0.095\textwidth}
    \centering
    \includegraphics[width=\linewidth]{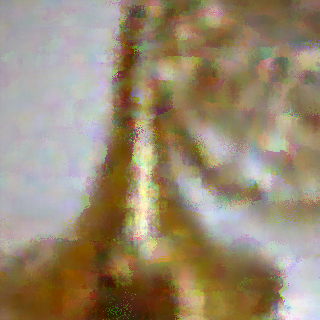}
    \caption*{\tiny $I{=}2$\\$\mathbf{0.129}$}
\end{subfigure}\hfill
\begin{subfigure}[b]{0.095\textwidth}
    \centering
    \includegraphics[width=\linewidth]{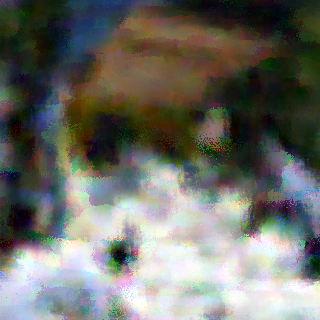}
    \caption*{\tiny $I{=}2$\\$\mathbf{0.125}$}
\end{subfigure}\hfill
\begin{subfigure}[b]{0.095\textwidth}
    \centering
    \includegraphics[width=\linewidth]{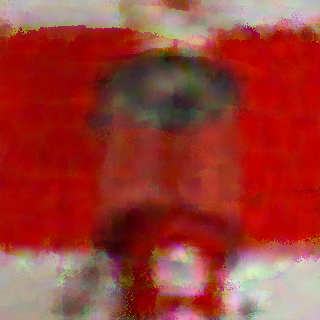}
    \caption*{\tiny $I{=}2$\\$\mathbf{0.102}$}
\end{subfigure}\hfill
\begin{subfigure}[b]{0.095\textwidth}
    \centering
    \includegraphics[width=\linewidth]{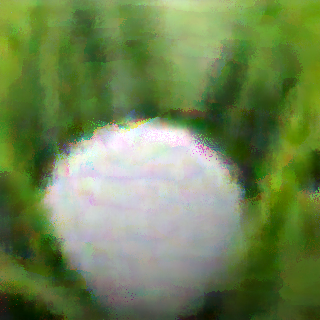}
    \caption*{\tiny $I{=}2$\\$\mathbf{0.067}$}
\end{subfigure}\hfill
\begin{subfigure}[b]{0.095\textwidth}
    \centering
    \includegraphics[width=\linewidth]{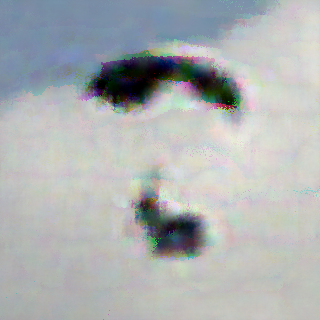}
    \caption*{\tiny $I{=}2$\\$0.069$}
\end{subfigure}

\vspace{0.5ex}
\begin{subfigure}[b]{0.095\textwidth}
    \centering
    \includegraphics[width=\linewidth]{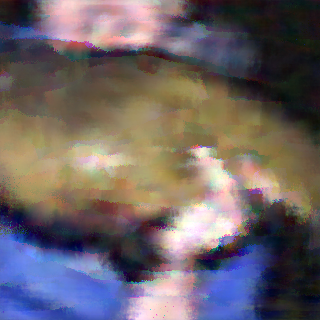}
    \caption*{\tiny $I{=}3$\\$\mathbf{0.104}$}
\end{subfigure}\hfill
\begin{subfigure}[b]{0.095\textwidth}
    \centering
    \includegraphics[width=\linewidth]{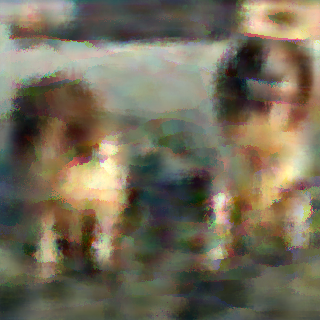}
    \caption*{\tiny $I{=}3$\\$\mathbf{0.128}$}
\end{subfigure}\hfill
\begin{subfigure}[b]{0.095\textwidth}
    \centering
    \includegraphics[width=\linewidth]{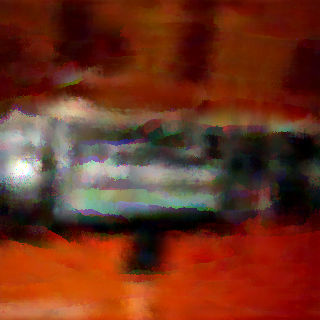}
    \caption*{\tiny $I{=}3$\\$\mathbf{0.111}$}
\end{subfigure}\hfill
\begin{subfigure}[b]{0.095\textwidth}
    \centering
    \includegraphics[width=\linewidth]{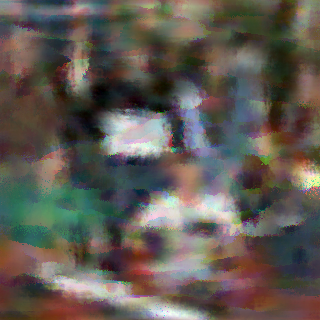}
    \caption*{\tiny $I{=}3$\\$\mathbf{0.099}$}
\end{subfigure}\hfill
\begin{subfigure}[b]{0.095\textwidth}
    \centering
    \includegraphics[width=\linewidth]{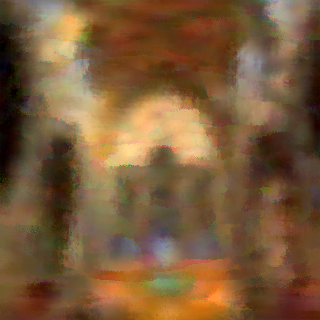}
    \caption*{\tiny $I{=}3$\\$\mathbf{0.097}$}
\end{subfigure}\hfill
\begin{subfigure}[b]{0.095\textwidth}
    \centering
    \includegraphics[width=\linewidth]{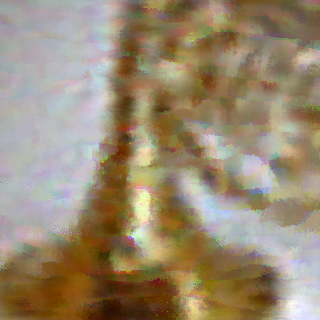}
    \caption*{\tiny $I{=}3$\\$\mathbf{0.129}$}
\end{subfigure}\hfill
\begin{subfigure}[b]{0.095\textwidth}
    \centering
    \includegraphics[width=\linewidth]{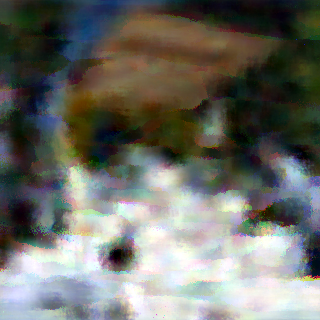}
    \caption*{\tiny $I{=}3$\\$\mathbf{0.121}$}
\end{subfigure}\hfill
\begin{subfigure}[b]{0.095\textwidth}
    \centering
    \includegraphics[width=\linewidth]{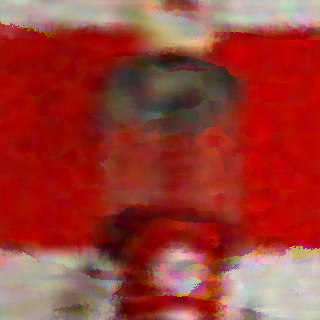}
    \caption*{\tiny $I{=}3$\\$\mathbf{0.098}$}
\end{subfigure}\hfill
\begin{subfigure}[b]{0.095\textwidth}
    \centering
    \includegraphics[width=\linewidth]{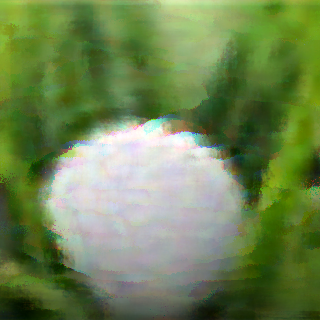}
    \caption*{\tiny $I{=}3$\\$\mathbf{0.067}$}
\end{subfigure}\hfill
\begin{subfigure}[b]{0.095\textwidth}
    \centering
    \includegraphics[width=\linewidth]{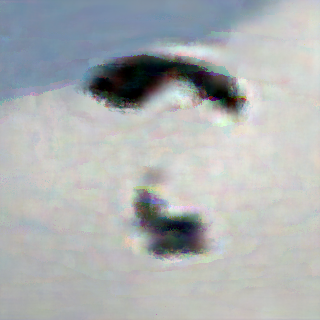}
    \caption*{\tiny $I{=}3$\\$\mathbf{0.065}$}
\end{subfigure}
\caption{\our{} MAP reconstructions on ten held-out ImageNette-$320$
images, one uniformly-random sample from each of the ten classes
(tench, English springer, cassette player, chain saw, church,
French horn, garbage truck, gas pump, golf ball, parachute; seed~$1234$).
\textbf{Row 1}: the original $320{\times}320$ inputs (no preprocessing
beyond resize-and-center-crop, flattened to $307{,}200$-vectors
before hitting the model), labelled with class.
\textbf{Row 2}: \our{} single-component ($I{=}1$, $\sim\!2.2$M
parameters); per-image RMSE shown in each sub-caption, mean
$\approx\!0.119$.
\textbf{Row 3}: \our{} multi-component ($I{=}2$, two arbitrary
$L_{i}$ factors, $\sim\!3.54$M parameters); mean RMSE
$\approx\!0.105$ ($-12\%$ vs.\ $I{=}1$) at $+59\%$ trainables.
\textbf{Row 4}: \our{} multi-component ($I{=}3$, three arbitrary
$L_{i}$, $\sim\!4.84$M parameters); mean RMSE
$\approx\!\mathbf{0.102}$ ($-14\%$ vs.\ $I{=}1$, $-3\%$ vs.\ $I{=}2$)
at $+117\%$ trainables vs.\ $I{=}1$. All three rows use the
Woodbury MAP $z^{\!*}{=}M^{-1}F^{\!\top}D^{-1}(x{-}m)$. The
ten-class progression $0.119\!\to\!0.105\!\to\!0.102$ tracks the
test-LL progression $198{,}549\!\to\!229{,}607\!\to\!238{,}578$
in Tab.~\ref{tab:imagenette-density-app}, with the marginal $I{=}3$
gain over $I{=}2$ visibly smaller than the $I{=}1{\to}I{=}2$
gain -- consistent with the diminishing-returns trend in
LL.}
\label{fig:imagenette-recon-gallery}
\end{figure}

\paragraph{Results.}
The picture
(Fig.~\ref{fig:imagenette-samples-app}, Tab.~\ref{tab:imagenette-density-app})
is the same \emph{at both budget regimes}. At
$\sim\!2.2$M trainables (Tab.~\ref{tab:imagenette-density-app},
rows 1, 4, 5) both budget-matched dense families collapse to
$\sim\!6$-dimensional latents--an unavoidable consequence of the
linear cost of storing $V_{k_{\mathrm{cov}}}$ or the dense $F$,
which pins $k_{\mathrm{cov}}\!\approx\!k_{\text{low-rank}}\!\approx\!6$
at $n{=}307{,}200$. Both produce qualitatively unstructured samples
(Tab.~\ref{tab:imagenette-density-app}, rows 4--5; not shown in
Fig.~\ref{fig:imagenette-samples-app}). \our{} converts the same
budget into an effective rank of $k_{\mathrm{laplex}}{=}1000$ and
yields \emph{spatially coherent} color blobs
(Fig.~\ref{fig:imagenette-samples-app}, middle)--contiguous patches
of green, brown or blue, even though the input is a flat $307{,}200$-vector and the model is
strictly \emph{linear} in its latent representation. On test
log-likelihood, \our{} ($I{=}1$) gains $\approx\!+154{,}000$ nats
over the PCA truncation ($198{,}549$ vs.\ $44{,}811$) and
$\approx\!+160{,}000$ over the dense low-rank baseline
($198{,}549$ vs.\ $38{,}395$); on per-church MAP RMSE it reaches
$0.129$, well ahead of $0.227$ (PCA) and $0.230$ (dense).

At the larger $\sim\!3.5$M-parameter budget
(Fig.~\ref{fig:imagenette-samples-app}, centre-right, and
Tab.~\ref{tab:imagenette-density-app}, row 2) the multi-component extension
$F{=}\mathrm{diag}(w_{1})AL_{1}^{\!\top}+\mathrm{diag}(w_{2})AL_{2}^{\!\top}$
with arbitrary $L_{i}$ lifts \our{} to test LL $229{,}607$ and
church RMSE $0.115$ ($+31{,}058$ nats and $-11\%$ RMSE relative to
$I{=}1$). The same $\sim\!3.5$M budget allotted to the dense
families merely lifts their effective rank from $\sim\!6$ to
$\sim\!11$, two orders of magnitude below
$k_{\mathrm{laplex}}{=}1000$, and their samples remain
qualitatively unstructured (LL $66{,}107$ and $67{,}834$ for dense
$k_{\text{low-rank}}{=}11$ and PPCA $k_{\mathrm{cov}}{=}11$
respectively, vs.\ $229{,}607$ for \our{} $I{=}2$ at the same
budget). Pushing the construction one component further to $I{=}3$
($\sim\!4.84$M trainables, Tab.~\ref{tab:imagenette-density-app},
row~3 and Fig.~\ref{fig:imagenette-samples-app}, right) yields
test LL $238{,}578$ and church RMSE $0.114$, a marginal
$+8{,}971$-nat / $-0.001$-RMSE improvement on top of $I{=}2$ at
$+37\%$ extra trainables -- substantially less per-parameter than
the $I{=}1{\to}I{=}2$ jump. The optimized single-$\tau$
kernel scale itself shifts from $\tau{=}0.029$ ($I{=}1$) to
$\tau{=}0.052$ ($I{=}2$). Across-image generalization mirrors the church result: on ten
held-out images, one uniformly-random per ImageNette class
(Fig.~\ref{fig:imagenette-recon-gallery}), mean MAP RMSE drops
monotonically with $I$: $0.119$ ($I{=}1$), $0.105$ ($I{=}2$, $-12\%$
vs.\ $I{=}1$), $0.102$ ($I{=}3$, $-14\%$ vs.\ $I{=}1$, $-3\%$ on top
of $I{=}2$), at $+59\%$ and $+117\%$ trainable parameters
respectively. Every one of the ten $I{=}2$ reconstructions has
strictly lower RMSE than its $I{=}1$ counterpart (per-class gain
$-8\%$ for gas pump up to $-20\%$ for parachute); the $I{=}3$ row
matches or beats $I{=}2$ on every class at the $10^{-3}$ level
(strictly improving in $8$ of $10$, exactly tied on French horn at
$0.129$, and within $10^{-4}$ on golf ball at $0.067$), but the
per-class improvements are visibly smaller, again consistent with
the diminishing-returns trend in test LL.

\paragraph{Comparison to the empirical reference.}
For reference we also evaluate the ridge-regularized empirical
Gaussian $\mathcal{N}(\hat m, \hat\Sigma+\varepsilon I)$, computed
exactly via the SVD trick. The optimal $\varepsilon{=}10^{-2}$
(selected by val LL) yields test LL $315{,}549$ and per-image MAP
RMSE $0.075$, at the cost of $\sim\!2.9$\,B effective parameters
($N{\cdot}n+n$: the storage of all centered training images plus the
mean) or, equivalently, the entire empirical covariance
($\sim\!47$\,B unique entries). The held-out LL gap of \our{} to this
reference is $117{,}000$, $85{,}942$ and $76{,}971$ nats for $I{=}1,2,3$, at
$\sim\!1{,}300{\times}$, $\sim\!820{\times}$ and $\sim\!600{\times}$ fewer
parameters, with samples
(Fig.~\mbox{\ref{fig:imagenette-samples-app}}, left vs.\ centre-left,
centre-right, and right) that are visually only marginally less
detailed than the reference, while the matched-budget dense families
do not approach either at either of the two budgets we evaluated.
The reference is a fitted, ridge-regularized
model rather than an upper bound on held-out likelihood: the grid-aligned
stationary models of Table~\ref{tab:density-raster} reach a higher test LL on the
same data.

\begin{table}[h]
\centering
\small

\caption{ImageNette-$320$ density. Rows 1, 4 and 5 are at the
matched parameter budget ($\sim\!2.2$M trainable each): the
Laplace-kernel covariance of \our{} ($I{=}1$) gains
$\approx\!+154{,}000$ nats of held-out log-likelihood over the
strongest matched-budget dense baseline (the empirical-covariance
truncation at $k_{\mathrm{cov}}{=}6$), and stays ahead of the dense
low-rank baseline by $\approx\!+160{,}000$ nats. Both matched-budget
dense families collapse to comparable $\sim\!6$-dimensional
latents -- an unavoidable consequence of the $2.2$M-parameter
budget at $n{=}307{,}200$, where the linear cost of storing the
$V_{k_{\mathrm{cov}}}$ basis or the dense $F$ pins
$k_{\mathrm{cov}}\!\approx\!k_{\text{low-rank}}\!\approx\!6$. Only
the kernel-structured covariance of \our{} converts the same budget
into an effective rank of $1{,}000$. Rows~2--3 are the
multi-component extensions at $I{=}2$ and $I{=}3$ with arbitrary
(full) $L_{i}$, which are \emph{not} budget-matched ($+59\%$ and
$+117\%$ trainables respectively) but successively lift the
effective family beyond a single $(w,L)$ pair: $+31{,}058$ nats
($I{=}2$) and $+40{,}029$ nats ($I{=}3$) of test LL relative to
$I{=}1$, plus an $11\%$ ($I{=}2$) and $12\%$ ($I{=}3$) MAP-RMSE
reduction on the church. The marginal gain $I{=}2{\to}I{=}3$ is
$+8{,}971$ nats and only $-0.001$ MAP-RMSE -- diminishing returns
visible already at $I{=}3$.
$^{\dagger}$Rows~2--3 use $I{\in}\{2,3\}$ arbitrary
$k_{\mathrm{laplex}}\!\times\!k_{\mathrm{laplex}}$ factors $L_{i}$
(no triangular constraint) and per-component weight vectors
$w_{i}$, sharing the same phased Laplace anchors
$(a,\phi,b,\psi,\tau)$ as $I{=}1$. Trained with the same $25$-epoch Adam schedule. The optimized single-$\tau$ grows
monotonically with $I$: $0.029$ ($I{=}1$), $0.052$ ($I{=}2$),
$0.062$ ($I{=}3$).
$^{\sharp}$The empirical-covariance truncation at $320{\times}320$
is computed via the ``SVD trick'' ($N{=}9{,}469\!\ll\!n{=}307{,}200$,
so the top $k_{\mathrm{cov}}$ eigenvectors of $\hat\Sigma$ are
recovered exactly from the $N\!\times\!N$ centered-data Gram
without ever forming the $307{,}200\!\times\!307{,}200$ covariance).
At $k_{\mathrm{cov}}{=}6$ the truncation captures $47.5\%$ of the
empirical variance; the PPCA noise floor is
$\sigma^{2}{=}\frac{1}{n-k_{\mathrm{cov}}}\bigl(\mathrm{tr}\hat\Sigma-\sum_{j\le k_{\mathrm{cov}}}\lambda_{j}\bigr)
\!\approx\!4.3{\times}10^{-2}$
\citep{tipping1999probabilistic}. The MAP reconstruction column
uses the same held-out church image across all rows.
$^{\flat}$The full empirical Gaussian at $320{\times}320$ is
rank-deficient by construction
($\mathrm{rank}\,\hat\Sigma\!\le\!N{-}1{=}9{,}468\!\ll\!n$); ridge
regularization $\hat\Sigma+\varepsilon I$ with
$\varepsilon{=}10^{-2}$ (best of a logarithmic sweep over val LL)
makes it positive-definite and yields the test LL reported. The
full symmetric covariance would store
$n(n{+}1)/2\!\approx\!47$\,B entries; the minimal-storage
representation through the rank-$(N{-}1)$ spectrum is
$\sim\!N{\cdot}n+n\!\approx\!2.9$\,B parameters. Either way the row
is \emph{not} matched-budget; it is an empirical reference.}
\label{tab:imagenette-density-app}

\begin{tabular}{l l r r r}
\toprule
Model & rank & \# parameters & Test LL ($\uparrow$) & MAP RMSE ($\downarrow$)\\
\midrule
\our{} ($I{=}1$, lower-tri $L$)        & $k_{\mathrm{laplex}}{=}1000$ & $2\,231\,801$ & $198\,549$ & $0.129$\\
\midrule
\our{} ($I{=}2$, full $L_{i}$)$^{\dagger}$ & $k_{\mathrm{laplex}}{=}1000$ & $3\,538\,001$ & $229\,607$ & $0.115$\\
\midrule
\our{} ($I{=}3$, full $L_{i}$)$^{\dagger}$ & $k_{\mathrm{laplex}}{=}1000$ & $4\,845\,201$ & $\mathbf{238\,578}$ & $\mathbf{0.114}$\\
\midrule
dense low-rank                         & $k_{\text{low-rank}}{=}6$  & $2\,150\,442$ & $38\,395$           & $0.230$\\
\midrule
covariance trunc.\ (PPCA)$^{\sharp}$   & $k_{\mathrm{cov}}{=}6$    & $2\,150\,406$ & $44\,811$           & $0.227$\\
\midrule \\[0.3em]
\multicolumn{5}{c}{\emph{Reference: full empirical Gaussian, ridge-regularized (not matched-budget)}}\\
\midrule
$\mathcal{N}(\hat m,\hat\Sigma+\varepsilon I)$$^{\flat}$ & $k_{\mathrm{cov}}{=}n$ & $\sim\!2.9\,\mathrm{B}$ & $315\,549$ & $0.075$\\
\bottomrule
\end{tabular}
\end{table}

\paragraph{Takeaway.}
The most striking finding of this experiment is qualitative and
visible by eye in Fig.~\ref{fig:imagenette-samples-app}: \our{} at
$k_{\mathrm{laplex}}{=}1000$, trained on flat $307{,}200$-vectors
with no positional encoding, has \emph{learned that pixel colors
are correlated by spatial position}. Already at the small
$\sim\!2.2$M-parameter budget the $I{=}1$ sample
(Fig.~\ref{fig:imagenette-samples-app}, middle) shows contiguous
patches of green, brown or blue -- the kind of soft color blobs
found in natural-image priors -- and its MAP reconstruction
(Fig.~\ref{fig:imagenette-recon-gallery}, middle row) preserves the
church silhouette and sky--ground gradient, whereas both
budget-matched dense families collapse to spatially uniform noise.
Since the model has no convolution or positional encoding, this spatial structure is carried
by the learned \our{} factor. The
multi-component extension (Fig.~\ref{fig:imagenette-samples-app},
centre-right and right) inherits the same qualitative behaviour at
all $I$ tested, and lifts test LL monotonically with $I$:
$+31{,}058$ nats $I{=}1{\to}I{=}2$ (at $+59\%$ trainables) and
$+8{,}971$ nats $I{=}2{\to}I{=}3$ (at $+37\%$ extra trainables on
top of $I{=}2$), each step requiring no architectural change
beyond Eq.~\eqref{eq:multi-component}. The marginal $I{=}3$ gain
is already much smaller than the $I{=}2$ jump -- a diminishing-returns pattern.
The dense families re-trained at $\sim\!3.5$M parameters still
cannot escape their effective rank of $\sim\!11$. Within the
Gaussian families tested here -- at either matched budget -- LAPLEX
is the only parameterization that turns a few million Woodbury
parameters into coherent ImageNette-$320$ samples and
reconstructions; dense low-rank
($k_{\text{low-rank}}{\in}\{6,11\}$) and empirical-covariance
truncation ($k_{\mathrm{cov}}{\in}\{6,11\}$) remain too rank-limited
to do so.

\clearpage

\subsection{Stationary covariances under a fixed coordinate permutation}
\label{sec:density-permuted}
We compare LAPLEX with a stationary FFT-diagonalizable block-circulant Gaussian,
a budget-matched mixture of such Gaussians, and a diagonal floor.
The block-circulant model is translation invariant on the $H\times W$
image torus with free color coupling: the per-channel 2-D DFT diagonalizes
its covariance $C$ into one $3\times3$ Hermitian block
$B_\omega=L_\omega L_\omega^{\mathsf H}+\sigma^{2}I_3$ per spatial frequency
$\omega$, with $L_\omega$ complex lower-triangular and $\sigma^{2}=10^{-4}$
fixed. This gives $9$ real parameters per frequency, $9H(\lfloor W/2\rfloor+1)$
in total, and an $\mathcal{O}(n\log n)$ likelihood. The mean $m$ is the frozen
empirical training mean, and the linear algebra uses \texttt{float64}.
For the single-component permuted circulant model,
$\Sigma_{\pi}=P^\top C P$, the permutation is
one random permutation of all $n=3HW$ coordinates, color channels
included, and is fixed across
training and evaluation. It removes the
alignment between the observed coordinate order and the image grid, and
also the grouping of the color channels of each pixel.
The $K$-component rows are Gaussian mixtures
$p(x)=\sum_{j=1}^{K}w_j\,\mathcal{N}(x;m,P^\top C_jP)$ with one shared mean,
$K$ block-circulant covariances and learned weights. We use $K=5$
on ImageNette and $K=3$ on DIV2K to match the parameter budget, rather
than as a performance-tuning sweep.
The circulant models use Adam with batch size $64$ for $25$ epochs on ImageNette
and batch size $4$ for $20$ epochs on DIV2K. The $K=1$ models start near white
noise and use learning rate $2\times10^{-3}$. Each mixture component starts
from the trained $K=1$ model of the same coordinate order plus $\mathcal{N}(0,0.02^{2})$ noise on
every parameter, with equal initial weights and learning rate
$5\times10^{-4}$. The diagonal floor is fitted in closed form.

The rows of Tables~\ref{tab:density-permuted} and~\ref{tab:density-raster} are:
\begin{itemize}
\item \emph{LAPLEX}: the raster-order runs of
Table~\ref{tab:imagenette-density-app} ($I=1$, $k=1000$, $2{,}231{,}801$ trainable
parameters) and of the DIV2K-1024 protocol of this appendix (width $1000$, batch
size $8$, $13.6$M parameters), in both tables.
\item \emph{Circulant mixture ($K=5$ or $K=3$)}: the mixture above, with
$K\cdot9H(\lfloor W/2\rfloor+1)+K$ parameters including the mixing weights
($2{,}318{,}405$ and $14{,}183{,}427$).
\item \emph{Circulant ($K=1$)}: one block-circulant Gaussian, with
$9H(\lfloor W/2\rfloor+1)$ parameters ($463{,}680$ and $4{,}727{,}808$).
\item \emph{Diagonal}: independent coordinates with the frozen mean and one
fitted variance each, so $n$ parameters. Coordinate order does not affect it, so both
tables show the same fit.
\end{itemize}
In Table~\ref{tab:density-permuted} the circulant rows are trained and evaluated
on the permuted coordinates; in Table~\ref{tab:density-raster} they use raster order.

\begin{table}[htbp]
\centering\small
\caption{Density modeling on permuted coordinates. The circulant mixtures match the LAPLEX parameter budgets within approximately
$4\%$ on ImageNette and $5\%$ on DIV2K. Test LL is in nats;
bits/dim is negative LL divided by dimension and by $\log 2$, so lower
bits/dim is better. These comparisons are distinct from the separate
raster-order denoising experiment in App.~\ref{app:density-downstream}. The LAPLEX rows are the raster-order runs of
Table~\ref{tab:imagenette-density-app} and Fig.~\ref{fig:div2k-samples}
(see text).}
\label{tab:density-permuted}
\begin{tabular}{@{}llrrr@{}}
\toprule
Dataset & model & parameters & test LL ($\uparrow$) & bits/dim ($\downarrow$)\\
\midrule
ImageNette-320 & LAPLEX ($I=1$) & $2.23$M & $+198{,}549$ & $-0.932$\\
& Circulant mixture ($K=5$) & $2.32$M & $+2{,}932$ & $-0.014$\\
& Circulant ($K=1$) & $0.46$M & $-15{,}677$ & $+0.074$\\
& Diagonal & $0.31$M & $-52{,}321$ & $+0.246$\\
\midrule
DIV2K-1024 & LAPLEX (width $1000$) & $13.6$M & $+1{,}266{,}679$ & $-0.581$\\
& Circulant mixture ($K=3$) & $14.2$M & $+126{,}214$ & $-0.058$\\
& Circulant ($K=1$) & $4.73$M & $-14{,}398$ & $+0.007$\\
& Diagonal & $3.15$M & $-408{,}917$ & $+0.188$\\
\bottomrule
\end{tabular}
\end{table}

Table~\ref{tab:density-permuted} shows that LAPLEX attains higher held-out
log-likelihood than the budget-matched stationary covariances in the
permuted setting. At the level of the representable family, permuting the
coordinates together with all per-coordinate parameters (anchors $a$,
phases $\phi$, weights $w$, noise scales $d$ and mean $m$) preserves the LAPLEX
model. The LAPLEX rows of
Table~\ref{tab:density-permuted} report the raster-order runs of
Table~\ref{tab:imagenette-density-app} and Fig.~\ref{fig:div2k-samples},
with anchors initialized from the coordinate index; LAPLEX was not retrained on
permuted coordinates. In raster order, the stationary circulant is instead
given the image grid and can exploit its alignment. With raster-ordered
coordinates, the grid-aligned $K=1$ circulant reaches $-3.51$ bits/dim on
ImageNette-320 and $-3.46$ bits/dim on DIV2K-1024, compared with $-0.93$ and
$-0.58$ for LAPLEX; Table~\ref{tab:density-raster} gives the full comparison.
It outperforms LAPLEX
in the separate raster-order benchmark in App.~\ref{app:density-downstream}.
The distinction is between learning geometry and supplying it in advance,
not a claim that LAPLEX is the best covariance model for images overall.

\begin{table}[htbp]
\centering\small
\caption{Density modeling in raster order, with the models and units of
Table~\ref{tab:density-permuted}. The LAPLEX and diagonal rows are the same as there.}
\label{tab:density-raster}
\begin{tabular}{@{}llrrr@{}}
\toprule
Dataset & model & parameters & test LL ($\uparrow$) & bits/dim ($\downarrow$)\\
\midrule
ImageNette-320 & LAPLEX ($I=1$) & $2.23$M & $+198{,}549$ & $-0.932$\\
& Circulant mixture ($K=5$) & $2.32$M & $+780{,}128$ & $-3.664$\\
& Circulant ($K=1$) & $0.46$M & $+747{,}647$ & $-3.511$\\
& Diagonal & $0.31$M & $-52{,}321$ & $+0.246$\\
\midrule
DIV2K-1024 & LAPLEX (width $1000$) & $13.6$M & $+1{,}266{,}679$ & $-0.581$\\
& Circulant mixture ($K=3$) & $14.2$M & $+7{,}997{,}322$ & $-3.668$\\
& Circulant ($K=1$) & $4.73$M & $+7{,}539{,}800$ & $-3.458$\\
& Diagonal & $3.15$M & $-408{,}917$ & $+0.188$\\
\bottomrule
\end{tabular}
\end{table}

\section{Attention as a drop-in operator}\space%
\label{app:attention}

We test \our{} as an attention mechanism in a controlled setting: a standard
pre-norm Transformer over pixel-token sequences in which \emph{only} the
attention sub-layer is swapped, holding architecture, optimizer, schedule and
sweep identical. We compare standard softmax attention,
Performer-style random-feature (RFF) linear attention\space%
 (with and without a learnable temperature), and
\our{} attention $A_{ij}=\exp(-|a_i-b_j|/\tau)$ with length scale $\tau>0$. Sweep:\space%
 MNIST/SVHN/CIFAR-10
$\times$ width $\{64,128,256\}\times$ depth $\{4,8\}\times 3$ repeats. Best
configuration per model (mean\,$\pm$\,std over $3$ repeats, test accuracy \%):
\begin{center}
\begin{tabular}{@{}lcccc@{}}
\toprule
dataset & Softmax & RFF & RFF$+\tau$ & \our{} \\
\midrule
MNIST    & $33.80\pm0.20$ & $33.74\pm0.05$ & $33.83\pm0.20$ & $\mathbf{33.93\pm0.27}$ \\
SVHN     & $19.59\pm0.00$ & $19.61\pm0.01$ & $19.61\pm0.02$ & $\mathbf{20.57\pm0.85}$ \\
CIFAR-10 & $42.47\pm0.64$ & $35.07\pm3.22$ & $44.92\pm2.50$ & $\mathbf{47.40\pm0.62}$ \\
\bottomrule
\end{tabular}
\end{center}
\our{} attention matches or exceeds softmax and the RFF baselines, most clearly
on CIFAR-10. The kernel is computed \emph{densely}
($\mathcal{O}(N^2)$ for $N$ tokens) here, so this tests modeling quality
rather than the runtime of the corresponding matrix--vector primitive. That primitive has
$\mathcal{O}(N\log N)$ total work including sorting and linear scan work
for sorted anchors; primitive runtime comparisons are reported separately
in Figs.~\ref{fig:comp_runtime} and \ref{fig:kerlap-vs-fft}.\space%
Absolute accuracies are low in this minimal pixel-token Transformer, so
the comparison is a controlled operator test rather than a competitive
image-classification benchmark.\space%

\section{Second-order optimization with an exact \our{} curvature}
\label{app:second-order}

We use the exact weighted Gram (Thm.~\ref{thm:gram}) as the curvature inside
Fisher scoring / natural gradient for a logistic GLM with Laplace-kernel
design $A=\our(a,b)\in\mathbb{R}^{n\times k}$ and logits $\ell=A^\top u$,
where $u\in\mathbb{R}^{n}$. With $p_j=(1+e^{-\ell_j})^{-1}$ and
$s_j=p_j(1-p_j)$, the regularized Fisher / generalized-Gauss--Newton
curvature is $H=\tfrac1k A\operatorname{diag}(s)A^\top+\lambda I$,
where $\lambda$ is the regularization coefficient. We form it through the
scan Gram without materializing $A$. The curvature and natural-gradient
step agree numerically with dense evaluation (maximum relative error
$\le1.4\times10^{-11}$ in fp64 over $n\le400$, $k\le20{,}000$).
Table~\ref{tab:second-order-convergence} reports the convergence trace:
the gradient norm reaches $1.46\times10^{-10}$ after three updates,
starting from $1.10$ at iteration zero.

\begin{table}[htbp]
\centering\small
\caption{Fisher-scoring convergence with exact LAPLEX curvature
($n=512$, $k=50{,}000$, fp64). Iteration zero is the initial point;
$L_\star$ denotes the reference optimum used to report the objective gap.}
\label{tab:second-order-convergence}
\begin{tabular}{@{}rrr@{}}
\toprule
iteration & $L(u)-L_\star$ & $\lVert\nabla L\rVert$\\
\midrule
$0$ & $2.49\times10^{-2}$ & $1.10$\\
$1$ & $1.21\times10^{-5}$ & $2.15\times10^{-2}$\\
$2$ & $4.66\times10^{-11}$ & $3.57\times10^{-5}$\\
$3$ & $0$ & $1.46\times10^{-10}$\\
$4$ & $0$ & $3.07\times10^{-17}$\\
\bottomrule
\end{tabular}
\end{table}

\begin{table}[htbp]
\centering\small
\caption{Iterations and wall-clock time to reach
$\lVert\nabla L\rVert<10^{-6}$ ($n=512$, $k=50{,}000$), mean over
three seeds. Adam uses its best of three learning rates. NR means the
tolerance was not reached within $20{,}000$ iterations; time to tolerance is then undefined (n/a).}
\label{tab:second-order}
\begin{tabular}{@{}lrrr@{}}
\toprule
method (order) & iterations to tolerance & time (s) & final $\lVert\nabla L\rVert$\\
\midrule
Fisher (exact, 2nd) & $\mathbf{3}$ & $0.017$ & $1.5\times10^{-10}$\\
Adam (1st) & NR & n/a & $3.8\times10^{-1}$\\
gradient descent (1st) & NR & n/a & $2.9$\\
\bottomrule
\end{tabular}
\end{table}

\begin{table}[htbp]
\centering\small
\caption{Rank-$r$ Nystr\"om curvature inside the same Fisher-scoring
loop, compared with exact LAPLEX curvature ($n=512$, $k=50{,}000$;
gradient-norm tolerance $10^{-6}$). NR means the tolerance was not reached.}
\label{tab:second-order-nystrom}
\begin{tabular}{@{}lrrr@{}}
\toprule
curvature & rank & iterations to tolerance & final $\lVert\nabla L\rVert$\\
\midrule
Nystr\"om & $16$ & NR & $1.0\times10^{1}$\\
Nystr\"om & $64$ & NR & $1.0\times10^{1}$\\
Nystr\"om & $256$ & $12$ & $8.7\times10^{-7}$\\
Nystr\"om & $1024$ & $6$ & $6.6\times10^{-7}$\\
Nystr\"om & $4096$ & $4$ & $9.7\times10^{-7}$\\
exact LAPLEX & full & $\mathbf{3}$ & $1.1\times10^{-10}$\\
\bottomrule
\end{tabular}
\end{table}

\begin{table}[htbp]
\centering\small
\caption{Cost of one curvature build versus inner dimension $k$
($n=512$, fp32). OOM denotes failure to fit the dense design in the
$10$ GB GPU used for this experiment. These are the numerical build-time
and memory results corresponding to Fig.~\ref{fig:second-order}.}
\label{tab:second-order-build}
\begin{tabular}{@{}rrrrr@{}}
\toprule
$k$ & LAPLEX (ms) & LAPLEX (GB) & dense (ms) & dense (GB)\\
\midrule
$2^{18}$ & $1.03$ & $0.026$ & $17.1$ & $1.088$\\
$2^{20}$ & $1.22$ & $0.064$ & $78.5$ & $4.321$\\
$2^{21}$ & $2.94$ & $0.120$ & $217.9$ & $8.633$\\
$2^{22}$ & $3.90$ & $0.231$ & OOM & OOM\\
$2^{24}$ & $14.0$ & $0.898$ & OOM & OOM\\
$2^{25}$ & $27.8$ & $1.787$ & OOM & OOM\\
\bottomrule
\end{tabular}
\end{table}
\begin{figure}[htpb]
\centering
\begin{minipage}{0.49\linewidth}\centering
\includegraphics[width=\linewidth]{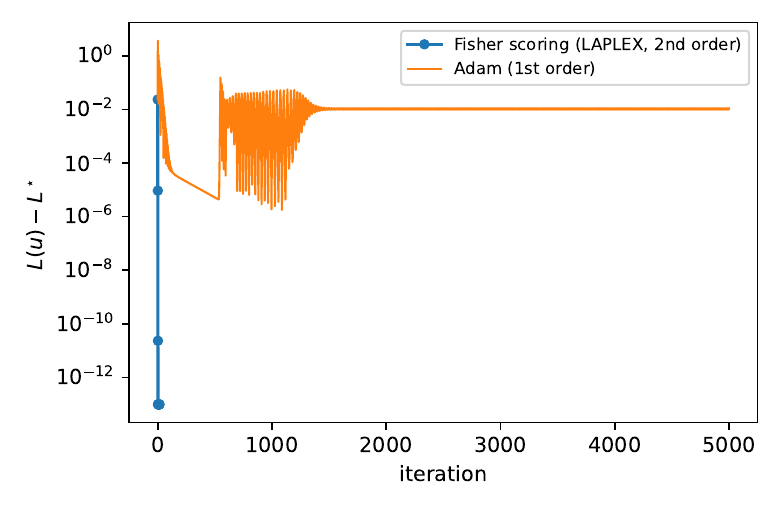}
\end{minipage}\hfill
\begin{minipage}{0.49\linewidth}\centering
\includegraphics[width=\linewidth]{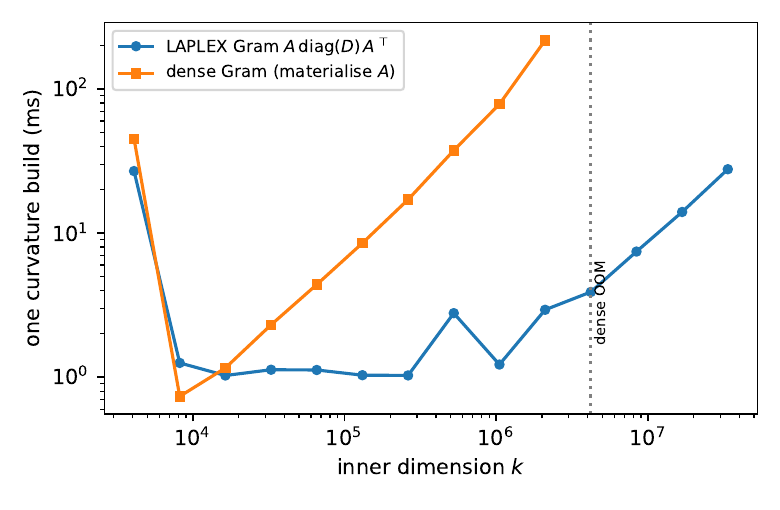}
\end{minipage}
\caption{\emph{Left:} Fisher scoring with the exact \our{} curvature vs.\
Adam on the same objective. \emph{Right:} cost of one curvature build vs.\ inner
dimension $k$ ($n{=}512$); the dense build OOMs at $k=2^{22}$ on a 10\,GB GPU
while the exact \our{} curvature stays $\le1.8$\,GB to $k=2^{25}$ ($\sim74\times$
faster / $\sim72\times$ lighter at $k=2^{21}$).}
\label{fig:second-order}
\end{figure}

\section{Density modeling: downstream task and stronger structured baselines}
\label{app:density-downstream}

This appendix complements Sec.~\ref{sec:experiments-div2k} with a
downstream MAP image-restoration (denoising) benchmark and two stronger
structured baselines: an FFT-circulant / stationary Gaussian (the fixed-geometry
foil, $\mathcal{O}(n)$ parameters, $\mathcal{O}(n\log n)$ compute) and a
Kronecker / tensor-normal Gaussian. This is a uniformly trained
comparison on raster-order ImageNette-320 ($n{=}307{,}200$; $8{,}000$ train,
$20$ epochs, one\space%
shared low-rank-plus-diagonal template per model); it is self-contained and separate from the fully-trained App.~\ref{sec:density-model-app} run, so its
per-dimension log-likelihoods are on a different scale (e.g.\ LAPLEX
$I=1$ in Table~\ref{tab:imagenette-density-app} has test LL $198{,}549$ nats)\space%
 and are meaningful only for within-table
cross-model comparison. Table~\ref{tab:density-raster-downstream} reports the results.

\begin{table}[htbp]
\centering\small
\caption{Separate raster-order ImageNette-320 comparison with stronger
structured covariances and MAP denoising ($8{,}000$ training images,
$20$ epochs). Models are evaluated within this shared protocol; the
parameter counts differ. Bold marks the best grid-agnostic model and the best model overall.}
\label{tab:density-raster-downstream}
\begin{tabular}{@{}lccccc@{}}\space%
\toprule
model & grid prior? & params & test LL/dim & MAP PSNR (dB) & MAP SSIM \\
\midrule
PPCA (empirical trunc.) & no & 1.23M & $0.118$ & $15.38$ & $0.208$ \\
dense low-rank & no & 1.54M & $0.125$ & $15.42$ & $0.212$ \\
\our{} & no (learned) & 2.23M & $\mathbf{0.284}$ & $\mathbf{16.37}$ & $\mathbf{0.215}$ \\
FFT-circulant & yes (fixed) & 0.61M & $1.681$ & $22.84$ & $0.527$ \\
Kronecker & yes (fixed) & 0.51M & $\mathbf{2.261}$ & $\mathbf{23.65}$ & $\mathbf{0.567}$ \\
\bottomrule
\end{tabular}
\end{table}

In this raster-order comparison, \our{} outperforms the tested grid-agnostic
dense low-rank and PPCA baselines on both\space%
log-likelihood and MAP denoising. The FFT-circulant and Kronecker
models, however, outperform LAPLEX on all three reported metrics while
using fewer parameters. They exploit the supplied pixel-grid geometry.
Thus the results support LAPLEX as a way to learn useful covariance
structure among the tested grid-agnostic models, rather than superiority
over fixed image-grid priors. The fixed-permutation comparison in
Table~\ref{tab:density-permuted} addresses a different setting and is not
paired with this run.%

\section{Anchor-learning dynamics: ordering, initialization, components}
\label{app:anchor-dynamics}

\paragraph{Do anchor crossings destabilize training?}

When two anchors within the same input or output set cross, their internal
sorted order changes. Changing the internal evaluation order alone does not
change the kernel represented in the original coordinate order. The kernel
$\exp(-|a_i-b_j|/\tau)$ is continuous everywhere and differentiable away
from input--output coincidences $a_i=b_j$. At such a coincidence the
absolute value has a kink, so continuity of the kernel does not imply
continuity of its anchor gradient. Same-side ordering changes should not
be conflated with these nondifferentiable points.

In the trainable-projection diagnostic ($m=500$), we trained the projection
of Table~\ref{tab:logreg-merged} (same model, AdamW optimizer and batch size $512$)
for six epochs on non-augmented training images and logged every optimizer step
(Fig.~\ref{fig:anchor-ordering}). A crossing rate is the fraction of pairs that are
adjacent in the sorted order and reversed afterwards, measured between consecutive
steps and between the start and the end of each epoch. The input anchors reshuffle
continuously: after the first step, $0.44$--$0.48$ of adjacent pairs swap at every
step, and $0.49$--$0.50$ over every epoch. The output anchors cross far less often:
$0.00$--$0.06$ per step, and $0.23$ over the first epoch versus about $0.10$ over each
later one. Despite these persistent crossings, the epoch-level training loss decreases
monotonically (from $2.13$ to $1.60$), and the joint anchor-gradient norm stays between
$0.13$ and $1.49$; it peaks right after the first update and stays between $0.31$ and
$0.84$ after the first five steps. The minibatch loss fluctuates from step to step (it
increases in $68$ of $143$ steps), as expected for stochastic training. With the
augmented training images of Table~\ref{tab:logreg-merged} the picture is the same:
per-epoch crossing rates of $0.49$ for the input anchors and of $0.22$ in the first
epoch and $0.08$--$0.13$ later for the output anchors, and an anchor-gradient norm
between $0.12$ and $2.40$. These observations support stability in this run, rather
than a general guarantee for every crossing or training configuration.

\begin{figure}[htpb]
\centering
\includegraphics[width=\linewidth]{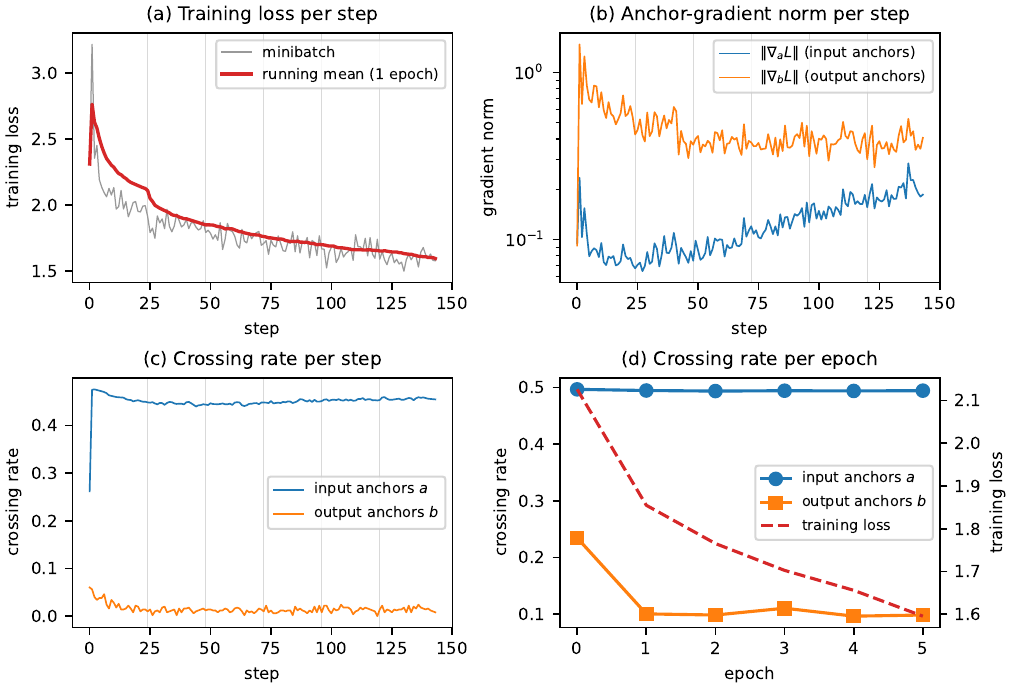}
\caption{Anchor crossings in the trainable projection ($m=500$, non-augmented
training images, six epochs). (a)~Minibatch training loss per optimizer step and its
running mean over one epoch. (b)~Gradient norms of the input anchors $a$ and the output
anchors $b$ per step. (c)~Crossing rate per step: the fraction of pairs adjacent in the
sorted order before a step that are reversed after it. (d)~Crossing rate per epoch (order
at the start versus the end of the epoch) with the epoch-level training loss (dashed).
Vertical lines in (a)--(c) mark epoch boundaries.}
\label{fig:anchor-ordering}
\end{figure}

\paragraph{Initialization and number of components.} Across seven anchor initializations the best accuracy varies by only $\approx2.6$~pp (random
\texttt{normal} best; a deterministic \texttt{linspace} grid $\sim2$~pp worse;
$\tau\in\{0.05,0.1,0.5\}$ within $0.3$~pp), and accuracy improves monotonically
with the number of components $I$ up to $I{=}3$ (matching the density ablation,
Table~\ref{tab:imagenette-density-app}). Both ablations use cached, non-augmented features,\space%
 so absolute accuracies
($\sim40\%$) sit below Table~\ref{tab:logreg-merged}'s augmented $\sim43.8\%$ at
$m{=}500$; the trends, not the absolute values, are the point.
Table~\ref{tab:anchor-ablations} gives both ablations.

\begin{table}[htbp]
\centering\small
\caption{Anchor initialization (left, five seeds) and number of components
(right, three seeds) for the trainable projection at $m=500$ on cached,
non-augmented features. Entries are best validation accuracy, averaged over seeds, and standard
deviation. The right-hand component sweep is independent of the
initialization rows; blanks do not denote missing runs.}
\label{tab:anchor-ablations}
\begin{tabular}{@{}lccc@{}}\space%
\toprule
init (scale, $\tau$) & acc (\%) & \quad $I$ & acc (\%) \\
\midrule
normal ($1.0,0.10$, default) & $40.48\pm1.09$ & $1$ & $39.40\pm1.11$ \\
normal ($1.0,0.05$) & $40.68\pm1.56$ & $2$ & $40.00\pm2.03$ \\
normal ($1.0,0.50$) & $40.40\pm1.62$ & $3$ & $41.67\pm0.99$ \\
normal ($3.0,0.10$) & $40.04\pm1.47$ &  &  \\%
uniform ($1.0,0.10$) & $39.16\pm0.48$ & $4$ & $40.40\pm1.59$ \\
uniform ($3.0,0.10$) & $39.52\pm0.84$ &  &  \\%
linspace grid ($3.0,0.10$) & $38.12\pm1.04$ & & \\
\bottomrule
\end{tabular}
\end{table}

\paragraph{Robustness and the smallest parameter budgets.}
The frozen-temperature sweep in App.~\ref{sec:logreg-frozen-app},
Table~\ref{tab:logreg-frozen-tau-sweep}, covers $\tau\in[0.01,1]$;
$\tau=0.1$ is a representative default rather than a tuned optimum.
The projection sweep improves up to $I=3$ and remains within variability
at $I=4$. The main setting requiring care is initialization at the
smallest parameter budgets: for example, unsuitable initialization and
learning-rate combinations can drive rank-$4$ heads
(Table~\ref{tab:heads_results}) towards chance accuracy. We selected
initialization and schedule once through a staged search on CIFAR-10 and
SVHN, then reused the selected choices.

\clearpage
\section*{NeurIPS Paper Checklist}

\begin{enumerate}

\item {\bf Claims}
    \item[] Question: Do the main claims made in the abstract and introduction accurately reflect the paper's contributions and scope?
    \item[] Answer: \answerYes{} 
    \item[] Justification: The paper makes concrete claims about (i) exact Laplace-based operators enabling near-linear matvec and Gram computations, (ii) improved scaling vs dense kernels and FFT-based methods, and (iii) accuracy advantages over RFF and structured projections. These claims are explicitly supported by both theory (Sections 3–6) and experiments.

\item {\bf Limitations}
    \item[] Question: Does the paper discuss the limitations of the work performed by the authors?
    \item[] Answer: \answerYes{} 
    \item[] Justification: A dedicated "Conclusions and Limitations" section discusses constraints such as reliance on learned anchor parameterization, potential regime dependence of empirical gains, and implicit assumptions about structure in data.

\item {\bf Theory assumptions and proofs}
    \item[] Question: For each theoretical result, does the paper provide the full set of assumptions and a complete (and correct) proof?
    \item[] Answer: \answerYes{} 
    \item[] Justification: The main text introduces formal constructions (Laplace algebra, structured operators), while full proofs are deferred to the appendix. Assumptions are stated alongside derivations.

    \item {\bf Experimental result reproducibility}
    \item[] Question: Does the paper fully disclose all the information needed to reproduce the main experimental results of the paper to the extent that it affects the main claims and/or conclusions of the paper (regardless of whether the code and data are provided or not)?
    \item[] Answer: \answerYes{} 
    \item[] Justification: The paper provides comprehensive technical details necessary for reproducing the results, along with datasets and experimental settings. The authors will also release the code and resulting checkpoints to ensure complete reproducibility.

\item {\bf Open access to data and code}
    \item[] Question: Does the paper provide open access to the data and code, with sufficient instructions to faithfully reproduce the main experimental results, as described in supplemental material?
    \item[] Answer: \answerYes{} 
    \item[] Justification: The code and/or datasets will be released after the paper is accepted.

\item {\bf Experimental setting/details}
    \item[] Question: Does the paper specify all the training and test details (e.g., data splits, hyperparameters, how they were chosen, type of optimizer) necessary to understand the results?
    \item[] Answer: \answerYes{} 
    \item[] Justification: The paper explains the computational regimes (large-scale kernels, GPU scaling behavior) and compares against multiple baselines in structured settings.

\item {\bf Experiment statistical significance}
    \item[] Question: Does the paper report error bars suitably and correctly defined or other appropriate information about the statistical significance of the experiments?
    \item[] Answer: \answerYes{} 
    \item[] Justification: The figures and tables report standard deviations computed over multiple random seeds. More detailed results are provided in the appendix.

\item {\bf Experiments compute resources}
    \item[] Question: For each experiment, does the paper provide sufficient information on the computer resources (type of compute workers, memory, time of execution) needed to reproduce the experiments?
    \item[] Answer: \answerYes{} 
    \item[] Justification: For each experiment, the paper specifies the hardware used, including the type and configuration of GPUs. Most experiments were conducted on a single NVIDIA H100 GPU.
    
\item {\bf Code of ethics}
    \item[] Question: Does the research conducted in the paper conform, in every respect, with the NeurIPS Code of Ethics \url{https://neurips.cc/public/EthicsGuidelines}?
    \item[] Answer: \answerYes{} 
    \item[] Justification: No ethical risks are apparent; the work is methodological and does not involve sensitive data.

\item {\bf Broader impacts}
    \item[] Question: Does the paper discuss both potential positive societal impacts and negative societal impacts of the work performed?
    \item[] Answer: \answerYes{} 
    \item[] Justification: We discuss the broader impact of our work in the Introduction and Conclusion sections.
    
\item {\bf Safeguards}
    \item[] Question: Does the paper describe safeguards that have been put in place for responsible release of data or models that have a high risk for misuse (e.g., pre-trained language models, image generators, or scraped datasets)?
    \item[] Answer: \answerNA{} 
    \item[] Justification: The paper poses no such risks.

\item {\bf Licenses for existing assets}
    \item[] Question: Are the creators or original owners of assets (e.g., code, data, models), used in the paper, properly credited and are the license and terms of use explicitly mentioned and properly respected?
    \item[] Answer: \answerYes{} 
    \item[] Justification: We cited all utilized works and resources according to the citing guidelines.

\item {\bf New assets}
    \item[] Question: Are new assets introduced in the paper well documented and is the documentation provided alongside the assets?
    \item[] Answer: \answerNA{} 
    \item[] Justification: The paper does not release new assets.

\item {\bf Crowdsourcing and research with human subjects}
    \item[] Question: For crowdsourcing experiments and research with human subjects, does the paper include the full text of instructions given to participants and screenshots, if applicable, as well as details about compensation (if any)? 
    \item[] Answer: \answerNA{} 
    \item[] Justification: The paper does not involve crowdsourcing or research with human subjects.

\item {\bf Institutional review board (IRB) approvals or equivalent for research with human subjects}
    \item[] Question: Does the paper describe potential risks incurred by study participants, whether such risks were disclosed to the subjects, and whether Institutional Review Board (IRB) approvals (or an equivalent approval/review based on the requirements of your country or institution) were obtained?
    \item[] Answer: \answerNA{} 
    \item[] Justification: This work does not include any crowdsourcing or studies involving human participants.

\item {\bf Declaration of LLM usage}
    \item[] Question: Does the paper describe the usage of LLMs if it is an important, original, or non-standard component of the core methods in this research? Note that if the LLM is used only for writing, editing, or formatting purposes and does \emph{not} impact the core methodology, scientific rigor, or originality of the research, declaration is not required.
    \item[] Answer: \answerNA{} 
    \item[] Justification: LLM was used only for editing, formating and polishing the text.

\end{enumerate}

\end{document}